\documentclass[a4paper,fleqn]{cas-sc}

\usepackage[numbers]{natbib}
\usepackage{CJKutf8}
\usepackage[section]{placeins}

\newtheorem{theorem}{Theorem}
\newtheorem{lemma}[theorem]{Lemma}
\newtheorem{proposition}[theorem]{Proposition}
\newtheorem{corollary}[theorem]{Corollary}

\newdefinition{definition}{Definition}
\newdefinition{rmk}{Remark}
\newdefinition{summary}{Summary}
\newdefinition{assumption}{Assumption}
\newproof{pf}{pf}

\def\tsc#1{\csdef{#1}{\textsc{\lowercase{#1}}\xspace}}
\tsc{WGM}
\tsc{QE}
\tsc{EP}
\tsc{PMS}
\tsc{BEC}
\tsc{DE}

\begin{document}
\let\WriteBookmarks\relax
\def\floatpagepagefraction{1}
\def\textpagefraction{.001}
\shorttitle{Conjugate Learning Theory}
\shortauthors{Binchuan Qi}

\title [mode = title]{Conjugate Learning Theory: Uncovering the Mechanisms of Trainability and Generalization in Deep Neural Networks}                      

\author[1,2]{Binchuan Qi}[orcid=0000-0001-5832-1884]
                        
\cormark[1]
\fnmark[1]
\ead{2080068@tongji.edu.cn}

\credit{Methodology, Software, Data curation, Writing - Original draft preparation}

\affiliation[1]{organization={Tongji University},
                addressline={Siping Street}, 
                postcode={200092}, 
                postcodesep={}, 
                city={Shanghai},
                country={China}}
                
\affiliation[2]{organization={Zhejiang Yuying College of Vocational Technology},
                addressline={Sihao Street}, 
                city={Hangzhou},
                postcode={310018}, 
                state={Zhejiang},
                country={China}}

\nonumnote{In this work, we propose the conjugate learning theory framework to systematically analyze the optimization dynamics and generalization mechanisms that underpin the performance of deep neural networks in practical learning scenarios.}

\begin{abstract}
Machine learning techniques centered on deep neural networks (DNNs) have achieved remarkable empirical success across a broad spectrum of domains, ranging from computer vision and natural language processing to reinforcement learning. However, owing to the inherent complexity and architectural diversity of modern DNNs, two fundamental challenges remain insufficiently addressed within the scope of classical learning theory: trainability, which concerns why simple gradient-based iterative optimization algorithms can effectively optimize highly non-convex DNN models to achieve low empirical risk, and generalization, which explains how over-parameterized DNNs consistently attain strong generalization performance on unseen data despite possessing far more parameters than training samples. In this work, we propose a notion of practical learnability grounded in finite sample settings, and develop a conjugate learning theoretical framework based on convex conjugate duality to characterize this learnability property. Building on this foundation, we demonstrate that training DNNs with mini-batch stochastic gradient descent (SGD) achieves global optimality of empirical risk by jointly controlling the extremal eigenvalues of a structure matrix and the gradient energy, and we establish a corresponding convergence theorem. We further elucidate the impact of batch size and model architecture (including depth, parameter count, sparsity, skip connections, and other characteristics) on non-convex optimization. Additionally, we derive a model-agnostic lower bound for the achievable empirical risk, theoretically demonstrating that data determines the fundamental limit of trainability. On the generalization front, we derive deterministic and probabilistic bounds on generalization error based on generalized conditional entropy measures. The former explicitly delineates the range of generalization error, while the latter characterizes the distribution of generalization error relative to the deterministic bounds under independent and identically distributed (i.i.d.) sampling conditions. Furthermore, these bounds explicitly quantify the influence of three key factors: (i) information loss induced by irreversibility in the model, (ii) the maximum attainable loss value, and (iii) the generalized conditional entropy of features with respect to labels. Moreover, they offer a unified theoretical lens for understanding the roles of regularization, irreversible transformations, and network depth in shaping the generalization behavior of deep neural networks. Finally, we validate our theoretical predictions through extensive deep learning experiments. The close alignment between theory and empirical results confirms the correctness of the proposed framework.
\end{abstract}


\maketitle

\section{Introduction}
\label{sec:introduction}

Machine learning techniques based on deep neural networks (DNNs) have achieved unprecedented empirical success across a broad spectrum of real-world applications, including image classification, natural language understanding, speech recognition, and autonomous driving. Despite this widespread practical success, the theoretical foundations underpinning the \textit{trainability} (the ability to optimize non-convex models to low empirical risk) and \textit{generalization} (the ability to perform well on unseen data) of DNNs remain poorly understood. As a result, deep learning is often characterized as an experimental science, with theoretical developments lagging behind practical advances and offering limited actionable guidance for real-world model design, algorithm selection, and hyperparameter tuning.

\subsection{The trainability puzzle}

The \textit{trainability} of DNNs refers to the well-documented empirical observation that highly over-parameterized, non-convex DNN models, when optimized with simple first-order optimization methods such as stochastic gradient descent (SGD) and its variants, consistently converge to high-quality solutions with low empirical risk, despite the absence of convexity or strong regularity assumptions that are typically required for theoretical guarantees in classical optimization~\cite{Zhang2021UnderstandingDL}. In contrast, classical non-convex optimization theory only guarantees convergence to stationary points (points where the gradient is zero), which may correspond to local minima, saddle points, or maxima~\cite{Jentzen2018StrongEA}, and even seemingly simple non-convex optimization problems (e.g., quadratic programming with non-convex constraints or copositivity testing) are proven to be NP-hard in the general case~\cite{nesterov2008advance}. This fundamental disconnect means that the remarkable practical efficiency of SGD in training DNNs cannot be explained by classical optimization theory, highlighting the need for new theoretical frameworks tailored to the unique properties of DNNs.

Several theoretical directions have emerged in recent years to address this trainability puzzle. One prominent line of work focuses on the infinite-width limit of DNNs, leading to the development of the Neural Tangent Kernel (NTK) framework~\cite{Jacot2018NeuralTK}, which provides valuable insights into the training dynamics of DNNs in the lazy training regime (where network parameters change minimally during training). Another complementary direction, based on Fenchel-Young losses, establishes a direct link between gradient norms and distribution fitting errors in supervised classification tasks~\cite{QiGL25}. However, as we detail in Section~\ref{sec:related_work}, these existing approaches have significant limitations: NTK theory struggles to capture the training dynamics of finite-width DNNs (the setting of practical interest), while the Fenchel-Young perspective is restricted to classification tasks and does not address the generalization properties of DNNs.

\subsection{The generalization paradox}

\textit{Generalization} refers to the ability of DNN models to make accurate predictions on unseen test data after being trained on a finite set of training samples. Classical statistical learning theory quantifies generalization performance by deriving upper bounds on generalization error (the difference between test and training error) using complexity measures of the hypothesis class, such as VC-dimension~\cite{Vapnik2006EstimationOD} or Rademacher complexity~\cite{Bartlett2003RademacherAG}. These classical bounds universally suggest that controlling the size and complexity of the model promotes better generalization performance, as more complex models are more prone to overfitting to noise in the training data. However, in the over-parameterized regime, where DNNs often contain orders of magnitude more parameters than training samples, these classical bounds fail to reflect the strong empirical generalization performance observed in practice, a phenomenon known as the generalization paradox of DNNs.

To explain this paradox, several alternative theoretical frameworks have been proposed in the literature. The \textit{flat minima} hypothesis posits that the inherent stochasticity of SGD acts as an implicit regularizer during training, steering the optimization process toward flat regions of the loss landscape (minima with low curvature) that are empirically associated with better generalization~\citep{KeskarMNST17,Chaudhari2017StochasticGD}. Information-theoretic approaches, most notably the Information Bottleneck (IB) principle, explain generalization through the lens of information compression in neural representations, arguing that DNNs learn to retain only the input information relevant for predicting target labels while discarding redundant or noisy components~\cite{TishbyZ15,ShwartzZiv2017OpeningTB}. Yet as we discuss in Section~\ref{sec:related_work}, each of these perspectives has its own critical limitations: flatness measures are not invariant to network parameterization~\citep{Dinh2017SharpMC} (meaning identical prediction functions can exhibit drastically different flatness values under different parameterizations), and IB theory faces significant practical challenges in estimating mutual information in high-dimensional neural representation spaces.

\subsection{Toward a unified framework}

Current empirical evidence and theoretical studies collectively suggest that the trainability and generalization of DNNs are influenced by multiple interrelated factors, including task type (classification, regression), intrinsic data characteristics (distribution, noise), model architectural design (depth, width, skip connections), optimization algorithm choices (batch size, learning rate), and loss function design. Existing theoretical approaches typically focus on only one or a small subset of these factors, making it difficult to offer a unified and comprehensive perspective on the behavior of DNNs in practical learning scenarios. Given these limitations of existing frameworks, there is an urgent need for a new theoretical foundation that unifies the analysis of trainability and generalization under a single coherent framework, while accounting for the interplay between key influencing factors.

Our key insight that addresses this gap is that all practical machine learning tasks can be fundamentally viewed as problems of conditional distribution estimation: learning the conditional distribution of a target variable given input features from a finite set of training samples. Building on this core insight, we develop \textit{conjugate learning theory} through three logically connected steps:
\begin{enumerate}
    \item \textbf{Conditional distribution view}: We formally argue that any practically learnable task inherently involves estimating the conditional distribution of a target variable $Y$ given input features $X$, rather than merely fitting a deterministic function from $X$ to $Y$.
    \item \textbf{Exponential family necessity}: By the Pitman--Darmois--Koopmans theorem~\cite{Pitman_1936,koopman1936distributions,darmois1935lois}, when the support of the target distribution is known a priori, the only families of distributions that can be consistently estimated from finite samples are exponential families. This fundamental result leads to a notion we term \textit{practical learnability}, which defines the set of tasks that can be effectively solved with limited training data.
    \item \textbf{Fenchel-Young loss}: We prove that maximum likelihood estimation of exponential family distributions is equivalent to minimizing a Fenchel-Young loss function, which thus becomes the mathematical cornerstone of our conjugate learning framework.
\end{enumerate}
In this unified framework, all learning tasks are formalized as problems of estimating a measurable mapping between random variables representing input features and output targets, based on a finite sampled dataset (not necessarily i.i.d.), domain-specific prior knowledge encoded as convex constraints, and a convex conjugate dual space defined over a well-specified hypothesis space of DNN models.

\subsection{Main contributions}

Our main contributions to the theoretical understanding of DNNs are as follows (with corresponding sections indicated for reference):
\begin{enumerate}
    \item We propose \textit{conjugate learning theory}, a novel and comprehensive theoretical framework for modeling machine learning tasks based on DNNs. Within this framework, diverse learning tasks (classification, regression, generative modeling) are unified under a common mathematical formalization that leverages convex conjugate duality. (Section~\ref{sec:framework})
    
    \item We establish a fundamental result that the Fenchel--Young loss is the unique admissible form of the loss function under mild regularity conditions (continuity, differentiability, and consistency with maximum likelihood estimation), and provide a principled methodology for designing task-specific loss functions by incorporating domain-specific prior knowledge as convex constraints. (Section~\ref{sec:framework})
    
    \item We introduce a novel \textit{structure matrix} to quantitatively characterize the trainability of DNNs induced by model architectural features. We prove a key equivalence result: the optimization of the non-convex empirical risk loss for DNNs can be equivalently understood as the problem of minimizing gradient energy (a measure of optimization progress) while controlling the extremal eigenvalues of the structure matrix (which capture the spectral properties of the model). (Section~\ref{sec:opt})
    
    \item We define a \textit{gradient correlation factor} that quantifies the joint influence of data properties, batch size, and model architecture on the convergence rate of mini-batch SGD. Based on this factor, we establish a rigorous convergence theorem that bounds the rate of empirical risk reduction for mini-batch SGD in the conjugate learning framework. (Section~\ref{sec:opt})
    
    \item We use \textit{generalized conditional entropy} to derive both deterministic and probabilistic bounds on generalization error for DNNs. These bounds explicitly quantify the impact of three key factors on generalization performance: information loss from irreversible model operations, loss function scale, and intrinsic data characteristics. (Section~\ref{sec:generalization})
    
    \item We validate the core theoretical predictions on trainability and generalization through extensive controlled experiments on benchmark datasets and standard DNN architectures. The experimental results demonstrate strong quantitative alignment with both our theoretical predictions and key findings from existing theoretical work. (Section~\ref{sec:experiments})
\end{enumerate}

\subsection{Paper organization}

The rest of the paper is organized as follows:
Section~\ref{sec:related_work} surveys related work in detail.  
Section~\ref{sec:pre} introduces fundamental concepts and key lemmas.  
Section~\ref{sec:framework} presents the conjugate learning framework (Contributions 1-2). 
Section~\ref{sec:opt} analyzes trainability through the structure matrix and gradient correlation factor (Contributions 3-4).
Section~\ref{sec:generalization} establishes generalization bounds based on generalized conditional entropy (Contribution 5).
Section~\ref{sec:experiments} validates all theoretical predictions experimentally (Contribution 6).
All technical proofs of theorems, lemmas, and corollaries are provided in full detail in the Appendix.

\section{Preliminaries}
\label{sec:pre}

This section establishes the foundational concepts and notation used throughout the paper. Subsection~\ref{subsec:notation} introduces the basic notation for probability, data, models, and convex analysis. Subsection~\ref{subsec:lemmas} presents supporting lemmas from convex analysis and information theory that will be used in later proofs.

\subsection{Notation}
\label{subsec:notation}

\subsubsection{Basic probability notation}
\begin{itemize}
    \item The random pair $Z = (X, Y)$ follows the distribution $q_{\mathcal{X}\mathcal{Y}}$ (abbreviated as $q$) and takes values in the product space $\mathcal{Z} = \mathcal{X} \times \mathcal{Y}$, where $\mathcal{X}$ denotes the input feature space and $\mathcal{Y}$ is a finite set of targets (labels). The cardinality of $\mathcal{Z}$ is denoted by $|\mathcal{Z}|$. Throughout this paper, we assume that $|\mathcal{X}|$ is finite, a condition that aligns with practical machine learning scenarios.
    \item For clarity and conciseness, we represent distribution functions in vector form:
\begin{equation}
\begin{aligned}
    &q_{\mathcal{Z}} := \bigl(q_{Z}(z_1), \ldots, q_{Z}(z_{|\mathcal{Z}|})\bigr)^\top, \\
    &q_{\mathcal{X}} := \bigl(q_{X}(x_1), \ldots, q_{X}(x_{|\mathcal{X}|})\bigr)^\top, \\
    &q_{\mathcal{Y}|x} := \bigl(q_{Y|X}(y_1|x), \ldots, q_{Y|X}(y_{|\mathcal{Y}|}|x)\bigr)^\top,
\end{aligned}
\end{equation}
where $q_X(\cdot)$ and $q_{Y|X}(\cdot|\cdot)$ denote the marginal and conditional probability mass functions (PMFs), respectively. In this paper, we use $q_{\mathcal{Z}}$, $q_{\mathcal{X}}$, and $q_{\mathcal{Y}|x}$ to denote the vectorized forms of $q_Z(\cdot)$, $q_X(\cdot)$, and $q_{Y|X}(\cdot|x)$, respectively. The conditional distribution $q_{\mathcal{Y}|X}$ is treated as a function parameterized by $X$, denoted as $q_{\mathcal{Y}}(x)$.

    \item Let $\delta_Z$ denote the random variable obtained by mapping elements of $\mathcal{Z}$ to $|\mathcal{Z}|$-dimensional one-hot vectors; we refer to this as the index representation of $Z$.
    
    \item For a given input $x \in \mathcal{X}$, define the conditional mean of $Y$ as $\bar{Y}|x := \sum_{y \in \mathcal{Y}} q_{\mathcal{Y}|x}(y) y$. The quantity $\bar{Y}|X$ is thus a random variable that depends on the input random variable $X$.
\end{itemize}

\subsubsection{Data and empirical distribution}
\begin{itemize}
    \item The sample dataset is denoted by the $n$-tuple 
    \[
    s^n := \{z^{(i)}\}_{i=1}^n = \{(x^{(i)}, y^{(i)})\}_{i=1}^n,
    \]
    where each pair $(x^{(i)}, y^{(i)})$ is drawn from the unknown true joint distribution $q_{\mathcal{X}\mathcal{Y}}$ (abbreviated as $q$), and $n$ is the total number of samples.

    \item Let $\hat{q}_{\mathcal{Z}}^n$ (abbreviated as $\hat{q}$) denote the empirical distribution induced by $s^n$, defined as
    \begin{equation}
        \hat{q}(z) := \frac{1}{n} \sum_{i=1}^n \mathbf{1}_{\{z^{(i)}\}}(z),
    \end{equation}
    for all $z \in \mathcal{X} \times \mathcal{Y}$.

    \item We define an auxiliary random variable $Z' = (X', Y')$ that follows the empirical distribution $\hat{q}$, i.e., $Z' \sim \hat{q}$. This probabilistic representation enables us to treat the finite training dataset $s^n$ as a realization of a well-defined distribution $\hat{q}$, while ensuring all realizations of $Z'$ remain within the product space $\mathcal{X} \times \mathcal{Y}$.
\end{itemize}

\subsubsection{Model and function space}
\begin{itemize}
    \item We use $f_\theta(x)$ (often abbreviated as $f(x)$) to denote a model parameterized by $\theta$ evaluated at input $x$. The function (hypothesis) space associated with input $x$ is defined as $\mathcal{F}_\Theta(x) := \{ f_\theta(x) : \theta \in \Theta \}$. Define $\mathcal{F}(X) = \{ f(X) | f \in \mathcal{F} \}$ as the set of random variables induced by $\mathcal{F}$ and $X$. This paper focuses on theoretical mechanisms of deep learning. Unless otherwise specified, the function space referred to in this work is the space of functions representable by neural network models.
\end{itemize}

\subsubsection{Convex analysis notation}
\begin{itemize}
    \item The indicator function of a convex set $S$ is defined as
    \begin{equation}
        I_S(x) =
        \begin{cases}
            0, & \text{if } x \in S, \\
            +\infty, & \text{otherwise}.
        \end{cases}
    \end{equation}

    \item The convex conjugate (Legendre–Fenchel conjugate) of a function $\Phi$ is defined as
    \begin{equation}
        \Phi^*(\nu) := \sup_{\mu \in \mathrm{dom}(\Phi)} \bigl\{ \langle \mu, \nu \rangle - \Phi(\mu) \bigr\},
    \end{equation}
    where $\langle \cdot, \cdot \rangle$ denotes the standard inner product~\cite{Todd2003ConvexAA}. When $\Phi$ is strictly convex and differentiable, its gradient with respect to $\mu$ is denoted by $\nabla \Phi(\mu)$. We define the \textit{conjugate dual} of $\mu$ with respect to $\Phi$ as $\mu^*_\Phi := \nabla \Phi(\mu)$. When $\Phi(\cdot) = \frac{1}{2}\|\cdot\|_2^2$, we have $\mu_\Phi^* = \mu$.
    \item The Fenchel--Young loss induced by a convex function $\Phi$ is defined as\label{def:fy_loss}
\begin{equation}
    d_{\Phi}(\mu, \nu) := \Phi(\mu) + \Phi^*(\nu) - \langle \mu, \nu \rangle,
\end{equation}
where $\mu \in \mathrm{dom}(\Phi)$ and $\nu \in \mathrm{dom}(\Phi^*)$~\cite{Blondel2019LearningWF}. This formulation plays a central role in our analysis.
\item For a strictly convex function $\Phi$, the Bregman divergence\label{def:bregman} is denoted by $B_\Phi(y, x)$.

\end{itemize}

\subsubsection{Generalized entropy and generalized relative entropy}
\begin{itemize}
    \item The \textit{generalized entropy} \label{def:gen_entropy}of a random variable $Y$ with respect to a convex function $\Phi$ is defined as~\cite{10.1093/acprof:oso/9780199535255.001.0001}
\begin{equation}
    \operatorname{Ent}_\Phi(Y) := \mathbb{E}_Y[\Phi(Y)] - \Phi(\bar{Y}),
\end{equation}
where $\bar{Y} := \mathbb{E}[Y]$. By Jensen's inequality, $\operatorname{Ent}_\Phi(Y) \geq 0$, with equality if and only if $Y$ is almost surely constant (provided $\Phi$ is strictly convex).
\item Let $\bar{Y} | x := \mathbb{E}_{Y \sim q_{\mathcal{Y} | x}}[Y]$ denote the conditional mean of $Y$ given $X = x$. The \textit{generalized conditional entropy} \label{def:gen_cond_entropy}of $Y$ given $X$ is defined as
\begin{equation}
    \operatorname{Ent}_\Phi(Y | X) 
    := \mathbb{E}_{X}\bigl[ \operatorname{Ent}_\Phi(Y | X = x) \bigr]
    = \mathbb{E}_{X,Y}[\Phi(Y)] - \mathbb{E}_{X}\bigl[ \Phi(\bar{Y} | X) \bigr].
\end{equation}
This scalar quantity captures the expected residual uncertainty in $Y$ after observing $X$, measured through the lens of the convex potential $\Phi$.

\end{itemize}

\subsubsection{Matrix notation and norms}
\begin{itemize}
    \item     The maximum and minimum eigenvalues of a matrix \( A \) are denoted by \( \lambda_{\max}(A) \) and \( \lambda_{\min}(A) \), respectively. For a strictly convex function \( \Phi \), we denote the largest and smallest eigenvalues of its Hessian matrix \( \nabla^2 \Phi(z) \) by \( \lambda_{\max}(H_\Phi(\theta)) \) and \( \lambda_{\min}(H_\Phi(\theta)) \), respectively.

    \item For any matrix \( A \in \mathbb{R}^{m \times n} \), its Frobenius norm is defined as
    \[
    \|A\|_{F} = \Bigl( \sum_{i=1}^{m} \sum_{j=1}^{n} |a_{ij}|^2 \Bigr)^{1/2} = \sqrt{\operatorname{tr}(A^{\top} A)}.
    \]
    It corresponds to the Euclidean norm of the vector formed by all entries of \(A\).
\end{itemize}

\subsection{Supporting Lemmas}
\label{subsec:lemmas}

\subsubsection{Statistical Foundations}

\begin{theorem}[Sufficiency Principle]
\label{def:suff_principle}
    If $T(\mathbf{X})$ is a sufficient statistic of $\theta$, then any inference about $\theta$ should depend on the sample $\mathbf{X}$ only through the value $T(\mathbf{X})$. That is, if $\mathbf{x}$ and $\mathbf{y}$ are two sample points such that $T(\mathbf{x})=T(\mathbf{y})$, then the inference about $\theta$ should be the same whether $\mathbf{X}=\mathbf{x}$ or $\mathbf{X}=\mathbf{y}$ is observed~\cite{10.1098/rspa.1937.0109}.
\end{theorem}

\begin{theorem}[Pitman--Darmois--Koopmans]
\label{thm:pitman}
    Among families of probability distributions whose domain does not vary with the parameter being estimated, sufficient statistics with bounded dimensionality (i.e., not growing with $n$) exist only for distributions in the exponential family~\cite{Pitman_1936,koopman1936distributions,darmois1935lois}.
\end{theorem}

\subsubsection{Convex Analysis}

\begin{lemma}[Properties of Convex Conjugate Duality]
\label{lem:fenchel_duality}
For all $\mu,\nu\in \mathbb{R}^k$, the following hold:
\begin{enumerate}
    \item \textbf{Fenchel--Young inequality}~\cite[Proposition 3.3.4]{Todd2003ConvexAA}:
    \begin{equation}
        \Phi(\mu) + \Phi^*(\nu) \ge \langle \mu,\nu\rangle, \quad \forall \mu \in \mathrm{dom}(\Phi),\, \nu \in \mathrm{dom}(\Phi^*).
        \label{eq:fenchel_young_inequality}
    \end{equation}
    Equality holds if and only if $\nu \in \partial \Phi(\mu)$. If, furthermore, $\Phi$ is strictly convex and differentiable, then equality holds iff $\nu = \mu_\Phi^*$.

    \item Let $\Phi = \Psi + I_C$. Then~\cite{Beck2012SmoothingAF}
    \begin{equation}
        \Phi^*(\nu) = \inf_{\nu_1 \in \mathbb{R}^d} \bigl\{ \sigma_C(\nu_1) + \Psi^*(\nu - \nu_1) \bigr\},
        \label{lem:sum_conjugate}
    \end{equation}
    where $\sigma_C(\nu) = \sup_{y \in C} \langle \nu, y \rangle$ is the support function of $C$ (i.e., $I_C^* = \sigma_C$).
    
    \item If $\Phi$ is strictly convex and differentiable, then $\Phi^*$ is also strictly convex and differentiable, and $\nu_\Phi^*\in \mathrm{dom}(\Phi)$ for all $\nu \in  \mathrm{dom}(\Phi^*)$~\cite{Todd2003ConvexAA}. 
    
    \item \textbf{Addition to affine function} \label{lem:addition_affine}
    If $f(x)=g(x)+a^Tx+b$, then $f^*(y)=g^*(y-a)-b$~\cite{Todd2003ConvexAA}. 
\end{enumerate}
\end{lemma}

\subsubsection{Properties of Fenchel-Young Losses}

\begin{lemma}[Properties of Fenchel--Young Losses~\cite{Blondel2019LearningWF}]$\ $
\label{lem:fy_losses}
\begin{enumerate}
    \item \textbf{Non-negativity.} $d_{\Phi}(\mu, \nu) \ge 0$ for any $\mu \in \mathrm{dom}(\Phi)$ and $\nu \in \mathrm{dom}(\Phi^*)$. If $\Phi$ is a proper, lower semi-continuous, convex function, then $d_{\Phi}(\mu, \nu) = 0$ if and only if $\nu \in \partial \Phi(\mu)$. If $\Phi$ is strictly convex and differentiable, then $d_{\Phi}(\mu, \nu) = 0$ iff $\nu = \mu_\Phi^*$.

    \item \textbf{Differentiability.} If $\Phi$ is strictly convex and differentiable, then $d_{\Phi}(\mu, \nu)$ is differentiable in both arguments. In particular, $\nabla_\nu d_{\Phi}(\mu, \nu) = \nu_{\Phi^*}^* - \mu$.

    \item \textbf{Relation to Bregman divergences.} Let $\nu = \mu_\Phi^*$ (i.e., $(\mu, \nu)$ form a dual pair), where $\Phi$ is strictly convex. Then the Bregman divergence satisfies $B_\Phi(y \| \mu) = d_{\Phi}(y, \nu)$. Thus, Fenchel--Young losses can be viewed as a “mixed-form” Bregman divergence~\cite[Theorem 1.1]{amari2016information}, where the second argument is expressed in dual coordinates.
\end{enumerate}
\end{lemma}

\subsubsection{Information-Theoretic Inequalities}

\begin{lemma}[KL Divergence Upper Bound]
\label{lem:kl_upper_bound}
    If $p$ and $q$ are probability densities/masses both supported on a bounded interval $I$, then we have~\cite{1603768}
    \begin{equation}
        D_{\textrm{KL}}(p,q) \leq \frac{1}{\inf_{x\in I} q(x)} \|p-q\|_2^2.
    \end{equation}
\end{lemma}

\begin{lemma}[Pinsker's Inequality~\cite{2017Elements}]
\label{lem:pinsker}
If $p$ and $q$ are probability densities/masses both supported on a bounded interval $I$, then we have
\begin{equation}
    D_{KL}(p\|q) \ge \frac{1}{2\ln 2} \|p-q\|_1^2.
\end{equation}
\end{lemma}

\section{Conjugate learning framework}
\label{sec:framework}
 
In this section, we formalize the machine learning task and derive the conjugate learning framework. We argue that conjugate learning arises not merely as a convenient modeling choice, but as a \textit{necessary} consequence of fundamental statistical principles governing learnability from finite data. Subsection~\ref{subsec:exp_family_learnability} establishes the practical learnability of exponential families. Subsection~\ref{subsec:framework_definition} presents the formal definition of conjugate learning. Subsection~\ref{subsec:components} elaborates on its three core components. Subsection~\ref{subsec:key_quantities} introduces key quantities for analyzing trainability and generalization.

\subsection{Practical learnability and exponential families}
\label{subsec:exp_family_learnability}

A general machine learning problem can be viewed as using observed samples to train a model that accurately predicts a target variable given input features. From a probabilistic perspective, this is equivalent to estimating the conditional distribution $p_{\mathcal{Y}|x}$ of the target $y$ given the feature $x$, based on a finite dataset.
When framing learning as conditional distribution estimation, two key observations emerge:
\begin{enumerate}
    \item The support (i.e., the set of possible values) of the target distribution is typically known \textit{a priori} and independent of the distribution's parameters. For example, classification tasks have targets in a finite label set $\{1,\dots,K\}$, while regression tasks have targets in $\mathbb{R}^d$, both of which are fixed regardless of the conditional distribution parameters.
    
    \item Due to the sufficiency principle, parametric statistical inference must be based on sufficient statistics. To fit and learn the underlying data distribution using a model with a finite number of parameters, the sufficient statistics of the distribution must be finite-dimensional and must not grow with the sample size $n$. In the absence of additional structural assumptions or prior knowledge, accurately learning the conditional probability distribution from finite samples requires that the distribution admit finite-dimensional sufficient statistics. 

\end{enumerate}

The second observation is critical: infinite-dimensional sufficient statistics cannot be fully determined from a finite number of samples, making consistent estimation impossible. This limitation is rigorously characterized by the Pitman--Darmois--Koopmans theorem, which states that among all parametric families with fixed support, only exponential family distributions possess sufficient statistics whose dimension remains bounded as the sample size $n \to \infty$. Consequently, exponential families are the only distributions that are \textit{practically learnable} from finite samples using finite-capacity parametric models.
\begin{proposition}[Practical Learnability]
\label{prop:learnable}
    In the absence of additional structural assumptions or prior knowledge, only distributions in the exponential family are practically learnable from finite samples using parametric models in machine learning.
\end{proposition}

Common examples of exponential-family distributions include the Gaussian, categorical, binomial, multinomial, Poisson, Gamma, Beta, and chi-squared distributions. In contrast, notable non-exponential-family distributions include the uniform, Cauchy, Laplace, Weibull, extreme-value, and hypergeometric distributions, as well as location-scale families without fixed support.
While real-world data may follow non-exponential-family laws, they can often be effectively approximated within the exponential-family framework. Specifically, under i.i.d.\ sampling, any distribution over a bounded domain can be approximated arbitrarily well by a discrete distribution with finite support. The joint distribution of $n$ independent samples from such a discrete distribution is exactly multinomial, a member of the exponential family. Therefore, \textit{by discretizing the target space finely enough, \textit{any} learning task can be reduced to estimating a multinomial (or categorical) conditional distribution with controllable approximation error}. This justifies the use of exponential-family modeling as a universal paradigm for practical machine learning.

\subsection{Formal definition of conjugate learning}
\label{subsec:framework_definition}

\subsubsection{Derivation from exponential families}

Building on the practical learnability of exponential family distributions, we focus on modeling the conditional probability distribution of the label $y$ given the feature $x$ under the exponential family assumption. We parameterize the conditional distribution as an exponential family with natural parameters dependent on the input features:
\begin{equation}
    q_{\mathcal{Y}|x}(y) = h(y) \exp\!\bigl\{ y^\top \eta_x - B(\eta_x) \bigr\},
\end{equation}
where $B(\eta_x) = \log \sum_{y \in \mathcal{Y}} h(y) \exp\{y^\top \eta_x\}$ is the cumulant function (also known as the log-partition function). Under mild regularity conditions (e.g., the support of $y$ is fixed and $h(y) > 0$), the cumulant function $B$ is strictly convex in $\eta_x$. Thus, our learning objective is to estimate the natural parameter $\eta_x$ using a parametric model $f_\theta(x)$, where $\theta$ denotes the model parameters.
The distribution predicted by the model is expressed as:
\begin{equation}
   p_{\mathcal{Y}|x}(y) = \exp\!\bigl\{ -d_\Omega(y, f_\theta(x)) + \log h(y) + \log \Omega(y) \bigr\},
\end{equation}
where $\Omega = B^*$ denotes the convex conjugate of $B$. The negative log-likelihood becomes
\begin{equation}
    \ell(\theta; s) \propto \sum_{(x,y) \in s} d_\Omega(y, f_\theta(x)).
\end{equation}
This result shows that maximum likelihood estimation under an exponential-family model is equivalent to minimizing the sum of Fenchel--Young losses between the target $y$ and the model output $f_\theta(x)$.

\subsubsection{The conjugate learning objective}

\begin{definition}[Conjugate Learning]
\label{def:con_learning_problem}
Given a sample dataset $s^n$, sample spaces $\mathcal{X}$ and $\mathcal{Y}$, prior knowledge encoded as a convex set $C \subseteq \mathbb{R}^d$ containing all feasible target values, a model space $\mathcal{F}_\Theta = \{ f_\theta : \theta \in \Theta \}$, and a differentiable strictly convex function $\Omega$ (the \textit{generating function}), the conjugate learning task is to solve
\begin{equation}
\min_{\theta \in \Theta} \mathbb{E}_{X,Y} \left[ d_\Phi(Y, f_\theta(X)_{\Phi^*}^*) \right],
\end{equation}
where $\Phi(\cdot) = \Omega(\cdot) + I_C(\cdot)$, $I_C$ is the indicator function of $C$, $f_\theta(X)_{\Phi^*}^*$ denotes the conjugate dual of $f_\theta(X)$ with respect to $\Phi$, and $\mathcal{X}$ is a finite set (justified by discretization of continuous features in practice).
\end{definition}

\subsubsection{Key features}

The conjugate learning framework differs from the classical formulation in the following key aspects:
\begin{enumerate}

    \item \textbf{Relaxed i.i.d. assumption}: The framework does not require samples to be strictly i.i.d., accommodating realistic scenarios such as data augmentation and sequential dependencies in natural language or time-series data. In practical machine learning, training data frequently deviates from the i.i.d. assumption due to factors such as data augmentation (e.g., Mixup~\cite{ZhangCDL18} in few-shot learning) or inherent sequential correlations in context-label pairs (e.g., natural language modeling). The conjugate learning framework is designed to handle such non-i.i.d. data distributions while retaining the ability to conduct theoretical analysis under the i.i.d. sampling regime as a special case.
    
    \item \textbf{Explicit prior integration}: Prior knowledge is encoded through the convex constraint set $C$, which guides the learning process by restricting the feasible region of predictions. This extends classical learning settings where prior information is either implicitly embedded in the model architecture or completely ignored. By encoding domain-specific or task-specific prior knowledge into the convex set $C$, the framework leverages structured prior information to improve generalization performance on complex tasks and reduce the hypothesis space complexity to accelerate optimization.
    
    \item \textbf{Conjugate dual prediction}: The final prediction is defined as the conjugate dual of the raw model output, ensuring structural alignment with the target space. Unlike classical frameworks that use the raw model output $f_\theta(x)$ directly for prediction, conjugate learning uses $f_\theta(x)_{\Phi^*}^*$ (the conjugate dual of $f_\theta(x)$ with respect to $\Phi$) as the final prediction. This design imposes meaningful structural constraints on the prediction space that align with the properties of the target distribution (e.g., probability simplex for classification).
    
    \item \textbf{Fenchel-Young loss family}: The loss function is constrained to the Fenchel-Young family, which includes common losses like cross entropy and MSE. Unlike arbitrary loss functions adopted in classical learning frameworks, Fenchel--Young losses possess well-defined structural properties that we demonstrate to play a pivotal role in ensuring both the trainability and generalization of the model in subsequent sections. 
    
    \item \textbf{Finite feature space}: The input feature space $\mathcal{X}$ is assumed to be finite, a condition justified by the discretization of continuous features in practical learning scenarios. While the label space $\mathcal{Y}$ can be infinite-dimensional (e.g., regression in $\mathbb{R}^d$), the feature space $\mathcal{X}$ is required to be finite. This assumption aligns with core learnability theory: consistent learning in continuous feature spaces is only possible when the space admits an effective discretization (e.g., via finite-resolution binning). Otherwise, if arbitrarily nearby inputs may exhibit arbitrarily divergent label distributions, the task becomes fundamentally unlearnable by any model of finite capacity.
\end{enumerate}
 
Figure~\ref{fig:conjugate_learning_structure} illustrates the architecture and processing flow of the conjugate learning framework. Starting from an input $x$, the parametric model $f_\theta$ produces an intermediate output, which is mapped to its conjugate dual via the operator $(\cdot)_{\Phi^*}^*$. The Bregman divergence \( B_\Phi(y, f_\theta(x)_{\Phi^*}^*) \) is then computed between this dual prediction and the target $y$, forming the optimization objective. Notably, in classification tasks, the softmax activation function and cross entropy loss correspond precisely to $(\cdot)_{\Phi^*}^*$ and \( B_\Phi(\mathbf{1}_y, \cdot) \) (respectively) when $\Phi$ is the negative Shannon entropy. In regression tasks, the mean squared error (MSE) loss arises as $ B_\Phi(y, \cdot) $ with $\Phi = \frac{1}{2}\|\cdot\|_2^2$, under which the conjugate dual operator $(\cdot)_{\Phi^*}^*$ reduces to the identity map.
\begin{figure}[ht]
\centering
\includegraphics[width=0.8\columnwidth]{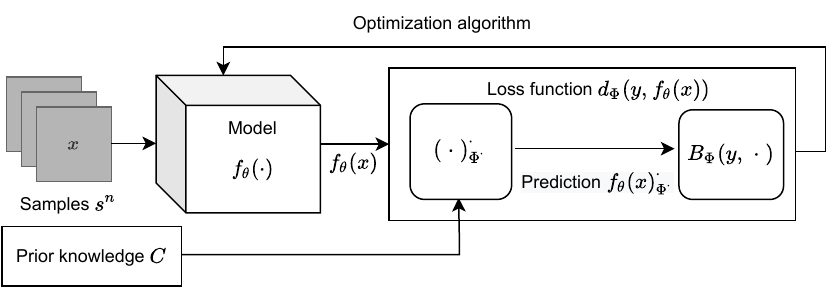}
\caption{Schematic illustration of the conjugate learning framework. The diagram outlines the complete processing pipeline from raw input to learning target approximation, emphasizing the interplay among model output, conjugate transformation, and distance measurement.}
\label{fig:conjugate_learning_structure}
\end{figure}
\FloatBarrier
Fundamentally, conjugate learning is a structured instantiation of the classical statistical learning framework. It preserves core components (training data, hypothesis space, loss function) while introducing a critical augmentation: explicit integration of prior knowledge via the convex set $C$. As shown in subsequent sections, this augmentation enhances analytical tractability and provides a unifying explanation for the empirical success of many deep learning practices.

\subsection{Core components of the framework}
\label{subsec:components}

\subsubsection{Generating function and prior knowledge}
\label{subsubsec:prior_knowledge}

If the true conditional distribution belongs to a known exponential family, its structure uniquely determines the generating function $\Omega$. For example, the Gaussian distribution corresponds to $\Omega = \frac{1}{2}\|\cdot\|_2^2$, while the categorical distribution corresponds to the negative Shannon entropy. If the true distribution does not belong to an exponential family, the empirical distribution can be approximated by a multinomial distribution, whose generating function is the negative Shannon entropy.

In practical machine learning tasks, the support of the target variable $Y$ (denoted $\mathcal{Y}$) is often known or partially characterized by prior knowledge. We express this prior knowledge as $\mathcal{Y} \subseteq C$, where $C$ is a convex set containing all values consistent with the task constraints. A canonical example is supervised classification with $m$ classes: since the goal is to estimate a conditional probability distribution, predictions must lie in the probability simplex $\Delta^m = \{ p \in \mathbb{R}^m : \sum_{i=1}^m p_i = 1, p_i \geq 0 \}$.
If the model architecture is designed such that its predictions inherently satisfy the constraint $C$, the structural mismatch between predictions and targets is eliminated. This restricts the optimization search space, facilitating faster convergence of both empirical and expected risk. Empirically, richer prior knowledge leads to a smaller feasible set $C$, which reduces the hypothesis complexity and improves learning efficiency and accuracy. The ubiquitous softmax layer in classification networks exemplifies this principle: it maps raw outputs from $\mathbb{R}^m$ onto the probability simplex $\Delta^m$, which, from the perspective of conjugate learning, emerges naturally as a mechanism to enforce prior knowledge and align the model's prediction geometry with that of the learning target.

\subsubsection{Conjugate dual prediction}
\label{subsubsec:conjugate_dual}

As discussed in Subsection~\ref{subsubsec:prior_knowledge}, a key challenge is ensuring that model predictions share the structural properties of the learning target. Conjugate learning addresses this by leveraging duality to enforce alignment between predictions and the target space.
Given a generating function $\Omega$ with domain $\mathrm{dom}(\Omega)$, define $\Phi = \Omega + I_C$, so that $\mathrm{dom}(\Phi) = \mathrm{dom}(\Omega) \cap C$. By Lemma~\ref{lem:fenchel_duality}, for any $\nu \in \mathrm{dom}(\Phi^*)$, it holds that $\nabla \Phi^*(\nu) \in \mathrm{dom}(\Phi)$. In the context of Definition~\ref{def:con_learning_problem}, the prediction obtained via dual mapping satisfies
\begin{equation}
f_\theta(x)_{\Phi^*}^* := \nabla \Phi^*(f_\theta(x)) \in \mathrm{dom}(\Phi) = \mathrm{dom}(\Omega) \cap C.
\end{equation}
Moreover, the convex conjugate $\Phi^*$ admits the representation
\begin{equation}
\Phi^*(\mu) = \inf_{\mu' \in \mathbb{R}^d} \left\{ \sigma_C(\mu') + \Omega^*(\mu - \mu') \right\},
\end{equation}
where $\sigma_C(\mu) = \sup_{q \in C} \langle \mu, q \rangle$ is the support function of $C$.

Although computing $\Phi^*$ from $\Omega$ and $C$ is generally nontrivial, many practical constraints are linear. For instance, in classification scenarios where we expect the prediction results to be the distribution of labels given features, the predictions lie within a probability simplex, such that $C=\{y\in \mathbb{R}^d:\sum_i y_i=1 \text{ and } y_i>0\}$. In this case, the constraint condition is linear, and the following theorem provides a closed-form solution.
\begin{theorem}
\label{thm:linear_conjugate}
Let $G \in \mathbb{R}^{m \times d}$ be a matrix with rows $G_i^\top$, and define the constraint set $C = \{ y \in \mathbb{R}^d : G y = b \}$. Then,
\begin{equation}
\Phi^*(\nu) = \Omega^*(\nu + G^\top \lambda_*) - \langle \lambda_*, b \rangle,
\end{equation}
where $\lambda_* \in \mathbb{R}^m$ satisfies $G \cdot (\nu + G^\top \lambda_*)_{\Omega^*}^* = b$.
\end{theorem}
The proof is provided in Appendix~\ref{appendix:proof_linear_conjugate}.

To clarify theorem’s implications, consider \( \Omega(p) = \sum_i (p_i \log p_i - p_i) \), for which \( \Omega^*(\nu) = \sum_i e^{\nu_i} \) and \( (\nu_{\Omega^*}^*)_i = e^{\nu_i} \), which does not lie in the simplex \( \Delta^m \). Imposing \( C = \Delta^m \) and setting \( \Phi = \Omega + I_C \), Theorem~\ref{thm:linear_conjugate} yields
\begin{equation}
\Phi^*(\nu) = 1 + \sum_i e^{\nu_i}.
\end{equation}
Consequently,
\begin{equation}
\nu_{\Phi^*}^* = \mathrm{Softmax}(\nu) = \left[ \frac{e^{\nu_i}}{\sum_j e^{\nu_j}} \right]_i \in C,
\end{equation}
and the associated Bregman divergence becomes
\begin{equation}
d_{\Phi}(p, \nu) = D_{\mathrm{KL}}(p \,\ \, \nu_{\Phi^*}^*).
\end{equation}

This result shows that the softmax function and cross entropy loss emerge naturally within the conjugate learning framework when $\Phi$ is the negative entropy regularized by the probability simplex constraint. Directly fitting $Y$ using the raw model output $f_\theta(x)$ ignores available prior information, leading to a larger search space, reduced training efficiency, and no guarantee that predictions satisfy the constraint $C$. In contrast, conjugate learning provides a principled mechanism to embed prior knowledge into the learning geometry, enhancing both theoretical interpretability and empirical performance.

\subsubsection{Fenchel-Young loss as the natural loss family}
\label{subsubsec:fy_loss}

We now show that Fenchel--Young losses are not merely a consequence of the Pitman--Darmois--Koopmans theorem (Subsection~\ref{subsec:exp_family_learnability}) but are theoretically inevitable under natural desiderata for well-behaved loss functions in machine learning.

It is well established that many standard losses, such as cross entropy, KL divergence, and mean squared error, belong to the Fenchel--Young family~\cite{Blondel2019LearningWF}. In machine learning, a loss function serves two fundamental roles: (i) quantifying discrepancy between predictions and targets, and (ii) guiding optimization toward a globally optimal solution. A well-designed loss should therefore satisfy:
\begin{enumerate}

\item \textbf{Strict convexity in structured predictions}: The loss $\ell(y, f_\theta(x)_{\Phi^*}^*)$ must be strictly convex in the prediction $f_\theta(x)_{\Phi^*}^*$. This ensures a unique minimizer for any fixed input $x$, guaranteeing that optimization converges to a well-defined target. Without this property, the loss landscape may contain multiple minima, preventing reliable convergence. Note that strict convexity in the prediction does not require strict convexity in the raw output $f_\theta(x)$.
\item \textbf{Properness}: A loss function $\ell$ is said to be \emph{proper} if the minimizer of the expected loss coincides with the mean of the target distribution; that is,
\begin{equation}
\bar{Y}|x = \arg\min_{\nu} \, \mathbb{E}_{Y}\bigl[\ell(Y, \nu)\bigr],
\end{equation}
where the expectation is taken with respect to the data-generating distribution of $Y$ given $x$. This property ensures that the population-optimal prediction recovers the true conditional expectation of the target variable. As a consequence, the empirical risk minimizer converges (in probability) to $\bar{Y}|x$ as the sample size grows. Combined with the Weak Law of Large Numbers, properness guarantees that the empirical risk asymptotically approaches the expected risk, enabling consistent learning. Properness also confers inherent robustness to zero-mean label noise: if observed labels are corrupted by additive noise with zero mean, the population-optimal predictor remains centered at the true underlying signal. Without properness, increasing the dataset size does not guarantee convergence to the correct target.

\end{enumerate}

These two properties are necessary and jointly sufficient to characterize well-behaved loss functions for machine learning. In fact, any differentiable loss function satisfying both properties must coincide (up to an additive term independent of the prediction) with a Fenchel--Young loss:
\begin{theorem}
\label{thm:uniqueness_fenchel}
Let $\ell(y, f_\theta(x)_{\Phi^*}^*)$ be a differentiable loss function. If $\ell$ satisfies both (i) strong convexity with respect to the model's prediction and (ii) properness, then there exists a strictly convex function $\Phi$ such that
\begin{equation}
\ell(y, f_\theta(x)_{\Phi^*}^*) - \ell(y, y) = d_\Phi(y, f_\theta(x)).
\end{equation}
\end{theorem}
The proof of this theorem is given in Appendix~\ref{appendix:proof_uniqueness_fenchel}.
 
This result shows that any well-behaved loss function (satisfying strict convexity and properness) is equivalent to a Fenchel--Young loss up to a constant offset $\ell(y,y)$. Consequently, focusing on Fenchel--Young losses entails no loss of generality, providing a rigorous and widely applicable theoretical foundation for loss function choice in machine learning. As a concrete example, the softmax cross entropy loss satisfies:
\begin{equation}
\ell_{\mathrm{CE}}(q_{\mathcal{Y}|x},\, p_{\mathcal{Y}|x}) - \ell_{\mathrm{CE}}(q_{\mathcal{Y}|x},\, q_{\mathcal{Y}|x})
= d_\Phi\!\bigl(q_{\mathcal{Y}|x},\, f_\theta(x)\bigr),
\end{equation}
where $\Phi(p) = -H(p)$ is the negative entropy and $p_{\mathcal{Y}|x} = \mathrm{Softmax}(f_\theta(x)) = f_\theta(x)_{\Phi^*}^*$. Hence, softmax cross entropy arises naturally as the Fenchel--Young loss induced by negative entropy.

\subsection{Key quantities for analysis}
\label{subsec:key_quantities}

\subsubsection{Risk and generalization measures}
We formally define the expected risk, empirical risk, and generalization error within the conjugate learning framework.
\begin{definition}[Risk and Generalization Error]
\label{def:risk}
The expected risk, empirical risk, and generalization error in conjugate learning are defined as
\begin{equation}
\begin{aligned}
\mathcal{R}_{\Phi}(\theta, q) &= \mathbb{E}_{Z \sim q}\bigl[ d_{\Phi}(Y, f_\theta(X)) \bigr], \\
\mathcal{R}_{\Phi}(\theta, \hat{q}) &= \mathcal{R}_{\Phi}(\theta, s) 
= \frac{1}{|s|} \sum_{z \in s} d_\Phi(y, f_\theta(x))
= \mathbb{E}_{Z' \sim \hat{q}}\bigl[ d_{\Phi}(Y', f_\theta(X')) \bigr],\\
\mathrm{gen}(f_\theta, s^n) &= \bigl| \mathcal{R}_\Phi(\theta, q) - \mathcal{R}_\Phi(\theta, \hat{q}) \bigr|,
\end{aligned}
\end{equation}
where $q$ denotes the true data distribution, $\hat{q}$ denotes the empirical distribution induced by the training set $s$, and $s^n$ denotes a training set of size $n$.
\end{definition}

We use both $\mathcal{R}_{\Phi}(\theta, \hat{q})$ and $\mathcal{R}_{\Phi}(\theta, s)$ interchangeably; the latter facilitates analysis of mini-batch SGD.
In the special case $\Phi(y) = \tfrac{1}{2}\|y\|_2^2$, we recover the squared $L_2$ loss:
$$
\mathcal{R}_{L_2^2/2}(\theta, q) = \mathbb{E}_{Z \sim q} \|Y - f_\theta(X)\|_2^2.
$$
For a single sample $z = (x, y)$, we denote the per-sample loss by $\mathcal{R}_\Phi(\theta, z) := d_\Phi(y, f_\theta(x))$.

The range of the loss function varies considerably with the choice of $\Phi$. To enable a fair comparison of model fitting performance across different loss functions, we define the standardized risk based on the mean squared error between the target and the dual prediction:
\begin{definition}[Standardized Expected and Empirical Risk]
\label{def:standardized_risk}
The standardized expected risk and empirical risk in conjugate learning are defined, respectively, as:
\begin{equation}
\begin{aligned}
\mathcal{R}^{\circ}_{\Phi}(\theta,q) &= \mathbb{E}_{Z\sim q}\, \|Y - f_\theta(X)_{\Phi^*}^*\|_2^2,\\
\mathcal{R}^{\circ}_{\Phi}(\theta,\hat{q}) &= \mathcal{R}^{\circ}_{\Phi}(\theta,s) = \mathbb{E}_{Z'\sim \hat{q}}\, \|Y - f_\theta(X)_{\Phi^*}^*\|_2^2.
\end{aligned}
\end{equation}
\end{definition}

To relate the standardized risk to the general Fenchel--Young loss risk, we first define the extremal eigenvalues of the Hessian at both the sample and dataset levels.
\begin{definition}[Extremal Eigenvalues of the Model-Induced Hessian]
\label{def:ex_eigen_fun}
We define the extremal eigenvalues of the model-induced Hessian as
\[
\lambda_{\min}(H_\Phi(\theta)) := \inf_{x \in \mathcal{X}} \lambda_{\min}\!\bigl( \nabla^2 \Phi(f_\theta(x)_{\Phi^*}^*) \bigr), \quad
\lambda_{\max}(H_\Phi(\theta)) := \sup_{x \in \mathcal{X}} \lambda_{\max}\!\bigl( \nabla^2 \Phi(f_\theta(x)_{\Phi^*}^*) \bigr).
\]
Equivalently, for all $x \in \mathcal{X}$,
\begin{equation}
   \lambda_{\min}(H_\Phi(\theta)) \, I \;\preceq\; \nabla^2 \Phi\bigl(f_\theta(x)_{\Phi^*}^*\bigr) \;\preceq\; \lambda_{\max}(H_\Phi(\theta)) \, I.
\end{equation}
These quantities characterize the global lower and upper bounds on the local curvature of $\Phi$ over the set of predictions generated by the model $f_\theta$ on the input domain $\mathcal{X}$. We assume throughout that $0 < \lambda_{\min}(H_\Phi(\theta)) \leq \lambda_{\max}(H_\Phi(\theta)) < \infty$, which holds for standard losses (e.g., squared error, logistic loss) on compact domains.
\end{definition}

The values $\lambda_{\min}(H_\Phi(\theta))$ and $\lambda_{\max}(H_\Phi(\theta))$ depend implicitly on the model parameters $\theta$ through the induced predictions $f_\theta(x)_{\Phi^*}^* $. For instance, when $\Phi$ is the negative Shannon entropy, i.e., $\Phi(p) = \sum_i p_i \log p_i$ defined on the probability simplex, the Hessian $\nabla^2 \Phi(p)$ is diagonal with entries $1/p_i$. In this case, the eigenvalues of $\nabla^2 \Phi(p)$ are precisely $\{1/p_i\}_{i=1}^k$, so the extremal eigenvalues over a dataset correspond to the reciprocals of the smallest and largest predicted probabilities across all classes and samples.
More generally, if the Hessian $\nabla^2 \Phi$ has uniformly bounded eigenvalues over its entire domain (e.g., as in the squared loss, where $\nabla^2 \Phi = I$), then one may take $\lambda_{\min}(H_\Phi(\theta))$ and $\lambda_{\max}(H_\Phi(\theta))$ to be these global constants, independent of $\theta$. Such a simplification does not affect the validity of our theoretical conclusions, as our bounds only require the existence of finite, positive extremal curvature constants, whether derived from the model’s current state or from global properties of $\Phi$.

Since the general Fenchel--Young risk satisfies $\mathcal{R}_\Phi(\theta, z) = d_\Phi(y, f_\theta(x)) = B_\Phi\bigl(y, f_\theta(x)_{\Phi^*}^*\bigr)$, by the integral form of Taylor’s theorem (or the mean-value inequality for strongly convex functions), the Bregman divergence admits the quadratic bound
\begin{equation}
\lambda_{\min}(H_\Phi(\theta)) \, \mathcal{R}^{\circ}_\Phi(\theta, z)
\;\leq\;
\mathcal{R}_\Phi(\theta, z)
\;\leq\;
\lambda_{\max}(H_\Phi(\theta)) \, \mathcal{R}^{\circ}_\Phi(\theta, z).
\end{equation}
Consequently, $\mathcal{R}_\Phi(\theta, z)$ and $\mathcal{R}^{\circ}_\Phi(\theta, z)$ are equivalent up to constant factors depending only on the geometry of $\Phi$ and the current model state $\theta$. This equivalence implies that minimizing one objective also minimizes the other. However, unlike $\mathcal{R}_\Phi$, whose scale is inherently tied to the choice of $\Phi$ (e.g., cross entropy vs. squared loss), the standardized risk $\mathcal{R}^{\circ}_\Phi$ provides a \emph{scale-invariant} measure of prediction error in the primal space. This property makes $\mathcal{R}^{\circ}_\Phi$ particularly useful for comparing optimization dynamics and generalization behavior across different loss functions within a unified framework.

We define the loss upper bound, which quantifies the maximum Fenchel--Young loss incurred by the model over all inputs and targets:
\begin{definition}[Loss Upper Bound]
\label{def:loss_uppper_bound}
The loss upper bound of a model $f_\theta$ is defined as
\begin{equation}
\gamma_\Phi(\theta) = \max_{x \in \mathcal{X},\, y \in \mathcal{Y}} d_\Phi\bigl(y,\, f_\theta(x)\bigr).
\end{equation}
When $ \Phi(p) = \sum_i p_i \log p_i $ (i.e., the negative Shannon entropy) and the task is classification, it follows from the definition that
$$
\gamma_\Phi(\theta) = \max_{x \in \mathcal{X},\, y \in \mathcal{Y}} D_{\mathrm{KL}}\bigl( \mathbf{1}_y \,\|\, p_{\mathcal{Y} | x} \bigr),
$$
where $ p_{\mathcal{Y} | x} = \mathrm{Softmax}(f_\theta(x)) $.
Consequently,
$$
\gamma_\Phi(\theta) = \log \frac{1}{p_{\min}}, \quad \text{with} \quad p_{\min} = \min_{x \in \mathcal{X},\, y \in \mathcal{Y}} p_{\mathcal{Y} | x}(y).
$$

\end{definition}

\subsubsection{Gradient-based quantities}
We define gradient-based quantities that characterize the optimization dynamics of the conjugate learning framework:
\begin{definition}[Gradient Energy]
\label{def:gradient_energy}
For a loss function \(\mathcal{R}_\Phi(\theta, z)\) parameterized by \(\theta\) and a data distribution \(q_{\mathcal{Z}}\), the gradient energy is defined as the expected squared Euclidean norm of the gradient:
\begin{equation}
\mathbb{E}_{Z \sim q_{\mathcal{Z}}} \bigl[ \bigl\| \nabla_\theta \mathcal{R}_\Phi(\theta, Z) \bigr\|_2^2 \bigr].
\end{equation}
For a single sample \(z = (x, y)\), the per-sample gradient energy is given by \(\bigl\| \nabla_\theta \mathcal{R}_\Phi(\theta, z) \bigr\|_2^2\).
\end{definition}

\begin{definition}[Structure Matrix]
\label{def:structure_matrix}
We define the structure matrix associated with the model \(f_\theta\) and input \(x\) as:
\begin{equation}
A_x := \nabla_\theta f_\theta(x) \nabla_\theta f_\theta(x)^\top.
\end{equation}
To characterize the spectral range of these matrices over a dataset \(s\), we introduce the shorthand notation
\begin{equation}
\begin{aligned}
&\lambda_{\min}(A_s) = \min_{(x,y) \in s} \lambda_{\min}(A_x), \\
&\lambda_{\max}(A_s) = \max_{(x,y) \in s} \lambda_{\max}(A_x),
\end{aligned}
\end{equation}
where \(\lambda_{\min}(A_s)\) and \(\lambda_{\max}(A_s)\) denote the smallest and largest eigenvalues observed across all sample-wise structure matrices in \(s\).
\end{definition}

\subsubsection{Information-theoretic quantities}
\label{subsubsec:info_quantities}

\paragraph{Generalized entropy and classical connections}

To connect the generalized entropy (Equation~\ref{def:gen_entropy}) and the generalized conditional entropy (Equation~\ref{def:gen_cond_entropy}) to classical information theory, consider classification tasks where labels are represented as one-hot vectors: for $y \in \{1,\dots,K\}$, let $\mathbf{1}_y \in \Delta^{K-1}$ denote the corresponding vertex of the probability simplex. Then $Y$ may be identified with $\mathbf{1}_Y$, and
\[
\bar{\mathbf{1}}_Y | x=\mathbb{E}[\mathbf{1}_Y| X = x] = q_{\mathcal{Y}|x},
\]
the conditional class distribution.
Under this identification, the generalized conditional entropy becomes
\begin{equation}
\operatorname{Ent}_\Phi(\mathbf{1}_Y|X)
= \mathbb{E}_{X}\Bigl[ \mathbb{E}_{Y|X}[\Phi(\mathbf{1}_Y)] - \Phi\bigl(q_{\mathcal{Y}|X}\bigr) \Bigr].
\end{equation}

When $\Phi(p) = \sum_{i=1}^K p_i \log p_i$ (the negative Shannon entropy), the following identities hold:

\begin{itemize}
    \item For any one-hot vector $\mathbf{1}_y$, we have $\Phi(\mathbf{1}_y) = 0$. Therefore,
    \[
    \mathbb{E}_{Y|X}[\Phi(\mathbf{1}_Y)] = 0.
    \]

    \item The generalized conditional entropy of the one-hot label $\mathbf{1}_Y$ given $X$ becomes
    \begin{equation}
    \begin{aligned}
    \operatorname{Ent}_\Phi(\mathbf{1}_Y | X)
    &= \mathbb{E}_{X}\!\left[ \mathbb{E}_{Y|X}[\Phi(\mathbf{1}_Y)] - \Phi\bigl( \mathbb{E}[\mathbf{1}_Y | X] \bigr) \right] \\
    &= -\,\mathbb{E}_{X}\!\left[ \Phi(q_{\mathcal{Y}|X}) \right] \\
    &= H(Y | X),
    \end{aligned}
    \end{equation}
    since $\Phi(q_{\mathcal{Y}|x}) = \sum_{y} q_{\mathcal{Y}|x}(y) \log q_{\mathcal{Y}|x}(y) = -H(Y | x)$.

    \item The Shannon mutual information arises as a generalized entropy of $q_{\mathcal{Y}|X}$:
    \begin{equation}
    \begin{aligned}
    \operatorname{Ent}_\Phi(q_{\mathcal{Y}|X}) 
    &= -\sum_{x\in\mathcal{X}} q_{\mathcal{X}}(x) H(q_{\mathcal{Y}|x}) - \bigl( -H\bigl(\sum_{x\in\mathcal{X}} q_{\mathcal{X}}(x) q_{\mathcal{Y}|x}\bigr) \bigr) \\
    &= H(Y) - H(Y | X) \\
    &= I(Y; X).
    \end{aligned}
    \end{equation}
\end{itemize}

\paragraph{Information loss}
We define two measures of information loss induced by a feature transformation $g(X)$, which quantify the reduction in predictive information about the target $Y$.
\begin{definition}[Absolute Information Loss]
\label{def:ab_info_loss}
Let $W = g(X)$ with finite support $\mathcal{W}$. The \emph{absolute information loss} of $g$ is
\[
\mathcal{L}(g(X))) := |\mathcal{X}| - |\mathcal{W}|.
\]
This quantifies the reduction in support size due to non-injectivity of $g$.
\end{definition}
The absolute information loss quantifies the reduction in the number of support points of the input random variable due to the non-invertibility of \(g\).

\begin{definition}[Relative Information Loss]
The quantity
\begin{equation}
\mathcal{L}_\Phi(Y|g(X))=\mathbb{E}_X \big[ B_\Phi(\bar{Y}|X,\, \bar{Y}|g(X)) \big]
\end{equation}
is defined as the \emph{relative information loss} induced by the function $g$, where $B_\Phi$ denotes the Bregman divergence associated with a strictly convex function $\Phi$. Equivalently, using the Fenchel--Young loss representation, it can be expressed as
\begin{equation}
\mathcal{L}_\Phi(Y|g(X))=\mathbb{E}_X \big[ d_\Phi(\bar{Y}|X,\, (\bar{Y}|g(X))_{\Phi}^*) \big],
\end{equation}
where $(\cdot)_{\Phi}^*$ denotes the dual mapping induced by $\Phi$.
\end{definition}
This quantity measures the expected discrepancy between the conditional distribution of $Y$ given $X$ and that given $g(X)$. It vanishes if and only if $\bar{Y} | g(X) = \bar{Y} | X$ almost surely, i.e., when $g(X)$ preserves all information relevant to predicting the conditional mean of $Y$.
In particular:
\begin{itemize}
    \item If $g$ is invertible, then $\bar{Y} | g(X) = \bar{Y} | X$, so $\mathcal{L}_\Phi(Y | g(X)) = 0$.
    \item If $g$ is constant (i.e., $g(X) = w_0$ almost surely), then $\bar{Y} | g(X) = \mathbb{E}[Y] =: \bar{Y}$, and
    \[
    \mathcal{L}_\Phi(Y | g(X)) 
    = \mathbb{E}_{X}\!\left[ B_\Phi\bigl( \bar{Y} | X,\; \bar{Y} \bigr) \right]
    = \operatorname{Ent}_\Phi\bigl( \bar{Y} | X \bigr),
    \]
    where $\operatorname{Ent}_\Phi(\bar{Y} | X) := \mathbb{E}_X[\Phi(\bar{Y} | X)] - \Phi(\bar{Y})$ is the generalized entropy of the random variable $\bar{Y} | X$.

\end{itemize}

In classification tasks, labels are represented as one-hot vectors $\mathbf{1}_Y \in \Delta^{K-1}$. In this case, we have
\[
\bar{\mathbf{1}}_Y | X = q_{\mathcal{Y} | X}, \quad
\bar{\mathbf{1}}_Y | g(X) = q_{\mathcal{Y} | g(X)}.
\]
The relative information loss becomes
\begin{equation}
\mathcal{L}_\Phi(\mathbf{1}_Y | g(X)) 
= \mathbb{E}_{X}\!\left[ B_\Phi\bigl( q_{\mathcal{Y} | X},\; q_{\mathcal{Y} | g(X)} \bigr) \right].
\end{equation}

When $\Phi(p) = \sum_{i=1}^K p_i \log p_i$ (the negative Shannon entropy), the Bregman divergence $B_\Phi$ coincides with the Kullback–Leibler divergence:
\[
B_\Phi(p, q) = \mathrm{KL}(p \,\|\, q).
\]
If $g$ is constant, then $q_{\mathcal{Y} | g(X)} = q_{\mathcal{Y}} = \mathbb{E}_X[q_{\mathcal{Y} | X}]$, and thus
\begin{equation}
\mathcal{L}_\Phi(\mathbf{1}_Y | g(X)) 
= \mathbb{E}_{X}\bigl[ \mathrm{KL}(q_{\mathcal{Y} | X} \,\|\, q_Y) \bigr] 
= I(Y; X),
\end{equation}
the Shannon mutual information between $Y$ and $X$.
 
Importantly, relative information loss differs fundamentally from absolute information loss (Definition~\ref{def:ab_info_loss}). While any non-injective $g$ incurs positive absolute loss (due to reduced support size), it may induce zero relative loss if it merges only inputs that yield identical conditional means. For example, suppose $x_1 \ne x_2$ but $\bar{Y} | x_1 = \bar{Y} | x_2$, and define $g$ such that $g(x_1) = g(x_2)$ while injective elsewhere. Then $\mathcal{L}_\Phi(Y | g(X)) = 0$, even though $g$ is non-invertible. Notably, relative information loss and absolute information loss are expected to play a role in generalization analysis, providing new perspectives to address existing limitations.

\section{Trainability in conjugate learning}
\label{sec:opt}

This section presents a systematic theoretical analysis of non-convex optimization within the conjugate learning framework, establishing rigorous foundations for the trainability of DNNs. Our analysis proceeds in four logically interconnected steps, each building on the preceding results to form a complete characterization of trainability:
\begin{itemize}
    \item Subsection~\ref{subsec:risk_bounds} derives tight bounds linking the empirical risk to gradient energy and the extremal eigenvalues of the structure matrix, revealing the core mechanism that enables effective optimization of non-convex objectives.
    \item Subsection~\ref{subsec:gradient_optimization} analyzes how mini-batch SGD minimizes gradient energy, with a focus on the role of batch size, learning rate, and model architecture in convergence behavior.
    \item Subsection~\ref{subsec:structure_control} investigates how architectural choices (e.g., skip connections, network depth, parameter scaling) modulate the spectral properties of the structure matrix, addressing the challenge of maintaining well-conditioned matrices during training.
    \item Subsection~\ref{subsec:fundamental_limits} establishes data-determined lower bounds on the achievable empirical risk, quantifying the fundamental limits of trainability that are intrinsic to the dataset rather than the model or optimization algorithm.
\end{itemize}
Collectively, these results demonstrate that effective non-convex optimization in conjugate learning arises from the joint minimization of gradient energy and the control of the structure matrix's spectral properties.

\subsection{Empirical risk bounds}
\label{subsec:risk_bounds}

We first establish theoretical bounds that connect the empirical risk (both standardized and unstandardized) to the gradient energy and the extremal eigenvalues of the structure matrix. These bounds form the cornerstone of our trainability analysis, as they reveal how gradient-based optimization directly translates to empirical risk reduction in the conjugate learning framework.
 
\begin{theorem}
    \label{thm:base_bound}
Let $Z\sim \hat{q}$. For empirical risk minimization, if $\lambda_{\min}(A_s) \neq 0$, we have
\begin{enumerate}
    \item For the standardized empirical risk,
\begin{equation}
\begin{aligned}
     \frac{1}{\lambda_{\max}(A_x)}\|\nabla_\theta \mathcal{R}_\Phi(\theta, z)\|_2^2
     &\le \mathcal{R}_\Phi^{\circ}(\theta, z)
     \le  \frac{1}{\lambda_{\min}(A_x)}\|\nabla_\theta \mathcal{R}_\Phi(\theta, z)\|_2^2, \\[4pt]
    \frac{1}{\lambda_{\max}(A_s)} \mathbb{E}_Z\|\nabla_\theta \mathcal{R}_\Phi(\theta, Z)\|_2^2
    &\le \mathcal{R}^{\circ}_\Phi(\theta, s)
    \le \frac{1}{\lambda_{\min}(A_s)} \mathbb{E}_Z\|\nabla_\theta \mathcal{R}_\Phi(\theta, Z)\|_2^2,
\end{aligned}
\end{equation}

    \item For the empirical risk,
\begin{equation}
\begin{aligned}
     \frac{\lambda_{\min}(H_\Phi(\theta))}{\lambda_{\max}(A_x)}\|\nabla_\theta \mathcal{R}_\Phi(\theta, z)\|_2^2
     &\le \mathcal{R}_\Phi(\theta, z)
     \le  \frac{\lambda_{\max}(H_\Phi(\theta))}{\lambda_{\min}(A_x)}\|\nabla_\theta \mathcal{R}_\Phi(\theta, z)\|_2^2, \\[4pt]
    \frac{\lambda_{\min}(H_\Phi(\theta))}{\lambda_{\max}(A_s)} \mathbb{E}_Z\|\nabla_\theta \mathcal{R}_\Phi(\theta, Z)\|_2^2
    &\le \mathcal{R}_\Phi(\theta, s)
    \le \frac{\lambda_{\max}(H_\Phi(\theta))}{\lambda_{\min}(A_s)} \mathbb{E}_Z\|\nabla_\theta \mathcal{R}_\Phi(\theta, Z)\|_2^2.
\end{aligned}
\end{equation}
\end{enumerate}
\end{theorem}
The proof of Theorem~\ref{thm:base_bound} is provided in Appendix~\ref{appendix:proof_base_bound}. 
This result establishes that the empirical risk is "sandwiched" between scalar multiples of the gradient energy, where the scaling factors depend exclusively on the extremal eigenvalues of the structure matrix $A_s$. A key implication is that\textit{ as long as these eigenvalues remain uniformly bounded, minimizing the gradient energy necessarily reduces the empirical risk, even for non-convex objectives typical of DNNs.}

The mapping between gradient energy and empirical risk is illustrated in Figure~\ref{fig:mapping}. Under the assumption that the structure matrix $A_s$ is positive definite (i.e., $\lambda_{\min}(A_s) > 0$), a point with zero gradient energy corresponds exactly to a global minimum of the empirical risk. This provides the theoretical foundation for our core claim about DNN trainability in conjugate learning: under controlled structural geometry (well-conditioned $A_s$), non-convex optimization via gradient descent is provably effective at reducing empirical risk to near-optimal levels.

\begin{figure}[ht]
\centering
\includegraphics[width=0.8\columnwidth]{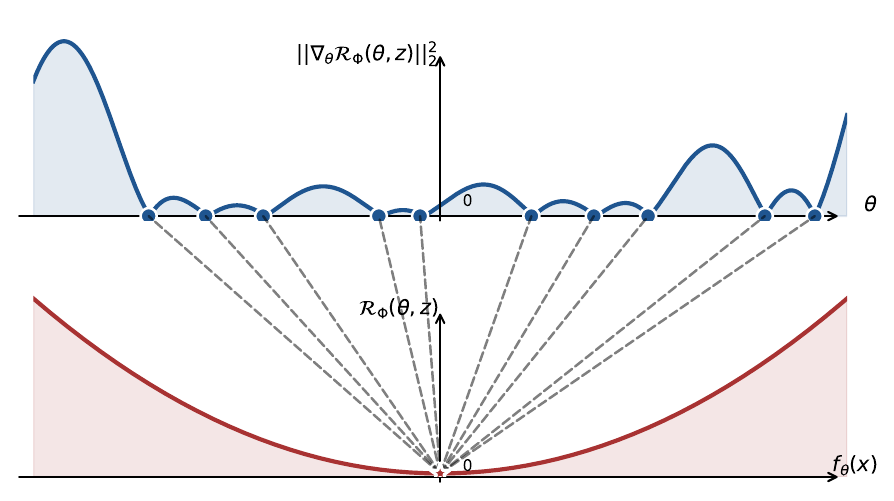}
\caption{Schematic illustration of the mapping between gradient energy and the global optimum of the empirical risk. The top subplot shows the gradient energy as a function of the model parameters $\theta$, with blue dots indicating points where the gradient energy is zero. The bottom subplot depicts the corresponding empirical risk landscape, with its global minimum marked by a pentagram. For positive definite structure matrices, zero gradient energy points coincide with the global minimum of the empirical risk, validating gradient energy minimization as a proxy for empirical risk reduction.}
\label{fig:mapping}
\end{figure}
\FloatBarrier

To ground these general bounds in practical settings, we specialize them to two widely used loss functions in machine learning: MSE and softmax cross entropy.
\begin{corollary}
\label{cor:kl_mse_bound}
Let $Z \sim \hat{q}$.
\begin{enumerate}
  \item For the MSE loss with $\Phi(y) = \frac{1}{2}\|y\|_2^2$,
  \begin{equation}
    \frac{\mathbb{E}_Z \bigl[ \|\nabla_\theta \mathcal{R}_\Phi(\theta, Z)\|_2^2 \bigr]}{2\lambda_{\max}(A_s)}
    \leq \mathcal{R}_\Phi(\theta, s)
    \leq \frac{\mathbb{E}_Z \bigl[ \|\nabla_\theta \mathcal{R}_\Phi(\theta, Z)\|_2^2 \bigr]}{2\lambda_{\min}(A_s)}.
  \end{equation}

  \item For the softmax cross entropy loss with $\Phi(q) = -H(q)$ where $q \in \Delta$,
  \begin{equation}
\frac{\mathbb{E}_Z \bigl[ \|\nabla_\theta \mathcal{R}_\Phi(\theta, Z)\|_2^2 \bigr]}{2\ln 2 \cdot \lambda_{\max}(A_s)}
\leq \mathcal{R}_\Phi(\theta, s)
\leq \frac{\mathbb{E}_Z \bigl[ \|\nabla_\theta \mathcal{R}_\Phi(\theta, Z)\|_2^2 \bigr]}{\min_{i} p_i \cdot \lambda_{\min}(A_s)},
  \end{equation}
  where $p = f_\theta(x)_{\Phi^*}^*$ is the softmax output and $y$ is a one-hot label.
\end{enumerate}
\end{corollary}
The proof of Corollary~\ref{cor:kl_mse_bound} is provided in Appendix~\ref{appendix:proof_kl_mse_bound}. 
For softmax cross entropy, the upper bound of $\mathcal{R}_\Phi(\theta, s)$ is inversely proportional to the minimum predicted probability $p_{\min}$: a larger $p_{\min}$ reduces the gap between the upper bound and the actual empirical risk, meaning gradient energy minimization translates more directly to risk reduction.

Collectively, Theorem~\ref{thm:base_bound} and Corollary~\ref{cor:kl_mse_bound} demonstrate that the empirical risk is tightly controlled by the gradient energy, up to scalar constants determined by the extremal eigenvalues of the structure matrix $A_s$. A critical conclusion for practice is that minimizing the gradient energy serves as a valid proxy for empirical risk minimization in conjugate learning, provided the structure matrix remains well-conditioned (i.e., its condition number $\lambda_{\max}(A_s)/\lambda_{\min}(A_s)$ is bounded by a moderate constant).

\subsection{Gradient energy minimization via SGD}
\label{subsec:gradient_optimization}

Since gradient energy provides a tight upper bound on the empirical risk (Theorem~\ref{thm:base_bound}), we now analyze how mini-batch SGD, the workhorse optimization algorithm for DNNs, minimizes gradient energy. Unlike classical analyses that focus on the stationarity of the empirical risk (i.e., the norm of the full-batch gradient $\|\nabla_\theta \mathcal{R}_\Phi(\theta,s)\|$), our analysis centers on the gradient energy $\mathbb{E}_{Z} \bigl[ \bigl\| \nabla_\theta \mathcal{R}_\Phi(\theta, Z) \bigr\|_2^2 \bigr]$, which captures per-sample gradient behavior and is more directly linked to empirical risk in conjugate learning.

We first formalize the mini-batch SGD update rule for conjugate learning:
\begin{equation}\label{def:basic-sgd}
    \theta_{k+1} = \theta_k - \alpha \nabla_\theta \mathcal{R}_\Phi(\theta_k, s_k),
\end{equation}
where $s_k \subseteq s$ is a randomly sampled mini-batch at iteration $k$, $\theta_k$ denotes the model parameters at iteration $k$, $\alpha > 0$ is the learning rate (step size), and the batch gradient is defined as $\nabla_\theta\mathcal{R}_\Phi(\theta_k,s_k) = \frac{1}{|s_k|}\sum_{z\in s_k} \nabla_\theta \mathcal{R}_\Phi(\theta_k,z)$ (empirical average of per-sample gradients over the mini-batch).

Let $s \setminus s_k$ denote the complement of the mini-batch $s_k$ in the full training set $s$, such that $|s \setminus s_k| = |s| - |s_k|$. To quantify the impact of a mini-batch update on out-of-batch samples, we introduce the gradient correlation factor $M$:
\begin{definition}[Gradient Correlation Factor]
\label{def:gradient_coor_factor}
The gradient correlation factor $M$ quantifies the impact of a mini-batch update on the loss of out-of-batch samples during SGD iterations, defined as
\begin{equation}
    M=  \max_{k}|\mathcal{R}_\Phi(\theta_{k+1}, s \setminus s_k) - \mathcal{R}_\Phi(\theta_k, s \setminus s_k)|,
\end{equation}
where $\theta_{k+1}-\theta_k$ is the update induced by batch $s_k$.
\end{definition}
Intuitively, $M$ measures the correlation between the batch gradient (for $s_k$) and the out-of-batch gradient (for $s \setminus s_k$): if $\nabla_\theta \mathcal{R}_\Phi(\theta_k, s_k)$ is nearly orthogonal to $\nabla_\theta \mathcal{R}_\Phi(\theta_k, s \setminus s_k)$, the update has minimal impact on out-of-batch risk, and $M$ approaches zero. Conversely, aligned gradients lead to large $M$, as the update affects both in-batch and out-of-batch samples strongly.
The gradient correlation factor $M$ is determined by three key hyperparameters and architectural choices:
\begin{enumerate}
    \item \textbf{Learning rate}: A smaller learning rate reduces the magnitude of the parameter update $\theta_{k+1} - \theta_k$, weakening the change in out-of-batch risk and thus decreasing $M$. This aligns with standard practice of using small learning rates for stable convergence.
    
    \item \textbf{Batch size}: Since the batch gradient $\nabla_\theta \mathcal{R}_\Phi(\theta_k, s_k)$ is an empirical average over $s_k$, a larger batch size aligns this gradient more closely with the full-batch gradient $\nabla_\theta \mathcal{R}_\Phi(\theta_k, s)$, and thus with the out-of-batch gradient $\nabla_\theta \mathcal{R}_\Phi(\theta_k, s \setminus s_k)$. This increases $M$, as the update direction is more correlated with out-of-batch samples. Conversely, a batch size of 1 (single-sample SGD) results in a batch gradient that is uncorrelated with the out-of-batch gradient, minimizing $M$.
    
    \item \textbf{Model architecture}: Parameter updates affect out-of-batch samples because model parameters are shared across all samples (e.g., convolutional kernels, fully connected layers). If parameters were fully independent across samples (a hypothetical extreme case), updates would only affect in-batch samples, yielding $M = 0$. Thus, $M$ decreases with the number of non-shared parameters (e.g., embedding layers) and sparser node connectivity, as these reduce cross-sample parameter sharing.
\end{enumerate}

The following theorem characterizes the convergence of mini-batch SGD for gradient energy minimization in conjugate learning:

\begin{theorem}[SGD Convergence for Gradient Energy]\label{thm:sgd_convergence}
Let $n = |s|$ denote the full dataset size, $m = |s_k|$ denote the mini-batch size, and $L = \max_{\theta, z \in s} \lambda_{\max}(\nabla^2_\theta d_\Phi(y, f_\theta(x)))$ denote the Lipschitz constant of the gradient. Applying mini-batch SGD as defined in Equation~\ref{def:basic-sgd} with the optimal learning rate $\alpha = 1/(2L)$ yields:
\begin{enumerate}
\item The expected squared batch gradient norm converges to a neighborhood around zero:
\begin{equation}
       \mathbb{E} \left\| \nabla_\theta \mathcal{R}_\Phi(\theta_k, s_k) \right\|_2^2 \leq \varepsilon^2 + \frac{4L(n - m)}{m}M
\end{equation}
   in $\mathcal{O}(\varepsilon^{-2})$ iterations for any target precision $\varepsilon > 0$.

\item For single-sample batches ($m = 1$, vanilla SGD), the gradient energy converges to:
\begin{equation}
       \mathbb{E}_Z \left\| \nabla_\theta \mathcal{R}_\Phi(\theta, Z) \right\|_2^2 \leq \varepsilon^2 + 4L(n - 1)M
\end{equation}
   in $\mathcal{O}(\varepsilon^{-2})$ iterations.
\end{enumerate}
\end{theorem}
The proof of Theorem~\ref{thm:sgd_convergence} is provided in Appendix~\ref{appendix:proof_sgd_convergence}. 
A key implication of this result is that a smaller gradient correlation factor $M$ leads to tighter convergence of the gradient energy to zero: models with more non-shared parameters (e.g., embedding modules) or sparser connectivity yield smaller $M$, enabling more effective gradient energy minimization via SGD.

Notably, single-sample SGD ($m = 1$) provides the tightest control over gradient energy, as the batch gradient directly reflects per-sample gradient behavior (rather than an average over multiple samples). Since the batch gradient for larger mini-batches is an empirical average, minimizing its norm is less sensitive to individual per-sample gradients, reducing the precision with which the algorithm controls the full gradient energy. By Theorem~\ref{thm:base_bound}, tighter control over gradient energy translates to lower empirical risk, explaining why small batch sizes often facilitate convergence to higher-quality solutions. However, excessively small batch sizes incur non-trivial tradeoffs in optimization efficiency, particularly regarding convergence speed. 
Specifically, when the batch size is too small (e.g., $m=1$), the parameter update direction at each iteration lacks consistency across samples: the stochastic gradient computed from a single sample is highly variable and may deviate significantly from the true full-batch gradient (the direction of steepest descent for the empirical risk). This issue is most pronounced in the early stages of training, when model parameters are far from the optimal solution manifold and the loss landscape is still characterized by high curvature and noisy gradients. 
In this regime, the frequent, erratic parameter updates induced by tiny batches cause the optimization trajectory to oscillate around the descent direction rather than making steady progress toward the minimum, resulting in slower convergence to the region of optimal parameters. 
Thus, the choice of batch size in SGD represents a fundamental tradeoff between two key objectives: (1) small batches enhance the precision of gradient energy control (enabling convergence to better solutions) and (2) moderately larger batches improve convergence speed by reducing gradient variance and aligning update directions with the full-batch gradient.

Combining Theorem~\ref{thm:sgd_convergence} with Theorem~\ref{thm:base_bound}, we derive an upper bound on the empirical risk attainable by mini-batch SGD:
\begin{corollary}[Achievable Empirical Risk]
    \label{cor:sgd_opt_bound}
    For a given gradient correlation factor $M$ of the model, the empirical risk attainable by mini-batch SGD admits the following upper bound:
\begin{equation}
    \begin{aligned}
            \mathcal{R}^\circ_\Phi(\theta, s) &\le \frac{4L(n-1)M}{\lambda_{\min}(A_s)},\\
        \mathcal{R}_\Phi(\theta, s) &\le \frac{4L(n-1)M\lambda_{\max}(H_\Phi(\theta))}{\lambda_{\min}(A_s)},
    \end{aligned}
\end{equation}
where $n = |s|$ and $L = \max_{\theta, z \in s} \lambda_{\max}(\nabla^2_\theta d_\Phi(y, f_\theta(x)))$. 

This upper bound represents the tightest possible upper bound on the empirical risk attainable by mini-batch SGD.
\end{corollary}

In summary, Theorem~\ref{thm:sgd_convergence} establishes that mini-batch SGD effectively minimizes gradient energy, with convergence quality determined by the gradient correlation factor $M$. Smaller $M$, achieved via smaller batch sizes, sparser network connectivity, or increased non-shared parameters, enables tighter convergence to low gradient energy values, which in turn reduces the empirical risk (via Theorem~\ref{thm:base_bound}). This provides a unified theoretical foundation for understanding how optimization hyperparameters (batch size, learning rate) and architectural choices jointly determine DNN trainability.

\subsection{Controlling the structure matrix}
\label{subsec:structure_control}

While SGD effectively minimizes gradient energy, regulating the extremal eigenvalues of the structure matrix $A_s$ is equally critical for achieving strong optimization performance. Directly optimizing the eigenvalues of $A_s$ is computationally infeasible for large-scale DNNs: eigen-decomposition of a $d \times d$ matrix incurs $\mathcal{O}(d^3)$ complexity, and Hessian computations (required for eigenvalue optimization) are prohibitively expensive in terms of memory and computation for high-dimensional parameter spaces. Instead, practical control of $A_s$'s spectral properties is achieved via architectural choices (e.g., skip connections, network depth, overparameterization) that implicitly modulate its eigenvalues.
In this subsection, we analyze how mini-batch SGD, network depth, skip connections, and model size influence the extremal eigenvalues of $A_s$, and connect our results to the classic flat minima theory of DNN optimization.

\subsubsection{Mini-batch SGD and implicit regularization}
\label{subsubsec:sgd}

We first analyze the effect of mini-batch SGD on the structure matrix's eigenvalues, providing a theoretical explanation for the well-known empirical observation that SGD (compared to full-batch gradient descent) induces implicit regularization toward flat minima.

By the chain rule, the per-sample parameter gradient decomposes into two components:
\begin{equation}
\nabla_\theta \mathcal{R}_\Phi(\theta,z) = \nabla_f \mathcal{R}_\Phi(\theta,z) \cdot \nabla_\theta f_\theta(x),
\end{equation}
where $\nabla_f \mathcal{R}_\Phi(\theta,z)$ is the loss gradient with respect to the model output $f_\theta(x)$, and $\nabla_\theta f_\theta(x)$ is the Jacobian of the model output with respect to the parameters $\theta$. During mini-batch SGD training, smaller batch sizes promote the reduction of gradient energy $\mathbb{E}_Z\|\nabla_\theta \mathcal{R}_\Phi(\theta,Z)\|_2^2$, which approaches zero for well-trained models with sufficient capacity (especially when using batch size 1). As $\|\nabla_\theta \mathcal{R}_\Phi(\theta,z)\|_2 \to 0$, the Frobenius norm of the Jacobian $\|\nabla_\theta f_\theta(x)\|_F$ also decreases (since $\|\nabla_f \mathcal{R}_\Phi(\theta,z) \nabla_\theta f_\theta(x)\|_2 = \|\nabla_\theta \mathcal{R}_\Phi(\theta,z)\|_2$), reducing the trace and thus the eigenvalues of $A_x$.
This leads to a key tradeoff in SGD-based optimization:
\begin{enumerate}
    \item Smaller batch sizes reduce the gradient energy, which helps decrease the empirical risk.
    \item Small-batch SGD suppresses the extremal eigenvalues of $A_s$, which impairs the control of empirical risk. This adverse effect becomes more pronounced with increasing network depth, as Jacobian norms decay exponentially with depth (analyzed in Subsusection~\ref{subsubsec:depth_skip}), making it harder to control the empirical risk.
\end{enumerate}

To connect this tradeoff to flat minima theory, we formalize the relationship between the structure matrix's eigenvalues and the curvature of the loss landscape (measured by the Hessian of the per-sample loss $H_z = \nabla^2_\theta \mathcal{R}_\Phi(\theta,z)$). Flat minima are typically defined as regions of the parameter space where $\lambda_{\max}(H_z)$ (the maximum eigenvalue of the loss Hessian) is small, indicating low curvature and stable generalization. The following theorem establishes a direct link between $\lambda_{\max}(H_z)$ and $\lambda_{\max}(A_x)$:
\begin{theorem}
\label{thm:NTK_ERF}
Suppose the model has converged to the population optimal prediction, such that the conjugate dual output matches the conditional expectation of the target:
\begin{equation}
\mathbb{E}_{Y| x}[Y] = f_\theta(x)_{\Phi^*}^*.
\end{equation}
Under this condition, the maximum eigenvalue of the per-sample loss Hessian $H_z$ satisfies the two-sided bound:
\begin{equation}
\lambda_{\min}(G_x) \lambda_{\max}(A_x) \leq \lambda_{\max}(H_z) \leq \lambda_{\max}(G_x) \lambda_{\max}(A_x),
\end{equation}
where $G_x = \nabla^2_f \Phi^*(f_\theta(x))$.
\end{theorem}
The proof of Theorem~\ref{thm:NTK_ERF} is provided in Appendix~\ref{appendix:proof:thm:NTK_ERF}. 
The key assumption ($\mathbb{E}_{Y|x}[Y] = f_\theta(x)_{\Phi^*}^*$) is theoretically well-founded: the properness of Fenchel-Young losses guarantees that the population optimal prediction converges to the conditional expectation of the target, which is achieved by well-trained models. 
This theorem reveals that the maximum curvature of the loss landscape ($\lambda_{\max}(H_z)$) is proportional to the maximum eigenvalue of the structure matrix ($\lambda_{\max}(A_x)$), with the proportionality constant determined by positive-definite symmetric matrix $G_x$.

Since mini-batch SGD suppresses $\lambda_{\max}(A_x)$, it also reduces $\lambda_{\max}(H_z)$, explaining why SGD favors flat minima (low curvature) compared to full-batch gradient descent. However, our framework adds a critical nuance to flat minima theory: while small $\lambda_{\max}(H_z)$ improves generalization, excessively small $\lambda_{\max}(A_x)$ weakens the risk-gradient energy bounds (Theorem~\ref{thm:base_bound}), hindering trainability. 
This tradeoff, between generalization (flat minima) and trainability (well-conditioned $A_s$), is a core challenge in DNN optimization. 
Furthermore, implicit regularization theory and our structure matrix perspective target different objectives: the former leads to flat minima that improve generalization, while the latter reveals that excessive flatness can hinder training.

\subsubsection{Depth and skip connections}
\label{subsubsec:depth_skip}

We now analyze how network depth and skip connections (residual blocks) influence the structure matrix's eigenvalues, showing that skip connections mitigate the exponential decay of Jacobian norms (and thus eigenvalue decay) with increasing depth.

For quantitative analysis, we decompose the model $f_\theta(x)$ into a sequence of $k$ stacked blocks, with the output of the $i$-th block denoted $h^{(i)}$ (and $h^{(0)} = x$, $h^{(k)} = f_\theta(x)$). Let $\theta^{(j)}$ denote the parameters of the $j$-th block, $I$ the identity matrix of appropriate dimension, $\nabla_{h^{(i)}} h^{(i+1)}$ the Jacobian of the $(i+1)$-th block's output with respect to its input, and $\nabla_{\theta^{(j)}} f_\theta(x)$ the Jacobian of the final output with respect to the $j$-th block's parameters.

For a standard feedforward network (no skip connections), the chain rule gives:
\begin{equation}
\label{eq:noskip}
\begin{aligned}
    \nabla_{\theta^{(j)}} f_\theta(x) &= \left( \prod_{i=j}^{k-1} \nabla_{h^{(i)}} h^{(i+1)} \right) \nabla_{\theta^{(j)}} h^{(j)}.
\end{aligned}
\end{equation}
Under standard random initialization, parameter norms are small (e.g., He initialization for ReLU networks), gradient energy is near zero, implying the entries of $\nabla_{h^{(i)}} h^{(i+1)}$ are much smaller than 1. The product of these Jacobians decays exponentially with the number of blocks $k - j$, causing $\|\nabla_{\theta^{(j)}} f_\theta(x)\|_F$ to vanish as depth increases. This leads to the eigenvalues of $A_x$ approaching zero, which weakens the risk-gradient energy bounds (Theorem~\ref{thm:base_bound}) and impairs trainability.

Residual blocks with skip connections~\cite{He2015DeepRL} address this decay by adding an identity skip path to each block. For a residual network $g_\theta(x)$, the chain rule becomes:
\begin{equation}\label{eq:skip}
\begin{aligned}
    \nabla_{\theta^{(j)}} g_\theta(x)  &= \left( \prod_{i=j}^{k-1} \left( \nabla_{h^{(i)}} h^{(i+1)} + I \right) \right) \nabla_{\theta^{(j)}} h^{(j)}.
\end{aligned}
\end{equation}
Since the entries of $\nabla_{h^{(i)}} h^{(i+1)}$ are much smaller than 1, we have $\nabla_{h^{(i)}} h^{(i+1)} + I \approx I$, so $\nabla_{\theta^{(j)}} g_\theta(x) \approx \nabla_{\theta^{(j)}} h^{(j)}$. This prevents the Jacobian norm from decaying with depth, preserving the eigenvalues of the structure matrix.

We formalize this insight in the following proposition:
\begin{proposition}[Role of Skip Connections]
    \label{prop:skip_con}
For deep residual networks in the conjugate learning framework:
\begin{enumerate}
    \item Skip connections mitigate the exponential decay of the structure matrix's extremal eigenvalues with increasing network depth, preserving the tightness of the risk-gradient energy bounds (Theorem~\ref{thm:base_bound}) and maintaining trainability.
    \item If skip connections ensure the singular values of $\nabla_{h^{(i)}} h^{(i+1)} + I$ are greater than 1 (a mild condition satisfied by standard residual blocks with ReLU activations), increasing network depth \emph{increases} the extremal eigenvalues of the structure matrix, strengthening the risk-gradient energy bounds and improving trainability.
\end{enumerate}
\end{proposition}
This proposition provides a theoretical justification for the empirical success of residual networks: skip connections not only prevent vanishing gradients (a well-known benefit) but also preserve the conditioning of the structure matrix, ensuring gradient energy minimization translates directly to empirical risk reduction, even for extremely deep networks.

\subsubsection{Parameter independence and overparameterization}
\label{subsubsec:independence}

We now analyze how model size (overparameterization) influences the structure matrix's spectral properties, using gradient independence as a bridge to connect overparameterization to trainability. 

We first introduce an idealized assumption that characterizes the gradient properties of randomly initialized DNNs:
\begin{assumption}[Gradient Approximate Independence (GAI) Assumption]
For DNNs with randomly initialized parameters (e.g., He or Xavier initialization), the gradients of each output dimension with respect to the model parameters are approximately statistically independent. 
\end{assumption}
The GAI assumption is an idealized simplification of real DNN gradients (which exhibit weak correlations due to shared parameters) but provides a tractable framework to analyze overparameterization. Under this assumption, \citet{QiGL25} proved the following result linking model size to the structure matrix's condition number:
\begin{theorem}
\label{thm:number_para}
Let $k = |f_\theta(x)|$ denote the dimension of the model output and $m = |\theta|$ denote the dimension of the parameter space, with $k \leq m - 1$. 
If the GAI assumption holds and each column of the Jacobian matrix $\nabla_\theta f(x)$ has an $\ell_2$-norm of $\epsilon$, then with probability at least $1 - O(1 / k)$, the condition number of the sample-wise structure matrix $A_x$ satisfies:
\begin{equation}
    \frac{\lambda_{\max}(A_x)}{\lambda_{\min}(A_x)} \le \zeta(m,k) + 1,
\end{equation}
where $\zeta(m,k) = \frac{2|y|\sqrt{6 \log k}}{\sqrt{m - 1}} \left(1 - \frac{2 \log k}{m}\right)^2$ is a decreasing function of the parameter dimension $m$ and an increasing function of the output dimension $k$.
\end{theorem}
Theorem~\ref{thm:number_para} reveals a key overparameterization effect: as the number of parameters $m$ increases (relative to the output dimension $k$), the condition number of $A_x$ decreases, meaning its extremal eigenvalues become more balanced. For sufficiently overparameterized models (large $m$), the condition number approaches 1 (identity matrix), so the upper and lower bounds of the standardized empirical risk (Theorem~\ref{thm:base_bound}) coincide. In this regime, mini-batch SGD alone is sufficient to perfectly control the empirical risk, as gradient energy minimization directly translates to risk reduction with no scaling ambiguity.
This regime aligns with the "lazy training" regime in NTK theory~\cite{Jacot2018NeuralTK,Chizat2018OnLT}, where parameters change minimally during training. However, our analysis offers two critical advantages over NTK theory: (1) it quantifies the impact of both parameter count and output dimension on trainability (via the condition number of $A_x$), and (2) it does not require the assumption of infinitely wide networks. While the GAI assumption is idealized (real gradients exhibit weak correlations), it captures the core stochasticity of overparameterized DNNs, providing a quantifiable framework to explain why overparameterization improves trainability (by balancing the structure matrix's eigenvalues). We validate this conclusion experimentally in Section~\ref{sec:experiments}.

\subsection{Fundamental limits: data-determined lower bounds}
\label{subsec:fundamental_limits}

In the preceding subsections, we analyzed how gradient energy minimization and structure matrix control enable effective optimization of the empirical risk. We now establish the fundamental limits of trainability, data-determined lower bounds on the empirical risk that are intrinsic to the dataset and cannot be overcome by any model or optimization algorithm. These bounds formalize the classic assertion that "data determines the upper limit of machine learning performance, while models and algorithms determine how close we can get to this limit."

The following theorem provides tight upper and lower bounds on the empirical risk in conjugate learning, with the lower bound determined exclusively by the dataset's information-theoretic properties:
\begin{theorem}[Data-Determined Bounds on Empirical Risk]
\label{thm:fitting_bounds}
In the conjugate learning framework, the empirical risk satisfies the following bounds:
\begin{equation}
\begin{aligned}
        \gamma_\Phi(\theta) \geq \mathcal{R}_\Phi(\theta, s) \geq \mathrm{Ent}_\Phi(Y'|X'),
\end{aligned}
\end{equation}
where $(X', Y') \sim \hat{q}$.
The lower bound is achieved if and only if $f_{\theta}(X') = (\bar{Y}' | X')_\Phi^*$.
\end{theorem}
The proof of this theorem is provided in Appendix~\ref{appendix:proof_fitting_bounds}. 
The upper bound $\gamma_\Phi(\theta)$ is determined by the model's maximum prediction error (worst-case loss over the dataset), while the lower bound $\mathrm{Ent}_\Phi(Y'|X')$ is an intrinsic property of the dataset, quantifying the irreducible uncertainty in predicting $Y'$ from $X'$. This lower bound is the "Bayes risk" of the conjugate learning framework: it represents the minimal achievable empirical risk, even for an optimal model that perfectly captures the conditional distribution of $Y'$ given $X'$.

For supervised classification tasks (one-hot labels $\mathbf{1}_{Y'} \in \Delta^K$), we specialize this theorem to information-theoretic bounds: 
In the context of classification tasks, we obtain the following corollary.
\begin{corollary}[Information-Theoretic Bounds for Classification]
\label{cor:Fitting_core_bound_prob}
In supervised classification, the following bounds hold:
\begin{equation}
\gamma_\Phi(\theta) \geq \mathcal{R}_\Phi(\theta, s) \geq \operatorname{Ent}_\Phi(\mathbf{1}_{Y'} | X').
\end{equation}
The lower bound is tight if and only if $f_\theta(X')_{\Phi^*}^* = \hat{q}_{\mathcal{Y'} | X'}$.
\end{corollary}

For classification tasks with the negative Shannon entropy as the generating function ($\Phi(p) = \sum_i p_i \log p_i$), the generalized conditional entropy reduces to the classical conditional Shannon entropy $H(Y' | X')$, and the loss upper bound $\gamma_\Phi(\theta)$ simplifies to the log of the reciprocal of the minimum predicted probability (Definition~\ref{def:loss_uppper_bound}). This yields the following corollary, which connects conjugate learning to classical information theory:
\begin{corollary}
\label{cor:kl_mutual_info_bound}
For $ \Phi(q) = \sum_i q_i \log q_i $, in the setting of classification tasks, we have
\begin{equation}
\label{eq:info_bound}
\log \frac{1}{p_{\min}} > \mathcal{R}_{\Phi}(\theta, s) \geq H(Y' | X'),
\end{equation}
where  $p_{\min} = \min_{x \in \mathcal{X}, y \in \mathcal{Y}} p_{\mathcal{Y} | x}(y)$ , $ H(Y' | X') $ denotes the conditional entropy and $ (X', Y') \sim \hat{q} $ is drawn from the empirical distribution.
\end{corollary}
Corollary~\ref{cor:kl_mutual_info_bound} establishes a direct link between conjugate learning and Shannon information theory: the conditional entropy $H(Y' | X')$ provides a fundamental, data-determined lower bound on the empirical risk (cross entropy loss), while the upper bound is governed by $p_{\min}$. A key implication is that reducing the empirical risk below $H(Y' | X')$ is impossible, even for perfect models, since this bound represents the irreducible uncertainty in the dataset. Conversely, increasing $p_{\min}$ reduces the upper bound, bringing the empirical risk closer to the fundamental lower bound.

\section{Generalization in conjugate learning}
\label{sec:generalization}

This section establishes rigorous generalization guarantees within the conjugate learning framework, bridging machine learning and information theory through the concept of generalized conditional entropy. We derive two complementary classes of generalization bounds, each addressing distinct aspects of generalization behavior:
\begin{itemize}
    \item \textbf{Deterministic bounds} (Subsection~\ref{subsec:det_bounds}): Valid for arbitrary sampling schemes, these bounds characterize the absolute feasible range of generalization error in terms of model properties (maximum loss, information loss) and intrinsic data properties (generalized conditional entropy).
    \item \textbf{Probabilistic bounds} (Subsection~\ref{subsec:prob_bounds}): Under i.i.d. sampling assumptions, these bounds quantify how sample size, information loss induced by the model, and distributional smoothness of the true data affect the probability of achieving a target generalization error.
\end{itemize}
Subsection~\ref{subsec:reg_interpretation} provides a theoretical interpretation of standard regularization techniques through the lens of these bounds, linking parameter norm constraints to the key generalization-controlling term $\gamma_\Phi(\theta)$. Subsection~\ref{subsec:eval_perspectives} further discusses practical implications for evaluating generalization performance, proposing information-theoretic metrics as alternatives to traditional test-set based evaluation.

\subsection{Deterministic bounds: architecture-independent guarantees}
\label{subsec:det_bounds}

We first establish deterministic bounds on generalization error that hold for any sampling scheme, without requiring i.i.d. assumptions or distributional constraints on the underlying data. These bounds are architecture-independent, relying solely on fundamental properties of the model and the dataset.
\begin{theorem}[Deterministic Generalization Bounds]
    \label{thm:det_gen_bound}
The deterministic generalization error is bounded as follows.
If $\mathcal{R}_\Phi(\theta, q) \geq \mathcal{R}_\Phi(\theta, \hat{q})$, then
\begin{equation}
    \mathrm{gen}(f_\theta,s^n) \leq \gamma_\Phi(\theta) - \mathrm{Ent}_\Phi(Y' | X') - \mathcal{L}_\Phi(Y'|f_\theta(X')).
\end{equation}
If $\mathcal{R}_\Phi(\theta, q) < \mathcal{R}_\Phi(\theta, \hat{q})$, then
\begin{equation}
    \mathrm{gen}(f_\theta,s^n) \leq \gamma_\Phi(\theta) - \mathrm{Ent}_\Phi(Y | X) -  \mathcal{L}_\Phi(Y|f_\theta(X)).
\end{equation}
\end{theorem}
The proof of Theorem~\ref{thm:det_gen_bound} is provided in Appendix~\ref{appendix:proof_det_gen_bound}. 
This result characterizes the fundamental feasible range of generalization error for any model in the conjugate learning framework, with three key insights:
\begin{enumerate}
    \item Reducing the maximum loss $\gamma_\Phi(\theta)$ tightens the upper bound on generalization error, as it limits the maximum possible deviation between population and empirical risk.
    \item A larger relative information loss, characterized by both $\mathcal{L}_\Phi(Y'|f_\theta(X'))$ and $\mathcal{L}_\Phi(Y|f_\theta(X))$ induced by the model, leads to a smaller upper bound on the generalization error. 
    \item From the data perspective, larger generalized conditional entropies $\mathrm{Ent}_\Phi(Y | X)$ (population) and $\mathrm{Ent}_\Phi(Y' | X')$ (empirical) correspond to smaller generalization error, as they capture higher intrinsic uncertainty in the data that cannot be eliminated by any model.
\end{enumerate}

Since gradient descent-based optimization algorithms are designed to minimize the empirical risk $\mathcal{R}_\Phi(\theta, \hat{q})$, the condition $\mathcal{R}_\Phi(\theta, q) \geq \mathcal{R}_\Phi(\theta, \hat{q})$ holds for well-trained models in most practical scenarios. For such models, the generalization error simplifies to the following tractable upper bound:
\begin{equation}
    \mathrm{gen}(f_\theta,s^n) \leq \gamma_\Phi(\theta) - \mathrm{Ent}_\Phi(Y' | X') - \mathcal{L}_\Phi(Y'|f_\theta(X')).
\end{equation}
Notably, all terms in this bound are computable from the model and the training dataset (no access to the true population distribution is required), enabling exact calculation of the generalization error upper bound for a given model and training set.

To ground these abstract bounds in practical classification tasks, we specialize them to the softmax cross entropy loss, connecting the deterministic bound to classical Shannon information theory.
\begin{corollary}[Deterministic Generalization Bound for Classification]
    \label{cor:all_bound_prob}
Let $ p_{\mathcal{Y} | x} = \mathrm{Softmax}(f_\theta(x)) $ denote the model's predicted class probability distribution for input $x$, and let $p_{\min} = \min_{x \in \mathcal{X},\, y \in \mathcal{Y}} p_{\mathcal{Y} | x}(y)$ denote the minimum predicted probability over all input-label pairs. For classification tasks with softmax cross entropy loss, the deterministic generalization error satisfies:
\begin{enumerate}
    \item If $\mathcal{R}_\Phi(\theta, q) \geq \mathcal{R}_\Phi(\theta, \hat{q})$, then
    \begin{equation}
        \mathrm{gen}(f_\theta,s^n) \le \log \frac{1}{p_{\min}}  - H(Y' ) + I(Y';X') - \mathcal{L}_\Phi(Y'|f_\theta(X')).
    \end{equation}
    \item If $\mathcal{R}_\Phi(\theta, q) < \mathcal{R}_\Phi(\theta, \hat{q})$, then
    \begin{equation}
        \mathrm{gen}(f_\theta,s^n) \le \log \frac{1}{p_{\min}} - H(Y ) + I(Y;X) - \mathcal{L}_\Phi(Y|f_\theta(X)).
    \end{equation}
\end{enumerate}
Here, $H(\cdot)$ denotes the Shannon entropy and $I(\cdot;\cdot)$ denotes the Shannon mutual information between random variables.
\end{corollary}
This corollary establishes a direct link between conjugate learning-based generalization bounds and classical Shannon information theory. The term $p_{\min}$ characterizes the smoothness of the model's predictive distribution over the feature space: a smoother predictive distribution (larger $p_{\min}$, no extremely low-probability predictions) is more conducive to controlling generalization error. This conclusion aligns with complexity-based generalization bounds in classical learning theory, which link model simplicity to better generalization. Simple models often fail to perfectly fit one-hot encoded labels (introducing bias) but produce smoother predictive distributions; our bound formalizes this tradeoff, connecting the classical bias-variance tradeoff to information-theoretic quantities in conjugate learning.

From the mutual information perspective, a larger mutual information $I(Y;X)$ between the target $Y$ and features $X$ leads to a larger upper bound on generalization error, a result consistent with empirical observations in practice. Higher mutual information indicates labels are more sensitive to feature variations, meaning learning from finite samples is more likely to lose critical predictive information, thereby increasing generalization error.

In summary, deterministic bounds define the fundamental range of generalization error purely in terms of model and data properties (no distributional assumptions). They reveal that generalization is governed by three core factors: the model's maximum loss, the data's intrinsic uncertainty (generalized conditional entropy), and the information loss induced by the model.

\subsection{Probabilistic bounds: sample-dependent guarantees}
\label{subsec:prob_bounds}

While deterministic bounds hold universally for any sampling scheme, they do not leverage the statistical properties of i.i.d. sampling, a standard assumption in machine learning. Under i.i.d. sampling, we can derive probabilistic bounds that quantify how sample size, information loss, and distributional smoothness affect the likelihood of achieving small generalization error.

\begin{theorem}[Probabilistic Generalization Bound]
    \label{thm:prob_gen_bound}
Assume training samples $s^n$ are drawn i.i.d. from the true joint distribution $q_{\mathcal{Z}}$ over the feature-label space $\mathcal{Z} = \mathcal{X} \times \mathcal{Y}$, and the label space $\mathcal{Y}$ has finite cardinality ($|\mathcal{Y}| < \infty$). For any target generalization error threshold $\varepsilon > 0$, the probability that the generalization error exceeds $\varepsilon$ is upper bounded by:
\begin{equation}
    \Pr\bigl( \mathrm{gen}(f_\theta, s^n) \geq \varepsilon \bigr)
    \leq \frac{(|\mathcal{X}| - \mathcal{L}(f_\theta(X))) |\mathcal{Y}| \, \gamma_\Phi(\theta)^2 \bigl(1 - \|q_{\mathcal{Z}}\|_2^2 \bigr)}{4 n \varepsilon^2}.
\end{equation}
Here, $\mathcal{L}(f_\theta(X)) = |\mathcal{X}| - |f_\theta(\mathcal{X})|$ denotes the absolute information loss and $n = |s^n|$ is the sample size.
\end{theorem}
The proof of Theorem~\ref{thm:prob_gen_bound} is provided in Appendix~\ref{appendix:proof_prob_gen_bound}. 
This probabilistic bound integrates three types of information: the true data distribution $q_{\mathcal{Z}}$, the training sample $S^n$, and the model $f_\theta$. Beyond well-known factors (sample size $n$, maximum loss $\gamma_\Phi(\theta)$), it identifies two novel critical factors: absolute information loss $\mathcal{L}(f_\theta(X))$ (model irreversibility) and distributional smoothness $1 - \|q_{\mathcal{Z}}\|_2^2$ (intrinsic data properties). We analyze the impact of each factor on the generalization error distribution below:
\begin{itemize}
    \item \textit{Sample size $n$}: A larger sample size $n$ increases the probability of achieving a small generalization error, a result rooted in the Central Limit Theorem (CLT). Finite sampling is the fundamental source of generalization error: as $n$ increases, the empirical distribution $\hat{Q}_{\mathcal{Z}}^n$ (induced by i.i.d. samples) converges to the true distribution $q_{\mathcal{Z}}$. Formally, the MSE between the empirical and true distributions, $\mathbb{E}\big[ \| q_{\mathcal{Z}} - \hat{Q}_{\mathcal{Z}}^n \|_2^2 \big]$, scales inversely with $n$. This reduces the deviation between empirical and population risk, tightening the probabilistic control over generalization error. In the limit as $n \to \infty$, the empirical distribution converges to the true distribution almost surely, and the generalization error tends to zero.
    
    \item \textit{Maximum loss $\gamma_\Phi(\theta)$}: A smaller maximum loss $\gamma_\Phi(\theta)$ increases the probability of small generalization error, consistent with the deterministic bound. $\gamma_\Phi(\theta)$ defines the upper limit of the loss function over all input-output pairs, and a smaller value indicates more stable loss behavior (no extreme loss values for any sample). In the probabilistic bound, $\gamma_\Phi(\theta)$ appears as a squared term in the numerator: reducing $\gamma_\Phi(\theta)$ directly lowers the upper bound on the probability of exceeding the target error $\varepsilon$. For softmax cross entropy loss, $\gamma_\Phi(\theta) = \log(1/p_{\min})$: a larger $p_{\min}$ reduces $\gamma_\Phi(\theta)$ and improves probabilistic generalization guarantees.
    
    \item \textit{Absolute information loss $\mathcal{L}(f_\theta(X))$}: A larger absolute information loss increases the probability of small generalization error, a result distinct from the deterministic bound (which focuses on relative information loss). $\mathcal{L}(f_\theta(X))$ quantifies the compression of the input feature space by the model: larger values mean more input features are mapped to identical output representations, reducing the effective support size $|f_\theta(\mathcal{X})|$. In the bound, the effective cardinality $(|\mathcal{X}| - \mathcal{L}(f_\theta(X))) |\mathcal{Y}|$ replaces the full joint space cardinality $|\mathcal{Z}| = |\mathcal{X}| |\mathcal{Y}|$. Reducing this effective cardinality simplifies the distribution estimation problem, enabling more accurate approximation of the true distribution with finite samples. This formalizes a key insight: moderate model irreversibility (controlled feature compression) improves probabilistic generalization performance.
    
    \item \textit{Distributional smoothness $1 - \|q_{\mathcal{Z}}\|_2^2$}: Less smooth true data distributions (larger values of $1 - \|q_{\mathcal{Z}}\|_2^2$) increase the probability of small generalization error. This can be illustrated with an extreme case: a Dirac delta distribution $q_{\mathcal{Z}} = \delta_{z_0}$ (concentrating all probability mass on a single point) has $1 - \|q_{\mathcal{Z}}\|_2^2 = 0$, leading to zero generalization error (empirical and true distributions are identical). In contrast, smoother distributions (uniform over many points) have larger $1 - \|q_{\mathcal{Z}}\|_2^2$, making finite samples more likely to misestimate the true distribution, increasing generalization error. Distributional smoothness thus reflects the difficulty of approximating the true distribution with finite samples: less smooth distributions are easier to estimate accurately, improving generalization guarantees.
\end{itemize}

Probabilistic bounds complement deterministic bounds by quantifying the statistical behavior of generalization error under i.i.d. sampling. They reveal that generalization depends not only on model capacity but also on the interplay between model architecture (via absolute information loss), model output stability (via maximum loss), and intrinsic data properties (via distributional smoothness).

\subsection{Relationship between deterministic and probabilistic bounds}
\label{subsec:bound_relationship}

Deterministic and probabilistic bounds serve complementary roles in characterizing generalization in conjugate learning:
\begin{itemize}
    \item \textbf{Deterministic bounds} define the \textit{feasible range} of generalization error for any sampling scheme. They depend exclusively on model and data properties (no sample size or distributional assumptions) and answer the question: "What values of generalization error are possible?"
    \item \textbf{Probabilistic bounds} characterize the \textit{statistical distribution} of generalization error within this feasible range under i.i.d. sampling. They quantify how sample size and distributional smoothness affect the likelihood of achieving small error, answering the question: "How likely is it to achieve a given generalization error with finite data?"
\end{itemize}
Together, these bounds provide a complete characterization of generalization: deterministic bounds establish the absolute limits of what is achievable, while probabilistic bounds quantify how close we can get to these limits with practical sample sizes.
 
\subsection{Regularization as control of loss upper bound}
\label{subsec:reg_interpretation}

Having established that the maximum loss $\gamma_\Phi(\theta)$ is a central control variable in both deterministic and probabilistic generalization bounds, we now connect standard regularization techniques (e.g., $L_2$ regularization) to the control of $\gamma_\Phi(\theta)$. We focus on $L_2$ regularization, as it is the most widely used form of parameter norm constraint in deep learning, and extend the result to general $L_p$ regularization.

We first establish a local equivalence between the parameter norm and the conjugate function $\Phi^*$, which forms the foundation for linking regularization to $\gamma_\Phi(\theta)$.
\begin{lemma}[Local Equivalence of Parameter Norm]
    \label{lem:regulation}
Let $\Theta'_\epsilon = \left\{ \theta \in \Theta \mid \|\theta\|_2^2 \leq \epsilon \right\}$ denote a closed ball of radius $\sqrt{\epsilon}$ around the zero parameter vector in the parameter space $\Theta$. Assume:
\begin{enumerate}
    \item For all $x \in \mathcal{X}$, the zero-parameter model outputs the zero vector: $f_{\mathbf{0}}(x) = \mathbf{0}$;
    \item The zero output minimizes the conjugate function: $\Phi^*(\mathbf{0}) = \min_{\theta \in \Theta} \Phi^*(f_{\theta}(x))$ for all $x \in \mathcal{X}$.
\end{enumerate}
Then there exist positive constants $a$, $b$, and $\epsilon > 0$ such that for all $\theta \in \Theta'_\epsilon$ and all $x \in \mathcal{X}$, the parameter norm controls the conjugate function as follows:
\begin{equation}
    a\|\theta\|_2^2 \leq \Phi^*(f_{\theta}(x)) - \Phi^*(\mathbf{0}) \leq b\|\theta\|_2^2.
\end{equation}
\end{lemma}
The proof of Lemma~\ref{lem:regulation} is provided in Appendix~\ref{appendix:proof_regulation}. 
The assumption $f_{\mathbf{0}}(x) = \mathbf{0}$ is well-motivated by standard deep learning practices: zero-initialized parameters (weights and biases in fully connected layers, kernels in convolutional layers, attention weights in transformers) result in zero outputs for most modern neural network architectures. This makes the assumption both theoretically consistent and practically relevant to real-world model design.

Using this lemma, we formalize the connection between parameter regularization and the maximum loss $\gamma_\Phi(\theta)$ for classification tasks:
\begin{proposition}[Regularization Reduces Maximum Loss]
    \label{prop:generalize_explain}
For classification tasks with softmax cross entropy loss, assume:
\begin{enumerate}
    \item $f_{\mathbf{0}}(x) = \mathbf{0}$ for all $x \in \mathcal{X}$;
    \item The conjugate dual of the zero vector is the uniform distribution over the label space: $\mathbf{0}_{\Phi^*} = u$, where $u(y) = 1/|\mathcal{Y}|$ for all $y \in \mathcal{Y}$.
\end{enumerate}
Then reducing the squared $L_2$ parameter norm $\|\theta\|_2^2$ is equivalent to reducing the maximum loss $\gamma_\Phi(\theta)$.
\end{proposition}
The proof of Proposition~\ref{prop:generalize_explain} is provided in Appendix~\ref{appendix:proof_generalize_explain}. 
This result directly links $L_2$ regularization to tighter generalization bounds: by constraining the parameter norm, $L_2$ regularization reduces $\gamma_\Phi(\theta)$, which in turn tightens both deterministic and probabilistic upper bounds on generalization error. For finite-dimensional parameter spaces, all $L_p$-norms are equivalent (there exist positive constants $c, d > 0$ such that $c\|x\|_2 \leq \|x\|_p \leq d\|x\|_2$ for all $x \in \mathbb{R}^d$). This equivalence implies that $L_p$-norm regularization (for any $p \geq 1$) also constrains the generalization error upper bound by reducing $\gamma_\Phi(\theta)$, providing a unified theoretical interpretation for all norm-based regularization techniques.

\subsection{Beyond test sets: information-theoretic evaluation metrics}
\label{subsec:eval_perspectives}

In practical machine learning applications, evaluating generalization ability is a core challenge: the exact generalization error requires knowledge of the true data distribution $q$, which is typically unknown and unobservable. Traditional evaluation methods rely on partitioning a "full dataset" into training and test sets, defining generalization error as the difference in model performance (e.g., loss, accuracy) between the test set and training set. Formally, this approach approximates the population risk with the test-set risk, assuming the test set is representative of the true distribution.

However, this paradigm suffers from a fundamental limitation: in real-world open-world scenarios, the "full dataset" (and its distribution) is inherently unknown, we only have access to a finite training sample. For a given training set, multiple distinct true distributions could generate the same sample, meaning the "true generalization error" is not uniquely defined. The test-set based approach implicitly assumes the test set distribution matches the true distribution, an assumption that often fails in practice (e.g., distribution shift, out-of-distribution data). This mismatch is a key reason why models that perform well in controlled laboratory settings (via cross-validation) often degrade significantly when deployed in real-world applications.

Instead of focusing on heuristic test-set performance, the generalization bounds derived in this section suggest an alternative, theoretically grounded approach: monitor the \textit{controllable generalization factors} that appear explicitly in our bounds, rather than estimating generalization error indirectly via test sets. These factors are:
\begin{itemize}
    \item \textbf{Maximum loss $\gamma_\Phi(\theta)$}: Reflects the stability and calibration of the model's predictions. Smaller values indicate more consistent loss behavior and tighter generalization bounds.
    \item \textbf{Absolute information loss $\mathcal{L}(f_\theta(X))$}: Measures the compression of the input feature space by the model. Moderate compression (merging inputs with identical conditional target means) improves probabilistic generalization guarantees.
    \item \textbf{Relative information loss $\mathcal{L}_\Phi(Y'|f_\theta(X'))$}: Quantifies the preservation of predictive information about the target $Y'$ through the model. Larger values (up to a theoretical maximum) indicate beneficial compression that tightens deterministic generalization bounds.
\end{itemize}

These information-theoretic metrics offer three critical advantages over traditional test-set evaluation:
\begin{enumerate}
    \item \textbf{Computable during training}: They depend only on the model and training data, enabling real-time monitoring of generalization progress without requiring a held-out test set.
    \item \textbf{Theoretically grounded}: They directly appear in our generalization bounds, unlike heuristic metrics (e.g., test accuracy) which lack formal connection to generalization theory.
    \item \textbf{Distribution-free}: They do not assume access to the true data distribution or a representative test set, making them robust to distribution shift and open-world scenarios.
\end{enumerate}
Given their rigorous theoretical foundation and practical computability, these metrics serve as a valuable supplement to test-set based evaluation, particularly in high-stakes domains (e.g., medical diagnosis, safety-critical systems) where reliable generalization assessment is essential. By monitoring these factors during training, practitioners can proactively control generalization behavior, rather than relying on post-hoc test-set evaluation.

\section{Empirical validation}
\label{sec:experiments}

This section empirically validates the core theoretical claims of this work, with a focus on trainability, the proposed mechanism linking structure matrix eigenvalues and gradient energy requires rigorous verification across diverse architectures and tasks. Subsection~\ref{subsec:exp_design} details the experimental setup, including datasets, model architectures, and evaluation metrics. Subsection~\ref{subsec:exp_trainability} presents results validating the theoretical risk bounds and the role of gradient energy in optimization. Subsection~\ref{subsec:exp_structure} investigates how architectural design choices modulate the properties of the structure matrix.

\subsection{Experimental design}
\label{subsec:exp_design}

The primary objective of these experiments is to verify whether models with distinct architectures adhere to the theoretical predictions of this work during training, rather than pursuing state-of-the-art performance or direct comparisons against existing algorithms and architectures. Accordingly, we do not focus on final convergence accuracy; instead, we analyze whether the dynamic behavior of models during initialization and training aligns with our theoretical claims. Below, we describe the datasets, model architectures, loss functions, and hyperparameter settings used in all experiments.

\paragraph{Dataset}

We evaluate our theoretical framework on four widely adopted benchmark datasets: \texttt{MNIST}~\cite{LeCun1998GradientbasedLA}, \texttt{FashionMNIST}~\cite{DBLP:journals/corr/abs-1708-07747}, \texttt{CIFAR-10}, and \texttt{CIFAR-100}~\cite{krizhevsky2009learning}. To enable clear tracking of the evolution of key metrics throughout training, we construct two reduced-scale datasets: \emph{mini CIFAR-10} and \emph{mini CIFAR-100}, generated by randomly sampling 16 examples from the full CIFAR-10 and CIFAR-100 datasets, respectively. Importantly, for a fixed model architecture, training on these small-scale datasets corresponds to the same non-convex optimization problem as training on the full datasets. This design eliminates confounding factors such as model capacity, allowing us to isolate and clearly observe the core metrics of interest, thereby providing reliable validation for the proposed framework.

\paragraph{Models}
\begin{figure}[htbp]
\begin{center}
\centerline{\includegraphics[width=0.6\columnwidth]{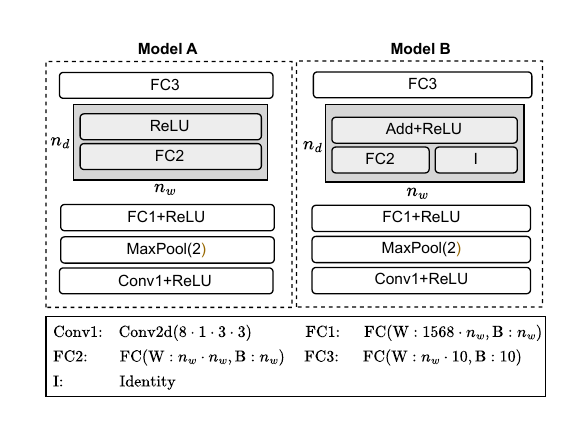}}
\caption{Custom-designed model architectures and configuration parameters. Gray blocks represent components where the number of repetitions can be adjusted via the parameter $n_d$, and model width can be tuned via the parameter $n_w$. Model B is a modified variant of Model A with additional skip connections. The symbol $I$ denotes an identity transformation, which preserves the dimensionality of feature maps in skip connection pathways.}
\label{fig:model_increase}
\end{center}
\end{figure}
\FloatBarrier
We employ two categories of models to verify our theoretical properties: custom-designed models (for controlled architectural ablation) and widely used classical architectures (for generalizability). The architecture and configuration parameters of our custom models are illustrated in Fig.~\ref{fig:model_increase}. Model B extends Model A by incorporating skip connections, where $n_d$ denotes the number of repeated blocks along the depth dimension (e.g., $n_d=2$ indicates two identical blocks cascaded sequentially) and $n_w$ denotes the number of repeated blocks along the width dimension. Increasing $n_d$ increases the number of cascaded blocks (and thus the total number of trainable parameters), while $n_w$ controls the model width (varied in steps of 16 in our experiments). These custom models are evaluated on MNIST and Fashion-MNIST to analyze how depth and width affect the structure matrix at initialization.

To validate the proposed trainability theory across standard architectures, we adopt three representative deep learning models: LeNet~\cite{5265772} (convolutional networks), ResNet18~\cite{He2015DeepRL} (networks with skip connections), and Vision Transformers (ViT)~\cite{dosovitskiy2021an} (transformer-based architectures). Detailed architectural specifications can be found in \citet{yoshioka2024visiontransformers}. These classical models are evaluated on CIFAR-10 and CIFAR-100. We use softmax cross entropy loss and MSE loss as test cases. To ensure consistent parameter and output behavior between training and inference, we disable techniques such as Batch Normalization and Dropout (which exhibit different behavior during training and testing). To enable fast convergence without Batch Normalization, we replace Batch Normalization with Layer Normalization in all models.

\paragraph{Environment and hyperparameters}

All experiments are conducted in the following environment: Python 3.7, PyTorch 2.2.2, and an NVIDIA GeForce RTX 2080 Ti GPU. The training configuration is consistent across all models: optimization is performed using SGD with momentum 0.9 and weight decay $5\times 10^{-4}$, and the learning rate is scheduled using Cosine Annealing. All models are initialized with the default PyTorch initialization scheme.

The learning rates and training epochs for each model-dataset pair are summarized in Table~\ref{tab:training_params}. For models trained on mini CIFAR-10 and mini CIFAR-100, we set the number of training epochs to 60 and the batch size to 2. For the full CIFAR-10 and CIFAR-100 datasets, we use 20 training epochs and a batch size of 32.
\begin{table}[h]
\centering
\caption{Training hyperparameters for different models, losses, and datasets.}
\label{tab:training_params}
\begin{tabular}{p{3cm} llp{6cm}}
\toprule
\textbf{Model} & \textbf{Loss} & \textbf{Dataset} & \textbf{lr}\\
\midrule
\multirow{8}{*}{\textbf{LeNet}}
& \multirow{4}{*}{\textbf{softmax cross entropy}}& mini CIFAR-10   & 0.005\\
& & mini CIFAR-100 & 0.01\\
& & CIFAR-10       & 0.001\\
& & CIFAR-100      & 0.002\\
& \multirow{4}{*}{\textbf{MSE}}
& mini CIFAR-10   & 0.1\\
& & mini CIFAR-100 & 1\\
& & CIFAR-10       & 0.1\\
& & CIFAR-100      & 1\\
\multirow{8}{*}{\textbf{ResNet}}
& \multirow{4}{*}{\textbf{softmax cross entropy}}
& mini CIFAR-10   & 0.01\\
& & mini CIFAR-100 & 0.01\\
& & CIFAR-10       & 0.1\\
& & CIFAR-100      & 0.1\\
& \multirow{4}{*}{\textbf{MSE}}
& mini CIFAR-10   & 0.002\\
& & mini CIFAR-100 & 0.04\\
& & CIFAR-10       & 0.01\\
& & CIFAR-100      & 0.1\\
\multirow{8}{*}{\textbf{ViT}}
& \multirow{4}{*}{\textbf{softmax cross entropy}}
& mini CIFAR-10   & 0.005\\
& & mini CIFAR-100 & 0.005\\
& & CIFAR-10       & 0.01\\
& & CIFAR-100      & 0.001\\
& \multirow{4}{*}{\textbf{MSE}}
& mini CIFAR-10   & 0.001\\
& & mini CIFAR-100 & 0.02\\
& & CIFAR-10       & 0.01\\
& & CIFAR-100      & 0.2\\
\bottomrule
\end{tabular}
\end{table}

\subsubsection{Evaluation metrics}

To avoid scaling issues across different loss functions and facilitate comparison and trend visualization, we use $\log_2 \mathcal{R}^{\circ}_\Phi(z)$ to measure fitting performance under the Fenchel--Young loss induced by $\Phi$. Specifically, for softmax cross entropy loss, we monitor $\log_2 \|y - \mathrm{Softmax}(f_\theta(x))\|_2^2$, and for MSE loss, we monitor $\log_2 \|y - f_\theta(x)\|_2^2$. According to Theorem~\ref{thm:base_bound}, the upper and lower bounds of $\log_2 \mathcal{R}^{\circ}_\Phi(z)$ are defined as:
\begin{equation}
    \begin{aligned}
        Ub(z) &= \log_2 \bigl\| \nabla_\theta \mathcal{R}_\Phi(\theta, z) \bigr\|_2^2 - \log_2 \lambda_{\min}(A_x),\\
        Lb(z) &= \log_2 \bigl\| \nabla_\theta \mathcal{R}_\Phi(\theta, z) \bigr\|_2^2 - \log_2 \lambda_{\max}(A_x).
    \end{aligned}
\end{equation}

A core claim of the conjugate learning framework for DNN trainability is that fitting of learning objectives during training is achieved by controlling the upper and lower bounds of the loss, which are composed of gradient energy and the extremal eigenvalues of the structure matrix. Model depth, width, and skip connections modulate the structure matrix, and thus these bounds. To verify this claim, we focus on three key experimental objectives:
\begin{enumerate}
    \item Whether the upper and lower bounds described in Theorem~\ref{thm:base_bound} and its corollaries hold strictly throughout training.
    \item Whether the loss ($\log_2 \mathcal{R}^{\circ}_\Phi(z)$) evolves consistently with its upper and lower bounds. This is verified by computing dynamic Pearson correlation coefficients between the loss and its bounds using a sliding window approach (window length = 4). The Pearson correlation coefficient ranges from -1 to 1, where -1 indicates perfect linear negative correlation, 1 indicates perfect linear positive correlation, and 0 indicates no linear correlation. We additionally compute the dynamic Pearson correlation coefficient between the loss and gradient energy to confirm that stability of the structure matrix's extremal eigenvalues is a prerequisite for controlling the loss via gradient energy.
    \item Whether varying model depth, width, and skip connection usage aligns with theoretical predictions: increasing parameter count and reducing parameter dependencies help control the structure matrix eigenvalues, and skip connections effectively suppress the decrease in extremal eigenvalues caused by increasing model depth.
\end{enumerate}

\subsection{Validation of trainability mechanisms}
\label{subsec:exp_trainability}

We first validate that the theoretical empirical risk bounds hold in practice and that gradient energy controls the optimization trajectory of deep neural networks.

\subsubsection{Experimental results on classical models}

\paragraph{LeNet results}
\begin{figure}[htbp]
	\centering
	\begin{minipage}{1.0\linewidth}
		\centering
            \centerline{\includegraphics[width=\columnwidth]{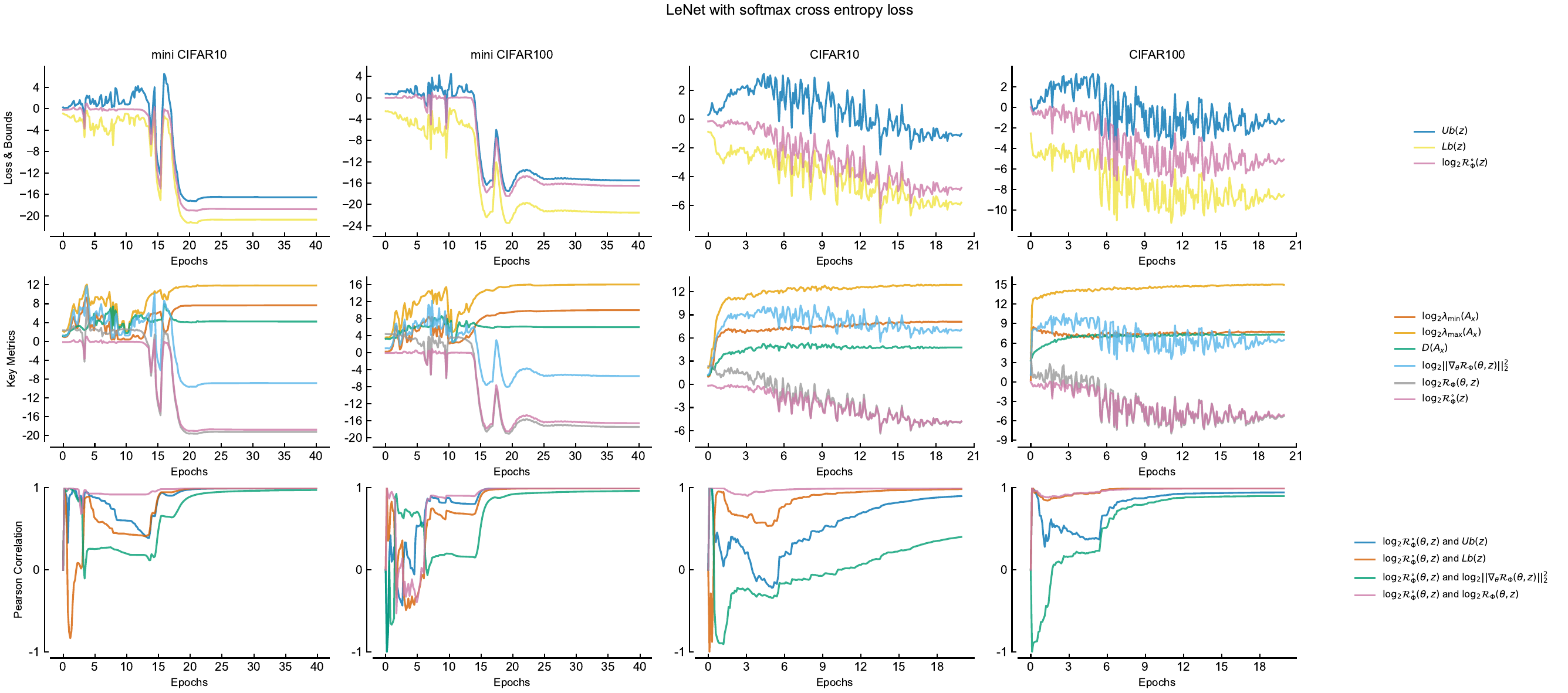}}
            \caption{Training dynamics of LeNet with softmax cross entropy loss. }
            \label{fig:lenet_sfce}
            \vspace{5mm}
	\end{minipage}
	\begin{minipage}{1.0\linewidth}
		\centering
        \centerline{\includegraphics[width=1.0\linewidth]{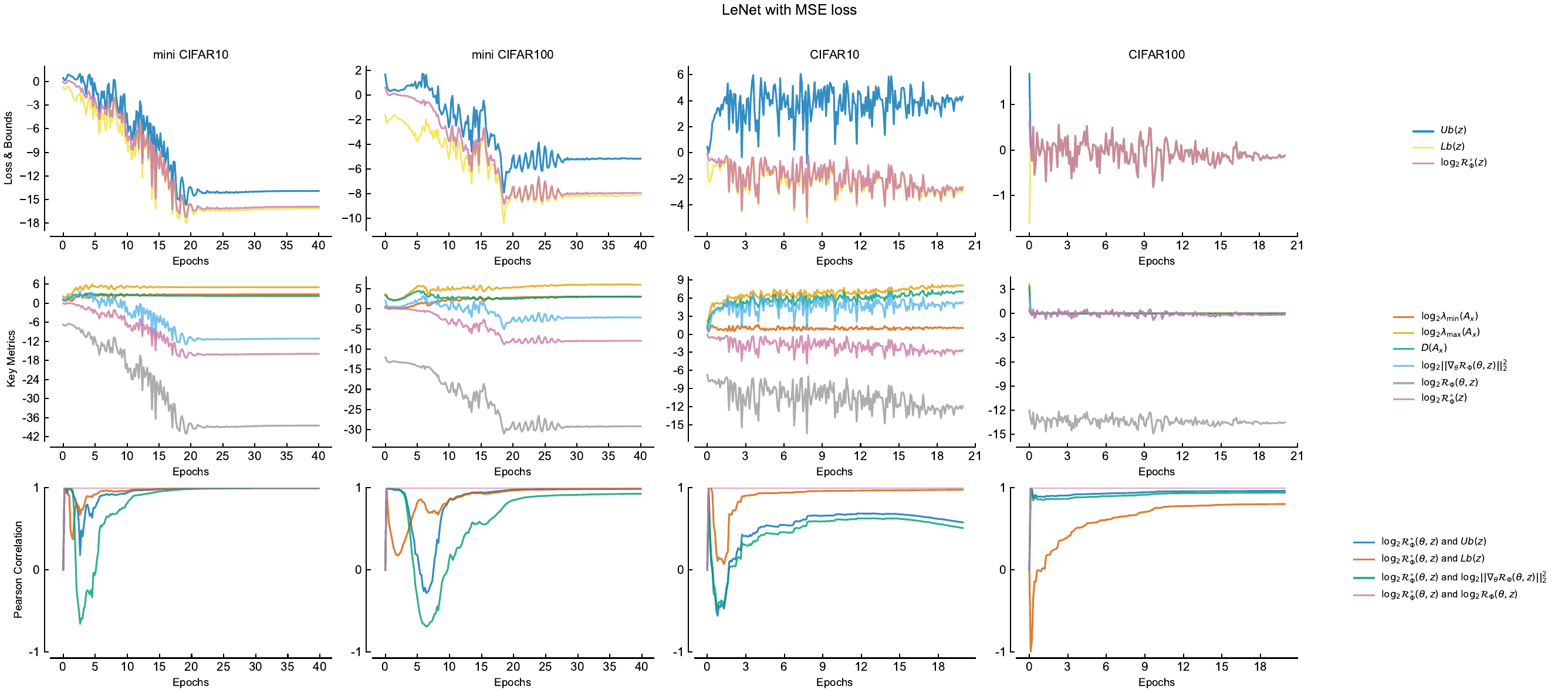}}
        \caption{Training dynamics of LeNet with MSE loss.}
        \label{fig:lenet_mse}
	\end{minipage}
    \caption{Training dynamics of LeNet across datasets and loss functions. The first row shows the evolution of standardized empirical risk and its upper/lower bounds during training. The second row depicts the changes in key metric terms over training iterations. The third row presents dynamic Pearson correlation coefficients between different terms.}
    \label{fig:lenet}
\end{figure}
Figure~\ref{fig:lenet} presents the training dynamics of LeNet across datasets and loss functions. The top row confirms that $Ub(z)$ and $Lb(z)$ consistently bound $\log_2\mathcal{R}^{\circ}_\Phi(z)$ throughout training, validating Theorem~\ref{thm:base_bound}. The bottom row shows that dynamic Pearson correlation coefficients between the loss and its bounds approach 1 after the initial training epochs, indicating that loss evolution is tightly coupled with the theoretical bounds.

\paragraph{ResNet18 results}
\begin{figure}[htbp]
	\centering
	\begin{minipage}{1.0\linewidth}
		\centering
            \centerline{\includegraphics[width=\columnwidth]{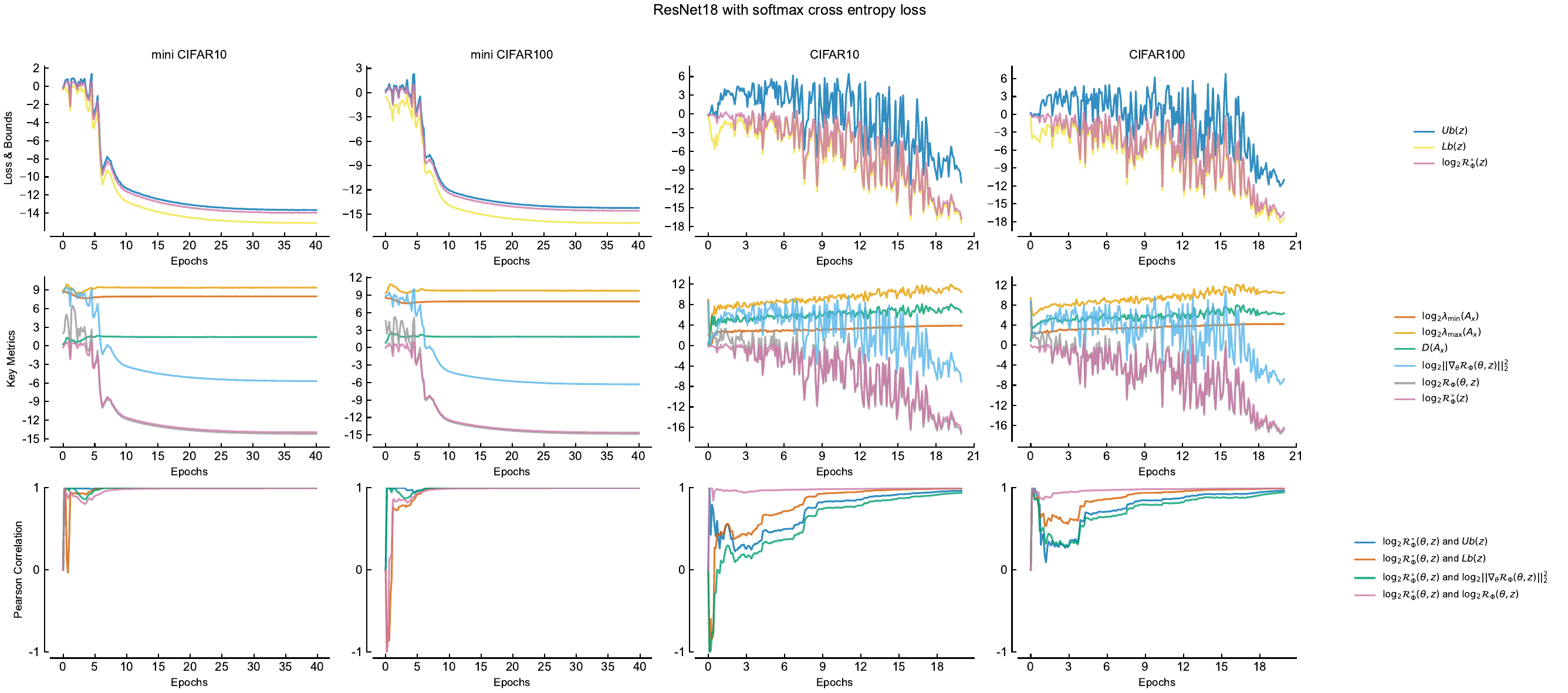}}
            \caption{Training dynamics of ResNet18 with softmax cross entropy loss.}
            \label{fig:resnet18_sfce}
            \vspace{5mm}
	\end{minipage}
	\begin{minipage}{1.0\linewidth}
		\centering
        \centerline{\includegraphics[width=1.0\linewidth]{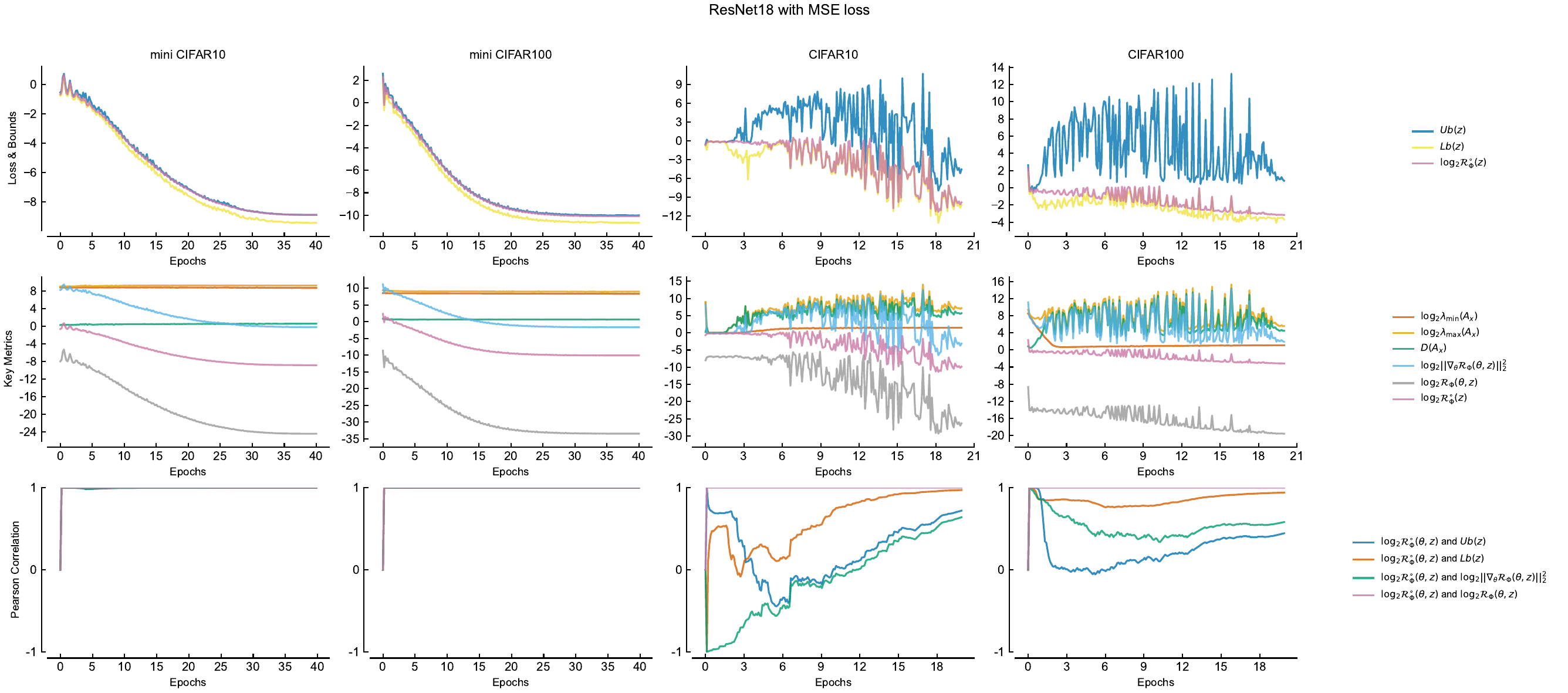}}
        \caption{Training dynamics of ResNet18 with MSE loss.}
        \label{fig:resnet18_mse}
	\end{minipage}
    \caption{Training dynamics of ResNet18 across datasets and loss functions. The first row demonstrates the evolution of standardized empirical risk and its upper/lower bounds during training. The second row shows the progression of key metrics throughout training. The third row plots dynamic Pearson correlation coefficients between different metrics over training iterations.}
    \label{fig:resnet18}
\end{figure}

ResNet18 (Fig.~\ref{fig:resnet18}) exhibits similar behavior to LeNet. Notably, on mini-datasets (where models have sufficient capacity to fit the data), the correlation between gradient energy and standardized risk approaches 1 after the structure matrix eigenvalues stabilize. This confirms that gradient energy controls risk only when the structure matrix is well-conditioned, as predicted by our theory.

\paragraph{ViT results}
\begin{figure}[htbp]
	\centering
	\begin{minipage}{1.0\linewidth}
		\centering
            \centerline{\includegraphics[width=\columnwidth]{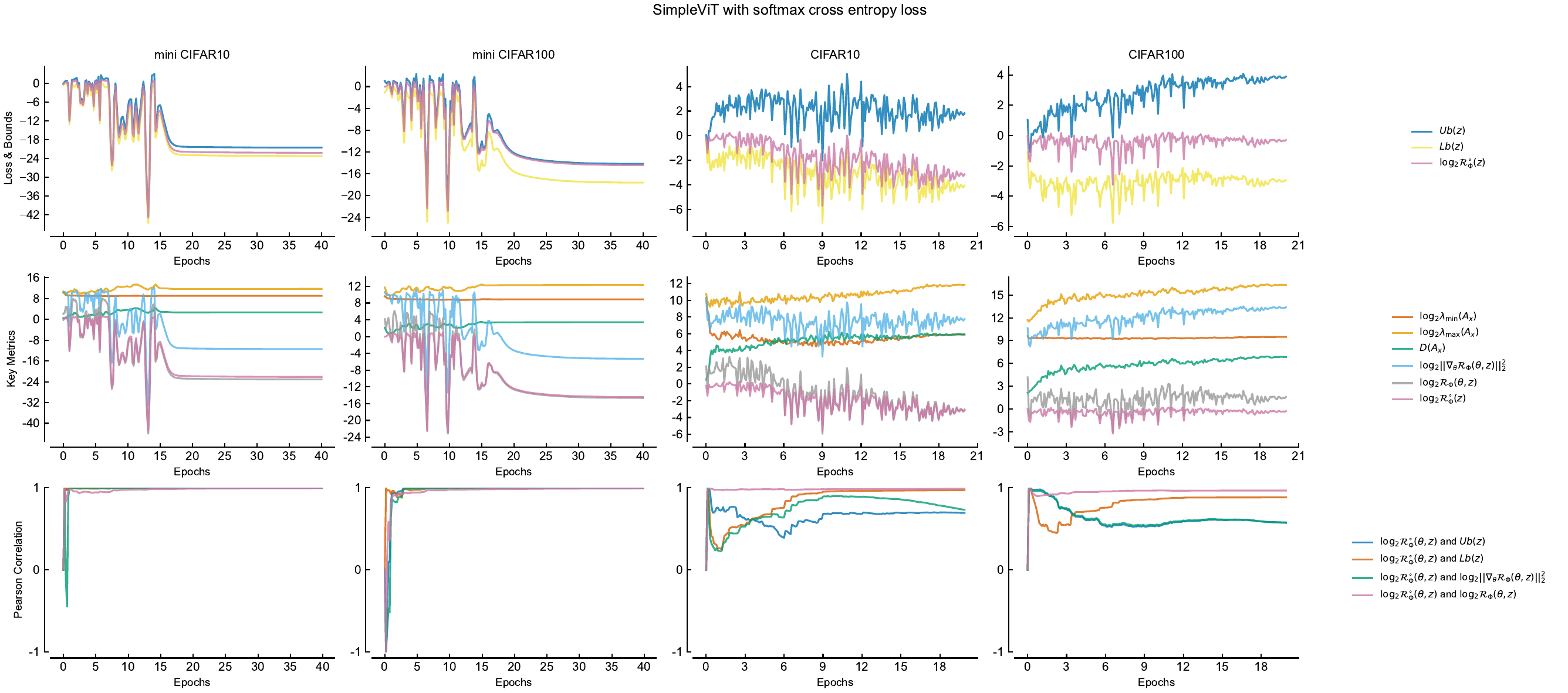}}
            \caption{Training dynamics of ViT with softmax cross entropy loss.}
            \label{fig:vit_sfce}
            \vspace{5mm}
	\end{minipage}
	\begin{minipage}{1.0\linewidth}
		\centering
        \centerline{\includegraphics[width=1.0\linewidth]{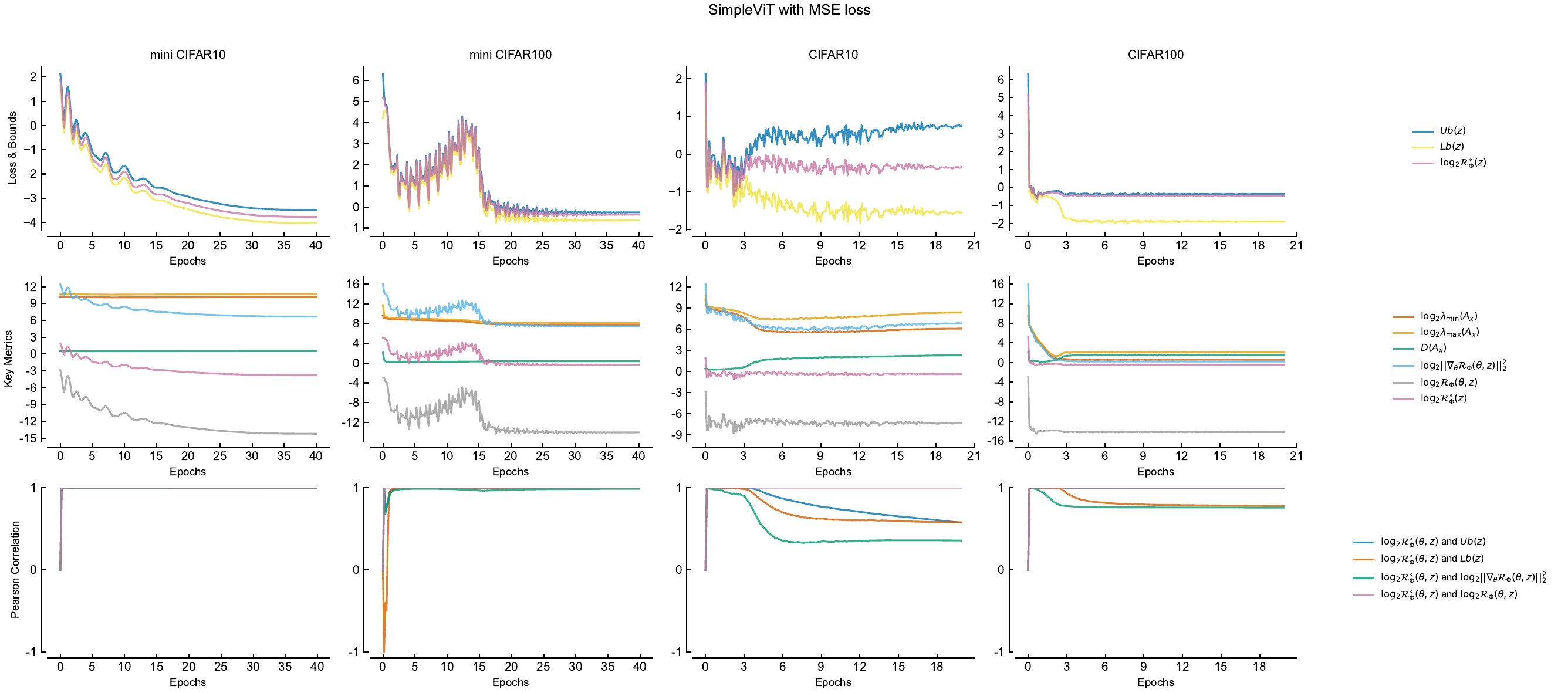}}
        \caption{Training dynamics of ViT with MSE loss.}
        \label{fig:vit_mse}
	\end{minipage}
    \caption{Training dynamics of ViT across datasets and loss functions. The first row exhibits the evolution of standardized empirical risk and its upper/lower bounds during training. The second row illustrates the variation of key metric terms over training iterations. The third row shows dynamic Pearson correlation coefficients between different terms throughout training.}
    \label{fig:vit}
\end{figure}
\FloatBarrier
Vision Transformers (Fig.~\ref{fig:vit}) exhibit the same pattern as convolutional networks, demonstrating the universality of the proposed trainability mechanism across diverse architectural paradigms.

\subsubsection{Summary of trainability results}

As shown in the first row of each subplot in Figs.~\ref{fig:lenet}, \ref{fig:resnet18}, and \ref{fig:vit}, $Ub(z)$ and $Lb(z)$ form valid upper and lower bounds for the standardized empirical risk $\log_2 \mathcal{R}^{\circ}_\Phi(z)$ across all datasets, loss functions, and model architectures tested. This result verifies the inequality relationships proposed in Theorem~\ref{thm:base_bound}.

We use dynamic Pearson correlation coefficients between the bounds and the standardized empirical risk to quantify the degree of control exerted by the bounds over the risk. A correlation coefficient close to 1 indicates that the standardized empirical risk evolves in complete alignment with its upper and lower bounds, implying that reducing the empirical risk can be achieved by tightening these bounds. This phenomenon is confirmed in the third row of each subplot: the dynamic Pearson correlation coefficients between the bounds and the standardized empirical risk quickly approach 1 after the start of training, and the correlation between the standardized empirical risk and the raw empirical risk simultaneously converges to 1.
Moreover, the dynamic Pearson correlation coefficient between the empirical risk and the standardized empirical risk remains close to 1 throughout training, demonstrating their full equivalence. \textit{Experimental conclusions for the standardized empirical risk therefore extend directly to the empirical risk.}

Based on Theorem~\ref{thm:sgd_convergence}, mini-batch SGD guarantees a reduction in gradient energy during training. Our experimental results thus demonstrate that fitting of learning objectives in DNN training is achieved by optimizing gradient energy, which in turn tightens the upper and lower bounds of the standardized empirical risk. This effect is particularly pronounced when the model has sufficient fitting capacity: on small-scale datasets (mini CIFAR-10 and mini CIFAR-100), after the extremal eigenvalues of the structure matrix stabilize, the dynamic Pearson correlation coefficients between gradient energy and the standardized empirical risk consistently approach 1 across all models, loss functions, and datasets. Notably, the correlation between gradient energy and the standardized empirical risk is smaller than the correlation between the bounds and the standardized empirical risk. This observation confirms that the optimization trajectory of the standardized empirical risk is governed by a two-dimensional "control surface" spanned by gradient energy and the extremal eigenvalues of the structure matrix. Only when the structure matrix eigenvalues stabilize can the standardized empirical risk be fully controlled by gradient energy.

In summary, our experimental results demonstrate that the training process of DNNs is fundamentally driven by joint control of the structure matrix's extremal eigenvalues and gradient energy, which together tighten the theoretical upper and lower bounds of the empirical risk. This mechanism exhibits high consistency across the CIFAR-10/CIFAR-100 datasets and the LeNet/ResNet18/ViT architectures, validating the broad applicability and robustness of the proposed conjugate learning framework and providing empirical evidence for its core claims.

\subsection{Experimental results on structure matrix properties}
\label{subsec:exp_structure}

We conduct two sets of controlled experiments on Model A and Model B (custom architectures) using the MNIST and Fashion-MNIST datasets to investigate the relationship between model architecture and structure matrix properties:
\begin{enumerate}
    \item Fix the width parameter $n_w=64$ and incrementally increase the depth parameter $n_d$ from 1 to 100.
    \item Fix the depth parameter $n_d=1$ and incrementally increase the width parameter $n_w$ from 16 to 1600.
\end{enumerate}
For each experimental setting, we observe the evolution of the structure matrix eigenvalues.

\begin{figure}[htbp]
	\centering
	\begin{minipage}{1.0\linewidth}
		\centering
            \centerline{\includegraphics[width=0.9\columnwidth]{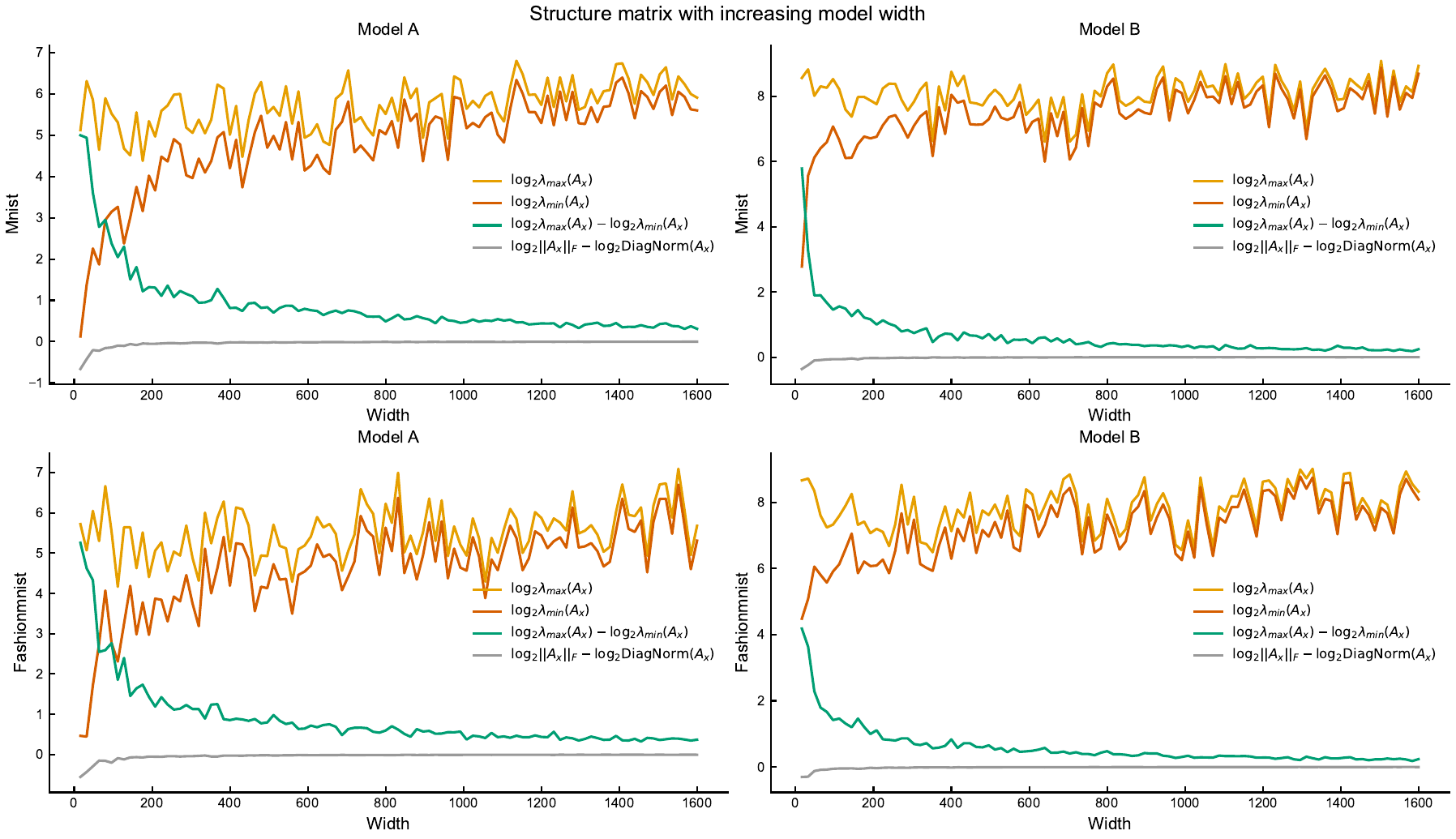}}
            \caption{Variation of the structure matrix with increasing model width.}
            \label{fig:width_init}
            \vspace{5mm}
	\end{minipage}
	\begin{minipage}{1.0\linewidth}
		\centering
        \centerline{\includegraphics[width=0.9\linewidth]{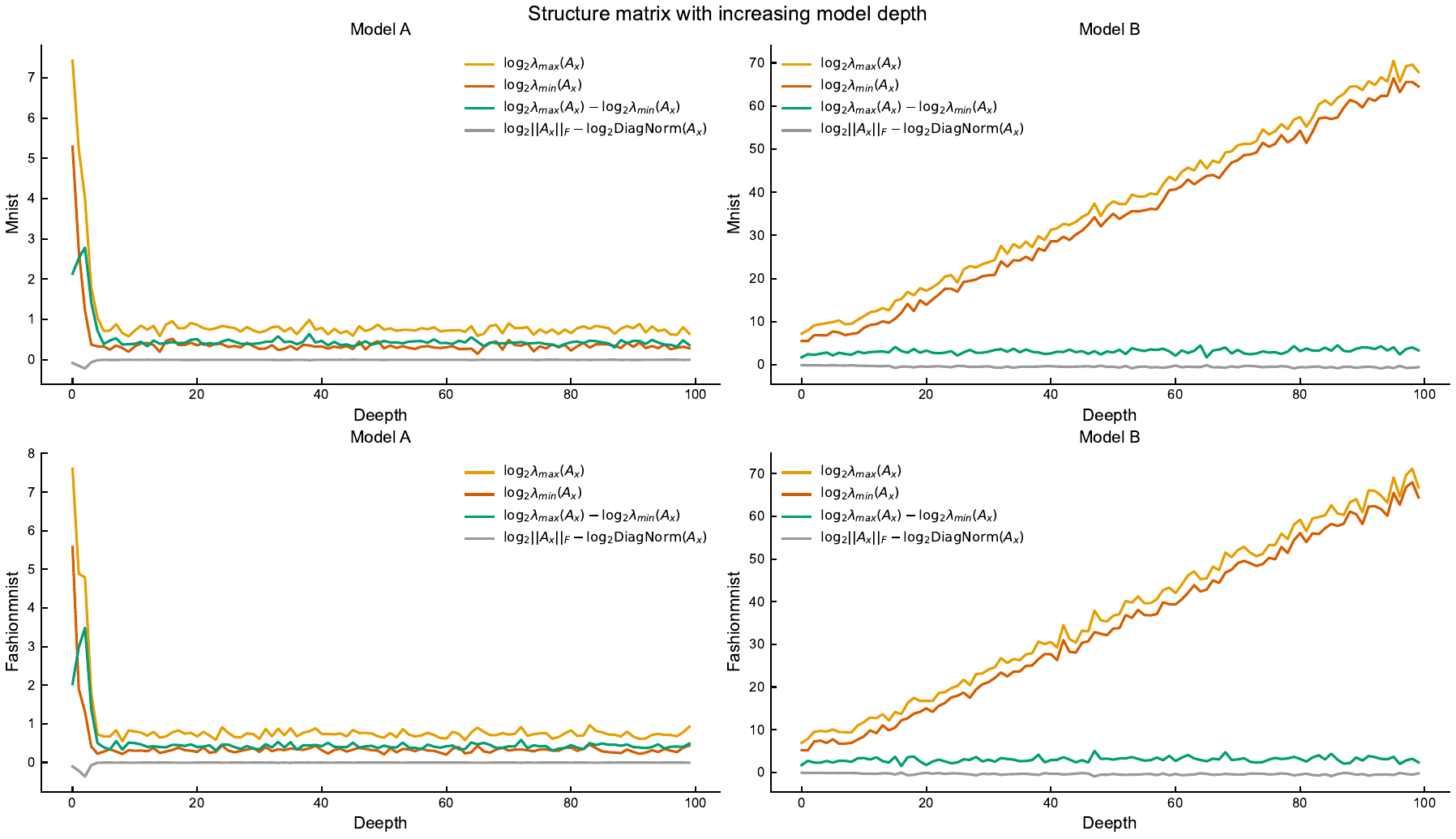}}
        \caption{Variation of the structure matrix with increasing model depth.}
        \label{fig:depth_init}
	\end{minipage}
    \caption{Relationship between the structure matrix and model depth/width at random initialization. $\|A_x\|_F$ denotes the Frobenius norm of the structure matrix $A_x$, and $\mathrm{DiagNorm}(A_x)$ represents the norm of the diagonal vector of $A_x$. The quantity $\log_2\|A_x\|_F - \log_2\mathrm{DiagNorm}(A_x)$ quantifies the diagonal dominance of $A_x$: a value of 0 indicates that $A_x$ is a diagonal matrix.}
    \label{fig:init}
\end{figure}
\FloatBarrier
\paragraph{Parameter Count and Spectral Gap}

As shown in Fig.~\ref{fig:init}, the spectral gap $\log_2 \lambda_{\max}(A_x) - \log_2 \lambda_{\min}(A_x)$ decreases as both model width and depth increase. Increasing model depth or width expands the total number of trainable parameters, validating Theorem~\ref{thm:number_para}: under the GAI assumption, increasing parameter count reduces the ratio between the extremal eigenvalues of the structure matrix. The GAI assumption is a prerequisite for Theorem~\ref{thm:number_para} to hold; at random initialization, model parameters are approximately independent, and their corresponding gradients are thus also approximately independent, satisfying the GAI assumption for our experimental settings.

\paragraph{Width and Eigenvalue Stabilization}

As shown in Fig.~\ref{fig:init}, increasing model width leads to an oscillatory increasing trend in the extremal eigenvalues $\log_2 \lambda_{\max}(A_x)$ and $\log_2 \lambda_{\min}(A_x)$ of both Model A and Model B, with the eigenvalues gradually stabilizing. If parameters change minimally during training (remaining close to their initialized values), the extremal eigenvalues of the structure matrix for wider models remain constant. In this scenario, the upper and lower bounds of the empirical risk are solely controlled by gradient energy, which can then be effectively optimized via mini-batch SGD. This observation aligns with the Lazy Regime perspective in NTK theory; however, our theoretical framework does not require the infinite width assumption (a key limitation of NTK-based approaches) and further provides explicitly computable upper and lower bounds for generalization error.

\paragraph{Depth and Skip Connections}

As shown in Fig.~\ref{fig:init}, increasing model depth leads to drastically different behavior in the structure matrix eigenvalues of Model A and Model B. The extremal eigenvalues of Model A (without skip connections) decrease rapidly with depth, while the logarithm of the extremal eigenvalues of Model B (with skip connections) increases linearly with depth. This result aligns with the predictions of Equations~\eqref{eq:noskip} and \eqref{eq:skip}: the structure matrix can be approximated as the product of Jacobian matrices of local blocks. For models without skip connections, random initialization results in local block Jacobian matrix elements that are much smaller than 1; the multiplicative effect causes the structure matrix eigenvalues to decrease rapidly with depth. In contrast, skip connections increase the magnitude of elements in the local block Jacobian matrices (formalized in Equation~\eqref{eq:skip}), leading to exponential growth of the extremal eigenvalues of Model B with depth. Larger extremal eigenvalues of the structure matrix lead to tighter upper and lower bounds of the empirical risk. Thus, Fig.~\ref{fig:depth_init} validates Proposition~\ref{prop:skip_con}: skip connections cause the extremal eigenvalues of the structure matrix to increase with depth, which in turn helps reduce $Ub(z)$ and $Lb(z)$.

\section{Related work}
\label{sec:related_work}

Before reviewing the existing literature on trainability and generalization of deep neural networks (DNNs), we first recall the standard formalization of a supervised machine learning task as established in \citep{Berner2021TheMM}. Let $\mathcal{Z} = \mathcal{X} \times \mathcal{Y}$ denote the product space of input features $\mathcal{X}$ and output targets $\mathcal{Y}$, and let $q$ represent the unknown true data-generating distribution over $\mathcal{Z}$. Given a training dataset $s^n = \{z^{(i)}\}_{i=1}^n$ sampled independently and identically (i.i.d.) from $q$, and a predefined loss function $\ell: \mathcal{F} \times \mathcal{Z} \to \mathbb{R}$ where $\mathcal{F}$ is a hypothesis class of functions mapping $\mathcal{X}$ to $\mathcal{Y}$, the core objective of machine learning is to learn a hypothesis $f \in \mathcal{F}$ that minimizes the expected risk, also known as the population risk, defined as $\mathcal{R}(f) := \mathbb{E}_{Z \sim q}[\ell(f, Z)]$. During practical training, the empirical risk $\widehat{\mathcal{R}}(f) := \frac{1}{n}\sum_{i=1}^n \ell(f, z^{(i)})$ serves as a computationally tractable proxy for the expected risk, since the true distribution $q$ is unknown. Within this foundational framework, the excess risk, defined as the difference between the expected risk of the learned hypothesis and the minimal achievable expected risk, can be decomposed into three distinct components: optimization error, generalization error, and approximation error~\citep{Berner2021TheMM}. Optimization error refers to the gap between the empirical risk achieved by the optimization algorithm and the minimal empirical risk over the hypothesis class $\mathcal{F}$, which directly reflects the trainability of the model. Generalization error describes the discrepancy between the expected risk and the empirical risk, which arises from learning from finite training samples rather than the full true distribution. Approximation error represents the difference between the minimal expected risk over $\mathcal{F}$ and the Bayes risk, which is determined by the expressive capacity of the chosen hypothesis class. Consequently, understanding the trainability of DNNs amounts to characterizing the mechanisms that govern optimization error, while understanding generalization corresponds to analyzing the behaviors and bounds of generalization error.

\subsection{Trainability}

Classical non-convex optimization theory provides only weak theoretical guarantees for general non-convex problems, typically ensuring convergence to stationary points (points where the gradient of the loss function is zero) rather than global minima~\citep{nesterov2008advance}. This limitation means classical theory fails to explain the consistent empirical success of stochastic gradient descent (SGD) and its variants in navigating the highly complex, non-convex loss landscapes of over-parameterized DNNs to achieve low empirical risk. A prevailing consensus in the literature holds that over-parameterization plays a pivotal role in enabling effective optimization in deep learning~\cite{Du2018GradientDP,Arjevani2022AnnihilationOS}.

\subsubsection{NTK theory}
\label{subsubsec:related_ntk}

Building on the foundational observation that infinitely wide DNNs with random initialization converge to Gaussian processes in the limit~\cite{Neal1995BayesianLF}, the NTK framework proposed by \citep{Jacot2018NeuralTK} demonstrates that in the infinite-width regime, DNNs behave as linear models parameterized by a fixed kernel matrix (the NTK) that remains constant during training. This linearization enables rigorous global convergence guarantees for gradient-based optimization algorithms, even for non-convex DNN architectures. The NTK framework has since been extensively extended and analyzed, with recent work refining its applicability to various network architectures and training regimes~\citep{liu2024ntk}. However, mounting empirical evidence indicates that finite-width DNNs (the setting of practical interest in real-world applications) often operate outside the so-called lazy training regime (where parameter updates are negligible relative to initialization). Finite-width networks exhibit dynamic feature learning, adaptive kernel evolution, and other key behaviors that are not captured by the static NTK theory~\citep{SeleznovaK21,seleznova2022neural,Vyas2023EmpiricalLO}, limiting the practical relevance of NTK-based guarantees for real-world DNN training.

\subsubsection{Fenchel-Young loss framework}
\label{subsubsec:related_fenchel_young}
 
A powerful unifying perspective on loss function design has emerged through the Fenchel--Young mathematical framework~\citep{Blondel2019LearningWF}. This framework shows that a broad range of standard loss functions used in machine learning, including softmax cross entropy (for classification), mean squared error (MSE) loss (for regression), hinge loss (for support vector machines), and perceptron loss, can be uniformly expressed in the form $d_\Omega(y, f_\theta(x))$, where $\Omega$ is a suitable strictly convex generating function and $d_\Omega$ denotes the Fenchel--Young divergence induced by $\Omega$. This formulation generalizes classical Bregman divergences to mixed input–output spaces and has been successfully applied to diverse tasks including supervised classification, structured prediction~\citep{blondel2022learning}, and variational inference~\citep{sklaviadis2025fenchel}.

Building on this unifying loss function framework, \citet{QiGL25} proposed an interpretation of supervised classification as the estimation of the conditional label distribution given input features. They established a key theoretical connection: under a Fenchel--Young loss, the squared norm of the gradient of the loss with respect to model parameters is tightly linked to the MSE between the predicted conditional distribution and the empirical conditional distribution estimated from training data. This connection is mediated by a \textit{structure matrix} defined as $A_x = \nabla_\theta f_\theta(x) \nabla_\theta f_\theta(x)^\top$, where $\nabla_\theta f_\theta(x)$ denotes the gradient of the model's output with respect to the parameter vector $\theta$ at input $x$. Their empirical experiments further suggested that DNN training for classification implicitly minimizes the squared gradient norm while maintaining the non-degeneracy of the structure matrix $A_x$ (i.e., ensuring $A_x$ is full rank to avoid singular optimization dynamics). Despite these valuable insights into the optimization dynamics of classification tasks, this approach has three critical limitations that restrict its broader applicability:

\begin{enumerate}
    \item It does not provide a theoretical guarantee of monotonic decrease for the Fenchel--Young loss itself during optimization, creating a potential misalignment between the surrogate loss (squared gradient norm) minimized in practice and the true objective (Fenchel--Young loss) of interest.
    \item The theoretical analysis is strictly restricted to supervised classification tasks and does not extend to other fundamental learning settings such as regression, generative modeling, or unsupervised learning.
    \item The framework focuses exclusively on trainability (optimization error) and does not address generalization error from the same distribution-learning perspective, leaving a critical gap in understanding end-to-end learning behavior.
\end{enumerate}

In summary, while NTK theory provides valuable theoretical insights into the infinite-width regime of DNNs and Fenchel-Young losses offer a elegant unifying perspective on loss function design for classification, neither approach provides a complete theoretical picture of DNN learning. NTK theory fails to capture the dynamic training dynamics of finite-width networks, and the Fenchel-Young framework is constrained to classification tasks and does not address generalization. These fundamental gaps in existing theory motivate our development of a unified conjugate learning framework that simultaneously addresses trainability and generalization across diverse learning tasks.

\subsection{Generalization}
\label{subsec:related_generalization}

Generalization refers to the ability of a trained DNN model to make accurate predictions on unseen test data drawn from the same underlying distribution as the training data. Researchers have investigated this critical property from multiple complementary theoretical perspectives, proposing a variety of frameworks to understand and quantify it. In this subsection, we review the most influential approaches to generalization analysis: classical complexity-based bounds, PAC-Bayesian bounds, the flat minima hypothesis, and information-theoretic perspectives. For each framework, we highlight its key insights and inherent limitations, which collectively motivate the development of our proposed conjugate learning framework.

\subsubsection{Classical complexity-based bounds}
\label{subsubsec:related_complexity_bounds}

Classical statistical learning theory quantifies generalization performance by deriving upper bounds on the generalization error based on measures of the hypothesis class complexity. Two of the most widely used complexity measures are the VC-dimension (which captures the capacity of the hypothesis class to shatter finite datasets)~\cite{Vapnik2006EstimationOD,Wu2021StatisticalLT} and Rademacher complexity (which measures the average sensitivity of the hypothesis class to random label perturbations)~\cite{Bartlett2003RademacherAG}. These complexity-based bounds universally predict that reducing the size or complexity of the hypothesis class will improve generalization performance by limiting the model's ability to overfit to noise in the training data.

However, in the over-parameterized regime, where DNNs typically contain orders of magnitude more parameters than training samples, these classical bounds become vacuous. A vacuous bound yields an upper limit on generalization error that is larger than the maximum possible value of the loss function, making it statistically uninformative and unable to reflect the strong empirical generalization performance of over-parameterized DNNs. A striking demonstration of this paradox is that DNNs can perfectly fit random (meaningless) labels on training data while still achieving strong generalization on real test data~\citep{ZhangBHRV17,ZhangBHRV21}. This fundamental disconnect between classical theory and empirical practice has motivated the development of alternative theoretical paradigms that move beyond simple hypothesis class complexity measures to explain generalization in over-parameterized DNNs.

\subsubsection{PAC-Bayesian Bounds}
\label{subsubsec:related_pac_bayes}

PAC-Bayesian theory provides a principled framework that bridges Bayesian inference and Probably Approximately Correct (PAC) learning~\citep{valiant1984theory}. Unlike classical complexity-based bounds that depend on global properties of the hypothesis class, PAC-Bayesian bounds characterize generalization through the divergence between a data-dependent posterior distribution and a data-independent prior distribution over hypotheses.
 
Let $\mathcal{Q}$ be a posterior distribution over the hypothesis space $\mathcal{F}$ (which may depend on the training data $s^n$), and let $\mathcal{P}$ be a prior distribution chosen without access to the data. PAC-Bayesian theory bounds the expected risk of a randomized classifier that samples hypotheses according to $\mathcal{Q}$. The canonical PAC-Bayesian bound states that with probability at least $1-\delta$ over the random draw of the training data, for any distribution $\mathcal{Q}$:

\begin{equation}
\label{eq:pacbayes}
\mathbb{E}_{f\sim\mathcal{Q}}[\mathcal{R}(f)] \leq \mathbb{E}_{f\sim\mathcal{Q}}[\widehat{\mathcal{R}}_s(f)] + \sqrt{\frac{\mathrm{KL}(\mathcal{Q}\|\mathcal{P}) + \ln\frac{n}{\delta}}{2(n-1)}},
\end{equation}

where $\mathrm{KL}(\mathcal{Q}\|\mathcal{P})$ denotes the Kullback-Leibler divergence between the posterior and prior, and $n$ is the sample size~\citep{McAllester1998SomePT}. This bound reveals a trade-off: a posterior that fits the training data well (low empirical risk) may incur a large KL penalty if it deviates too far from the prior, while staying close to the prior may limit the ability to achieve low empirical risk.

PAC-Bayesian bounds offer several advantages: they are among the tightest known generalization bounds for randomized classifiers~\citep{Oneto2020ModelSA}, they explicitly account for the learning algorithm through the posterior distribution, and they allow incorporation of structural knowledge through the choice of prior. However, these guarantees are often loose or data-agnostic in practice, and the choice of prior $\mathcal{P}$ significantly influences the resulting bound. Moreover, like classical complexity bounds, PAC-Bayesian bounds do not directly capture the nuanced interplay among architecture, optimization, and data observed in modern deep learning practice.

\subsubsection{Flat minima hypothesis}
\label{subsubsec:related_flat_minima}

The \textit{flat minima} hypothesis posits that the inherent stochasticity of stochastic optimization algorithms such as SGD acts as an implicit regularizer during DNN training. This stochasticity steers the optimization process toward flat minima of the loss landscape, regions where the loss function changes only minimally in response to small perturbations of the model parameters, rather than sharp minima where small parameter changes lead to large loss increases~\citep{KeskarMNST17,Chaudhari2017StochasticGD}. Empirical studies have consistently associated flat minima with better generalization performance, and a common operational measure of flatness is the largest eigenvalue $\lambda_{\max}$ of the Hessian matrix of the loss function evaluated on the training dataset (smaller values of $\lambda_{\max}$ indicate flatter minima).

However, a critical limitation of the flat minima hypothesis was identified by \citet{Dinh2017SharpMC}, who demonstrated that commonly used flatness measures are not invariant under model reparameterization. Specifically, they showed that identical input–output mappings (and thus identical generalization performance) can be represented by different parameterizations of the same DNN architecture, yielding arbitrarily different flatness values (e.g., $\lambda_{\max}$) at their respective minima. This finding casts significant doubt on whether flatness alone is a reliable or intrinsic predictor of generalization performance, challenging the causal link between flat minima and improved generalization and highlighting the need for more robust, reparameterization-invariant characterizations of the loss landscape.

\subsubsection{Information-Theoretic Approaches}
\label{subsubsec:related_info_theory}

Information-theoretic approaches offer a conceptually intuitive framework for understanding generalization in DNNs by framing learning as a process of information compression. The central pillar of this perspective is the IB principle~\citep{TishbyZ15}, which posits that optimal learning requires balancing two competing objectives: \textit{sufficiency} (preserving all information in the input features $X$ that is relevant to predicting the target labels $Y$) and \textit{minimality} (discarding as much irrelevant or noisy information from $X$ as possible to avoid overfitting).
A landmark study by \citet{ShwartzZiv2017OpeningTB} argued that DNN training unfolds in two distinct phases driven by the IB principle: an initial fitting phase, during which both mutual informations $I(X;T)$ (information between inputs $X$ and hidden representations $T$) and $I(Y;T)$ (information between targets $Y$ and hidden representations $T$) increase; followed by a compression phase, in which $I(X;T)$ decreases (irrelevant information is discarded) while $I(T;Y)$ remains approximately constant (relevant information is preserved), thereby improving generalization through information minimization.
 
Despite its elegance and intuitive appeal, the IB framework faces significant theoretical and empirical challenges that limit its practical applicability to modern DNNs:

\begin{itemize}
    \item \textbf{Estimation difficulty}: Estimating mutual information in the high-dimensional, continuous spaces of modern DNN hidden activations is a notoriously difficult problem in practice. Common mutual information estimators, including those based on binning, $k$-nearest neighbors, or kernel density estimation, are highly sensitive to hyperparameter choices and suffer from severe bias and variance, especially in the finite-sample regime typical of real-world training datasets~\citep{KraskovSG04,Paninski03}. This estimation challenge has led to inconsistent empirical findings about whether the proposed compression phase genuinely occurs during DNN training.
    
    \item \textbf{Non-universality of compression}: \citet{SaxeBDAKTC18} challenged the universality of the compression phase by demonstrating that it vanishes when using non-saturating nonlinear activation functions such as the rectified linear unit (ReLU), which are the dominant choice in modern DNN architectures. This suggests that the compression phase may be an artifact of specific activation function choices (e.g., sigmoid or tanh) rather than a fundamental mechanism of generalization in DNNs.
    
    \item \textbf{Causal ambiguity}: Subsequent empirical and theoretical work further questioned the causal link between information compression and generalization: DNN models can achieve strong generalization performance without exhibiting clear compression dynamics in their hidden representations, and conversely, compression can occur even when generalization performance is poor~\citep{GoldfeldBGMNKP19}. This ambiguity undermines the IB principle as a causal explanation for generalization.
\end{itemize}

Despite these technical challenges and ongoing debates in the literature, applying the IB principle to study generalization in DNNs has established a valuable connection between machine learning theory and information theory. It offers a natural, intuitive compression-based perspective on generalization that continues to inspire new theoretical and empirical research directions in deep learning.

\section{Conclusion}
\label{sec:conclusion}

This work introduces conjugate learning theory, a unified framework for understanding trainability and generalization in DNNs. By framing learning tasks as conditional distribution estimation and leveraging convex conjugate duality, we establish that trainability arises from the joint control of gradient energy and structure matrix eigenvalues, with convergence characterized by a gradient correlation factor. We further show that architectural choices modulate these key quantities: increasing width narrows the spectral gap of the structure matrix, increasing depth amplifies eigenvalue decay, and skip connections counteract this decay. For generalization, we derive deterministic bounds (valid for arbitrary sampling schemes) and probabilistic bounds (under i.i.d. sampling), revealing how maximum loss, information loss, and data smoothness collectively determine generalization performance.
Experimental results across diverse architectures (LeNet, ResNet18, ViT) and datasets validate these theoretical predictions: the empirical risk bounds hold throughout training, and architectural choices modulate structure matrix eigenvalues as theoretically derived. Limitations of this work include the gradient approximate independence assumption, which holds at initialization but may break down during training, and the finite feature space assumption, which theoretically justifies discretization but warrants further analysis of granularity effects. Future work will focus on designing adaptive algorithms based on the gradient correlation factor to further improve optimization efficiency and generalization guarantees.

\section*{Acknowledgments}
This work was supported in part by the National Natural Science Foundation of China under Grant 72171172 and 92367101.

\appendix
\section{Appendix}
\label{appendix}

\subsection{Proof of Theorem~\ref{thm:linear_conjugate}}
\label{appendix:proof_linear_conjugate}

\begin{theorem}
Let $G \in \mathbb{R}^{m \times d}$ be a matrix with rows $G_i^\top$, and define the constraint set $C = \{ y \in \mathbb{R}^d : G y = b \}$. Then,
\begin{equation}
\Phi^*(\nu) = \Omega^*(\nu + G^\top \lambda_*) - \langle \lambda_*, b \rangle,
\end{equation}
where \( \lambda_* \in \mathbb{R}^m \) satisfies \( G \cdot (\nu + G^\top \lambda_*)_{\Omega^*}^* = b \).
\end{theorem}

\begin{pf}
By definition of the convex conjugate, we have
\begin{equation}
    \label{eq:def_1}
    \Phi^*(\nu) = \sup_{y \in \mathrm{dom}(\Phi)} \bigl\{ \langle y, \nu \rangle - \Phi(y) \bigr\},
\end{equation}
where $\Phi(y) = \Omega(y) + I_C(y)$ and $C = \{ y \in \mathbb{R}^d : G y = b \}$.

Since $I_C(y)$ enforces the linear equality constraint $G y = b$, problem~\eqref{eq:def_1} is equivalent to the constrained optimization problem
\[
\sup_{y \in \mathbb{R}^d} \bigl\{ \langle y, \nu \rangle - \Omega(y) \bigr\} \quad \text{subject to} \quad G y = b.
\]
The corresponding Lagrangian is
\[
L(y, \lambda) = \langle y, \nu \rangle - \Omega(y) - \langle \lambda, G y - b \rangle,
\]
where $y \in \mathbb{R}^d$ and $\lambda \in \mathbb{R}^m$.

Because $\Omega$ is strictly convex and differentiable, strong duality holds and the Karush--Kuhn--Tucker (KKT) conditions are necessary and sufficient for optimality. A point $y_*$ solves the problem~eq\ref{eq:def_1} if and only if there exists $\lambda_* \in \mathbb{R}^m$ such that
\begin{equation*}
    \begin{aligned}
        &\nabla \Omega(y_*) - \nu - G^\top \lambda_* = 0, \\
        &G y_* - b = 0.
    \end{aligned}
\end{equation*}
From the stationarity condition, we obtain
\[
\nabla \Omega(y_*) = \nu + G^\top \lambda_*.
\]
Under the strict convexity and differentiability of $\Omega$, the gradient map $\nabla \Omega$ is invertible and $(\nabla \Omega)^{-1} = \nabla \Omega^*$. Hence,
\[
y_* = \nabla \Omega^*(\nu + G^\top \lambda_*) = (\nu + G^\top \lambda_*)_{\Omega^*}^*,
\]
where we adopt the notation $z_{\Omega^*}^* := \nabla \Omega^*(z)$.
Substituting this expression into the primal feasibility condition yields
\[
G \cdot (\nu + G^\top \lambda_*)_{\Omega^*}^* = b,
\]
which implicitly defines $\lambda_*$ as a function of $\nu$, $G$, and $b$.

Alternatively, by Lemma~\ref{lem:addition_affine}, the dual function associated with the Lagrangian is
\begin{equation*}
    \inf_{y \in \mathbb{R}^d} \bigl\{ \Omega(y) - \langle y, \nu \rangle - \langle \lambda, G y - b \rangle \bigr\}
    = -\Omega^*(\nu + G^\top \lambda) + \langle \lambda, b \rangle.
\end{equation*}

Since $\Omega$ is a strictly convex function, it follows from equation~\eqref{eq:def_1} that
\begin{equation*}
    \Phi^*(\nu) = \Omega^*(G^\top \lambda_* + \nu) - \langle \lambda_*, b \rangle,
\end{equation*}
where $\lambda_*$ satisfies
\begin{equation*}
    G \cdot (G^\top \lambda_* + \nu)^*_{\Omega^*} = b.
\end{equation*}
This completes the proof.
\end{pf}

\subsection{Proof of Theorem~\ref{thm:uniqueness_fenchel}}
\label{appendix:proof_uniqueness_fenchel} 

\begin{theorem}
Let $\ell(y, f_\theta(x)_{\Phi^*}^*)$ be a differentiable loss function. If $\ell$ satisfies both (i) strong convexity with respect to the model's prediction and (ii) properness, then there exists a strictly convex function $\Phi$ such that
\[
\ell(y, f_\theta(x)_{\Phi^*}^*) - \ell(y, y) = d_\Phi(y, f_\theta(x)).
\]
\end{theorem}

\begin{pf}
For notational convenience, let $\mu := y$ denote the target (e.g., a one-hot label vector), and let $\nu := f_\theta(x)_{\Phi^*}^*$ denote the model's prediction. 

Since $\ell(\mu, \nu)$ is differentiable and strictly convex with respect to $\nu$, it possesses a unique minimum at $\nu = \mu$. At this point, the gradient of $\ell(\mu, \nu)$ with respect to $\nu$ vanishes, i.e., $\nabla_{\nu} \ell(\mu, \mu) = 0$.
Thus, we may express the gradient of the loss function as:
\begin{equation}
\nabla_{\nu}\ell(\mu,\nu) = (\nu - \mu)^\top g(\mu,\nu),
\end{equation}
where $g(\mu,\nu)$ is a vector-valued function that captures the direction and rate of change of the loss with respect to $\nu$.
Due to the uniqueness of the minimizer, the expected loss $\mathbb{E}_{\xi \sim Q} [\ell(\xi, \nu)]$ is strictly convex in $\nu$, and at its minimizer $\bar{\xi} = \mathbb{E}_{\xi \sim Q}[\xi]$, the gradient with respect to $\nu$ vanishes:
\[
\nabla_\nu \mathbb{E}_{\xi\sim Q}[\ell(\xi,\bar{\xi})] = 0,
\]
which implies
\[
\sum_{\xi \in \mathcal{Q}} Q(\xi)(\bar{\xi} - \xi)^\top g(\xi,\bar{\xi}) = 0,
\]
where $\xi \sim Q$.
Since $Q$ is an arbitrary distribution, $g(\xi,\nu)$ must be independent of $\xi$; we may therefore replace $g(\xi,\nu)$ with $g(\nu)$. 

Because $\mathbb{E}_{\xi\sim Q} [\ell(\xi, \nu)]$ is strictly convex, its Hessian with respect to $\nu$ is positive definite:
\begin{equation}
\nabla^2_\nu \mathbb{E}_{\xi\sim Q} [\ell(\xi, \nu)] = g(\nu) + (\nu - \bar{\xi}) \nabla_\nu g(\nu) \succ 0.
\end{equation}
Setting $\nu = \bar{\xi}$ yields $\nabla^2_\nu \mathbb{E}_{\xi\sim Q} [\ell(\xi, \bar{\xi})] = g(\bar{\xi}) \succ 0$. Since $\bar{\xi}$ is arbitrary, it follows that $g(\nu) \succ 0$ for all $\nu$.

Define a strictly convex function $\Phi$ such that its Hessian is given by $\nabla^2 \Phi(\nu) = g(\nu)$. The gradient of the loss $\ell(\mu,\nu)$ with respect to $\nu$ can then be expressed as:
\begin{equation}
    \begin{aligned}
        \nabla_\nu \ell(\mu,\nu)&=\nabla^2_\nu\Phi(\nu)(\nu-\mu)\\
        &=\nabla_\nu\nu_\Phi^*\nu-\nabla_\nu \nu_\Phi^*\mu\\
        &=\nabla_\nu(\langle \nu, \nu_\Phi^*\rangle -\Phi(\nu)-\mu^\top \nu_\Phi^*)\\
        &=\nabla_\nu(\Phi^*(\nu_\Phi^*)-\langle \mu,\nu_\Phi^*\rangle +h(\mu)),
    \end{aligned}
\end{equation}
where $h(\mu)$ is a function independent of $\nu$.
Since $\nabla_\nu \ell(\mu,\mu) = 0$, we obtain
\[
h(\mu) = \Phi^*(\nabla \Phi(\mu)) - \langle \mu, \nabla \Phi(\mu) \rangle + c = \Phi(\mu) + c,
\]
where $c$ is a constant.
This leads to the following expression for the loss function:
\begin{equation}
    \begin{aligned}
        \ell(\mu,\nu)&=\Phi^*(\nu_\Phi^*)-\langle \mu,\nu_\Phi^*\rangle +\Phi(\mu)+c,\\
        &=d_\Phi(\mu,\nu_\Phi^*),\\
        \ell(\mu,\mu)&=c.
    \end{aligned}
\end{equation}
Therefore, it follows that:
\begin{equation}
\ell(\mu, \nu) - \ell(\mu, \mu) = d_\Phi(\mu, \nu_\Phi^*).
\end{equation}
In particular, setting $f_\theta(x)=\nu_\Phi^*$, we obtain
\begin{equation*}
\ell(y, f_\theta(x)_\Phi^*) - \ell(y, y) = d_\Phi(y, f_\theta(x)).
\end{equation*} 
\end{pf}

\subsection{Proof of Theorem~\ref{thm:base_bound}}
\label{appendix:proof_base_bound}

\begin{theorem}
 
Let $Z\sim \hat{q}$. For empirical risk minimization, if $\lambda_{\min}(A_s) \neq 0$, we have
\begin{enumerate}
    \item For the standardized empirical risk,
\begin{equation*}
\begin{aligned}
     \frac{1}{\lambda_{\max}(A_x)}\|\nabla_\theta \mathcal{R}_\Phi(\theta, z)\|_2^2
     &\le \mathcal{R}_\Phi^{\circ}(\theta, z)
     \le  \frac{1}{\lambda_{\min}(A_x)}\|\nabla_\theta \mathcal{R}_\Phi(\theta, z)\|_2^2, \\[4pt]
    \frac{1}{\lambda_{\max}(A_s)} \mathbb{E}_Z\|\nabla_\theta \mathcal{R}_\Phi(\theta, Z)\|_2^2
    &\le \mathcal{R}^{\circ}_\Phi(\theta, s)
    \le \frac{1}{\lambda_{\min}(A_s)} \mathbb{E}_Z\|\nabla_\theta \mathcal{R}_\Phi(\theta, Z)\|_2^2,
\end{aligned}
\end{equation*}

    \item For the empirical risk,
\begin{equation*}
\begin{aligned}
     \frac{\lambda_{\min}(H_\Phi(\theta))}{\lambda_{\max}(A_x)}\|\nabla_\theta \mathcal{R}_\Phi(\theta, z)\|_2^2
     &\le \mathcal{R}_\Phi(\theta, z)
     \le  \frac{\lambda_{\max}(H_\Phi(\theta))}{\lambda_{\min}(A_x)}\|\nabla_\theta \mathcal{R}_\Phi(\theta, z)\|_2^2, \\[4pt]
    \frac{\lambda_{\min}(H_\Phi(\theta))}{\lambda_{\max}(A_s)} \mathbb{E}_Z\|\nabla_\theta \mathcal{R}_\Phi(\theta, Z)\|_2^2
    &\le \mathcal{R}_\Phi(\theta, s)
    \le \frac{\lambda_{\max}(H_\Phi(\theta))}{\lambda_{\min}(A_s)} \mathbb{E}_Z\|\nabla_\theta \mathcal{R}_\Phi(\theta, Z)\|_2^2.
\end{aligned}
\end{equation*}
\end{enumerate}
\end{theorem}
\begin{pf}

Recall that the loss $\mathcal{R}_\Phi(\theta, z)$ is defined as the Fenchel--Young loss $d_\Phi(y, f_\theta(x))$ for a sample $z = (x, y)$.
By the chain rule, the squared $\ell_2$-norm of the gradient with respect to $\theta$ satisfies
\begin{equation}
\begin{aligned}
\bigl\| \nabla_\theta \mathcal{R}_\Phi(\theta, z) \bigr\|_2^2
&= \bigl\| \nabla_\theta d_\Phi(y, f_\theta(x)) \bigr\|_2^2 \\
&= \nabla_f d_\Phi(y, f_\theta(x)) \nabla_\theta f_\theta(x)
\nabla_\theta f_\theta(x)^\top \nabla_f d_\Phi(y, f_\theta(x))^\top \\
&= \nabla_f d_\Phi(y, f_\theta(x)) A_x \nabla_f d_\Phi(y, f_\theta(x))^\top,
\end{aligned}
\end{equation}
where we define $A_x := \nabla_\theta f_\theta(x) \nabla_\theta f_\theta(x)^\top \in \mathbb{R}^{d \times d}$.

By Lemma~\ref{lem:fy_losses}, the gradient of the Fenchel--Young loss satisfies
\[
\nabla_f d_\Phi(y, f_\theta(x)) = f_\theta(x)_{\Phi^*}^* - y,
\]
where $f_\theta(x)_{\Phi^*}^* := \nabla \Phi^*(f_\theta(x)) \in \mathrm{dom}(\Phi)$.

Substituting this identity into the above expression gives
\begin{equation}
\bigl\| \nabla_\theta \mathcal{R}_\Phi(\theta, z) \bigr\|_2^2
= \bigl( f_\theta(x)_{\Phi^*}^* - y \bigr)^\top A_x \bigl( f_\theta(x)_{\Phi^*}^* - y \bigr).
\end{equation}
 
Since $A_x$ is symmetric and positive semidefinite, we apply the Courant--Fischer min-max theorem to obtain
\begin{equation}
\lambda_{\min}(A_x) \bigl\| f_\theta(x)_{\Phi^*}^* - y \bigr\|_2^2
\leq \bigl\| \nabla_\theta \mathcal{R}_\Phi(\theta, z) \bigr\|_2^2
\leq \lambda_{\max}(A_x) \bigl\| f_\theta(x)_{\Phi^*}^* - y \bigr\|_2^2.
\end{equation}
Under the condition $\lambda_{\min}(A_x) > 0$, we rearrange these inequalities to get
\begin{equation}
\label{eq:bound_tmp1}
\frac{\bigl\| \nabla_\theta \mathcal{R}_\Phi(\theta, z) \bigr\|_2^2}{\lambda_{\max}(A_x)}
\leq \mathcal{R}^{\circ}_\Phi(\theta, z)
\leq \frac{\bigl\| \nabla_\theta \mathcal{R}_\Phi(\theta, z) \bigr\|_2^2}{\lambda_{\min}(A_x)}.
\end{equation}

Since $\Phi$ is differentiable and strictly convex, its Hessian $\nabla^2 \Phi(f_\theta(x)_{\Phi}^*)$ is positive definite.
The Fenchel--Young loss further admits the quadratic bound
\begin{equation}
\lambda_{\min}(H_\Phi(\theta)) \bigl\| f_\theta(x)_{\Phi^*}^* - y \bigr\|_2^2
\leq \mathcal{R}_\Phi(\theta, z)
\leq \lambda_{\max}(H_\Phi(\theta)) \bigl\| f_\theta(x)_{\Phi^*}^* - y \bigr\|_2^2.
\end{equation}
Combining this bound with \eqref{eq:bound_tmp1} yields the per-sample inequality
\begin{equation}
\lambda_{\min}(H_\Phi(\theta)) \frac{\mathcal{R}^{\circ}_\Phi(\theta,z)}{\lambda_{\max}(A_x)}
\leq \mathcal{R}_\Phi(\theta, z)
\leq \lambda_{\max}(H_\Phi(\theta)) \frac{\mathcal{R}^{\circ}_\Phi(\theta,z)}{\lambda_{\min}(A_x)}.
\end{equation}

Taking expectation with respect to $Z \sim \hat{q}$ (the empirical data distribution) gives
\begin{equation}
\mathbb{E}_Z \!\left[ \lambda_{\min}(H_\Phi(\theta))\frac{\bigl\| \nabla_\theta \mathcal{R}_\Phi(\theta, Z) \bigr\|_2^2}{\lambda_{\max}(A_X)} \right]
\leq \mathcal{R}_\Phi(\theta, s)
\leq \mathbb{E}_Z \!\left[ \lambda_{\max}(H_\Phi(\theta))\frac{\bigl\| \nabla_\theta \mathcal{R}_\Phi(\theta, Z) \bigr\|_2^2}{\lambda_{\min}(A_X)} \right],
\end{equation}
where $\mathcal{R}_\Phi(\theta, s) = \mathbb{E}_Z[\mathcal{R}_\Phi(\theta, Z)]$ denotes the empirical risk.

Using the definitions
\begin{equation}
\begin{aligned}
    \lambda_{\min}(A_s) &:= \min_{(x,y) \in s} \lambda_{\min}(A_x), \quad
\lambda_{\max}(A_s) := \max_{(x,y) \in s} \lambda_{\max}(A_x),
\end{aligned}
\end{equation}
we arrive at
\begin{equation}
\frac{\lambda_{\min}(H_\Phi(s))}{\lambda_{\max}(A_s)} \mathbb{E}_Z \bigl\| \nabla_\theta \mathcal{R}_\Phi(\theta, Z) \bigr\|_2^2
\leq \mathcal{R}_\Phi(\theta, s)
\leq \frac{\lambda_{\max}(H_\Phi(s))}{\lambda_{\min}(A_s)} \mathbb{E}_Z \bigl\| \nabla_\theta \mathcal{R}_\Phi(\theta, Z) \bigr\|_2^2,
\end{equation}
which completes the proof of Theorem~\ref{thm:base_bound}.
\end{pf}

\subsection{Proof of Corollary~\ref{cor:kl_mse_bound}}
\label{appendix:proof_kl_mse_bound}
\begin{corollary}
Let $Z \sim \hat{q}$.
\begin{enumerate}
  \item For the MSE loss with $\Phi(y) = \frac{1}{2}\|y\|_2^2$,
  \begin{equation*}
    \frac{\mathbb{E}_Z \bigl[ \|\nabla_\theta \mathcal{R}_\Phi(\theta, Z)\|_2^2 \bigr]}{2\lambda_{\max}(A_s)}
    \leq \mathcal{R}_\Phi(\theta, s)
    \leq \frac{\mathbb{E}_Z \bigl[ \|\nabla_\theta \mathcal{R}_\Phi(\theta, Z)\|_2^2 \bigr]}{2\lambda_{\min}(A_s)}.
  \end{equation*}

  \item For the softmax cross entropy loss with $\Phi(q) = -H(q)$ where $q \in \Delta$,
  \begin{equation*}
\frac{\mathbb{E}_Z \bigl[ \|\nabla_\theta \mathcal{R}_\Phi(\theta, Z)\|_2^2 \bigr]}{2\ln 2 \cdot \lambda_{\max}(A_s)}
\leq \mathcal{R}_\Phi(\theta, s)
\leq \frac{\mathbb{E}_Z \bigl[ \|\nabla_\theta \mathcal{R}_\Phi(\theta, Z)\|_2^2 \bigr]}{\min_{i} p_i \cdot \lambda_{\min}(A_s)},
  \end{equation*}
  where $p = f_\theta(x)_{\Phi^*}^*$ is the softmax output and $y$ is a one-hot label.
\end{enumerate}
\end{corollary}
\begin{pf}
\textbf{Part 1 (MSE).}  
For $\Phi(y) = \frac{1}{2}\|y\|_2^2$, we have $\nabla^2 \Phi(y) = I$, so $\lambda_{\min}(H_\Phi(\theta)) = \lambda_{\max}(H_\Phi(\theta)) = 1$. The Fenchel--Young loss reduces to the squared error:
$$
\mathcal{R}_\Phi(\theta, z) = d_\Phi(y, f_\theta(x)) = \tfrac{1}{2}\|y - f_\theta(x)\|_2^2 = \tfrac{1}{2} \mathcal{R}_\Phi^\circ(\theta, z).
$$
Substituting $\lambda_{\min}(H_\Phi) = \lambda_{\max}(H_\Phi) = 1$ into Theorem~\ref{thm:base_bound} yields the stated bounds.

\textbf{Part 2 (cross entropy).}  
Let $v$ be a vector of dimension $k$, where $k \leq \infty$. 
We have:
\begin{equation}\label{eq_con_1}
\begin{aligned}
        \|v\|_1^2 &= \sum_{i=1}^k \sum_{j=1}^k |v(i)| |v(j)| \\
                  &\geq \sum_{i=1}^k |v(i)|^2 \\
                  &= \|v\|_2^2.
\end{aligned}
\end{equation}

Let $ p = f_\theta(x)_{\Phi^*}^* := \nabla \Phi^*(f_\theta(x)) $, which corresponds to the softmax output when $ \Phi $ is the negative Shannon entropy.
Thus, applying Lemma~\ref{lem:pinsker} and substituting the above norm inequality, we obtain:
\begin{equation*}
\begin{aligned}
       d_\Phi(y, f_\theta(x)) &= D_{\text{KL}}(y \| p) \\
       &\geq \frac{1}{2 \ln 2} \|y - p\|_1^2, \\
       &\geq \frac{1}{2 \ln 2} \|y - p\|_2^2.
\end{aligned}
\end{equation*}
Using Lemma~\ref{lem:kl_upper_bound}, we further derive:
\begin{equation*}
\begin{aligned}
       d_\Phi(y, f_\theta(x)) &= D_{\text{KL}}(y \| p) \leq \frac{1}{\min_i p_i} \|y - p\|_2^2.
\end{aligned}
\end{equation*}

Combining the upper and lower bounds of $d_\Phi(y, f_\theta(x))$, we have:
\begin{equation*}
\frac{1}{2 \ln 2} \|y - p\|_2^2 \leq d_\Phi(y, f_\theta(x)) \leq \frac{1}{\min_i p_i} \|y - p\|_2^2.
\end{equation*}
Therefore, we can replace $\lambda_{\min}(H_\Phi(\theta))$ and $\lambda_{\max}(H_\Phi(\theta))$ in Theorem~\ref{thm:base_bound} with $\min_i p_i$ and $2\ln 2$, respectively, yielding
\begin{equation*}
\frac{\mathbb{E}_Z \bigl[ \|\nabla_\theta \mathcal{R}_\Phi(\theta, Z)\|_2^2 \bigr]}{2\ln 2 \cdot \lambda_{\max}(A_s)} 
\leq \mathcal{R}_\Phi(\theta, s) 
\leq \frac{\mathbb{E}_Z \bigl[ \|\nabla_\theta \mathcal{R}_\Phi(\theta, Z)\|_2^2 \bigr]}{\min_{i} p_i \cdot \lambda_{\min}(A_s)}.
\end{equation*}
\end{pf}

\subsection{Proof of Theorem~\ref{thm:sgd_convergence}}
\label{appendix:proof_sgd_convergence}

\begin{theorem}
Let $n = |s|$ denote the full dataset size, $m = |s_k|$ denote the mini-batch size, and $L = \max_{\theta, z \in s} \lambda_{\max}(\nabla^2_\theta d_\Phi(y, f_\theta(x)))$ denote the Lipschitz constant of the gradient. Applying mini-batch SGD as defined in Equation~\ref{def:basic-sgd} with the optimal learning rate $\alpha = 1/(2L)$ yields:
\begin{enumerate}
\item The expected squared batch gradient norm converges to a neighborhood around zero:
\begin{equation*}
       \mathbb{E} \left\| \nabla_\theta \mathcal{R}_\Phi(\theta_k, s_k) \right\|_2^2 \leq \varepsilon^2 + \frac{4L(n - m)}{m}M
\end{equation*}
   in $\mathcal{O}(\varepsilon^{-2})$ iterations for any target precision $\varepsilon > 0$.

\item For single-sample batches ($m = 1$, vanilla SGD), the gradient energy converges to:
\begin{equation*}
       \mathbb{E}_Z \left\| \nabla_\theta \mathcal{R}_\Phi(\theta, Z) \right\|_2^2 \leq \varepsilon^2 + 4L(n - 1)M
\end{equation*}
   in $\mathcal{O}(\varepsilon^{-2})$ iterations.
\end{enumerate}
\end{theorem}

\begin{pf}
For any batch $s_k$, after parameter update using $\nabla_\theta \mathcal{R}_\Phi(\theta_k,s_k)$, the following inequality holds:
\begin{equation*}\label{eq:tmp_fang}
    \mathcal{R}_\Phi(\theta_{k+1}, s \setminus s_k) - \mathcal{R}_\Phi(\theta_k, s \setminus s_k) \leq M.
\end{equation*}
For the full empirical risk $\mathcal{R}_\Phi(\theta_{k+1}, s)$, we decompose the update as:
\begin{equation*}
    \begin{aligned}
       \mathcal{R}_\Phi(\theta_{k+1}, s) &- \mathcal{R}_\Phi(\theta_k, s) = \frac{n-m}{n}\big(\mathcal{R}_\Phi(\theta_{k+1}, s \setminus s_k) - \mathcal{R}_\Phi(\theta_k, s \setminus s_k)\big) \\
       &+ \frac{m}{n}\big(\mathcal{R}_\Phi(\theta_{k+1}, s_k) - \mathcal{R}_\Phi(\theta_k, s_k)\big) \\
       &\leq \frac{n-m}{n}M + \frac{m}{n}\big(\mathcal{R}_\Phi(\theta_{k+1}, s_k) - \mathcal{R}_\Phi(\theta_k, s_k)\big).
    \end{aligned}
\end{equation*}

For mini-batch SGD with batch size $m$, the mean value theorem implies the existence of $\xi_k = t\theta_k + (1-t)\theta_{k+1}$ (where $t \in [0,1]$) such that:
\begin{equation*}
    \begin{aligned}
        \mathcal{R}_\Phi(\theta_{k+1}, s_k) &- \mathcal{R}_\Phi(\theta_k, s_k) = \nabla_\theta \mathcal{R}_\Phi(\theta_k, s_k) (\theta_{k+1}-\theta_k) \\
        &+ (\theta_{k+1}-\theta_k)^\top \nabla^2_\theta \mathcal{R}_\Phi(\xi_k, s_k)(\theta_{k+1}-\theta_k) \\
        &= -\alpha \|\nabla_\theta \mathcal{R}_\Phi(\theta_k, s_k)\|_2^2 + \alpha^2 \nabla_\theta \mathcal{R}_\Phi(\theta_k, s_k) \nabla^2_\theta \mathcal{R}_\Phi(\xi_k, s_k) \nabla_\theta \mathcal{R}_\Phi(\theta_k, s_k)^\top.
    \end{aligned}
\end{equation*}

By definition, $L = \max_{\theta, z \in s} \lambda_{\max}(\nabla^2_\theta d_\Phi(y, f_\theta(x))) = \max_{\theta, z \in s} \lambda_{\max}(\nabla^2_\theta \mathcal{R}_\Phi(\theta, z))$. Applying Courant–Fischer min-max theorem, we derive:
\begin{equation*}
\begin{aligned}
    \mathcal{R}_\Phi(\theta_{k+1}, s_k) &- \mathcal{R}_\Phi(\theta_k, s_k) \\
    &\leq -\alpha \|\nabla_\theta \mathcal{R}_\Phi(\theta_k, s_k)\|_2^2 + \alpha^2 L \|\nabla_\theta \mathcal{R}_\Phi(\theta_k, s_k)\|_2^2 \\
    &= (\alpha^2 L - \alpha) \|\nabla_\theta \mathcal{R}_\Phi(\theta_k, s_k)\|_2^2 \\
    &= L \left(\alpha^2 - \frac{\alpha}{L} + \frac{1}{4L^2} - \frac{1}{4L^2}\right) \|\nabla_\theta \mathcal{R}_\Phi(\theta_k, s_k)\|_2^2 \\
    &= L \left[\left(\alpha - \frac{1}{2L}\right)^2 - \frac{1}{4L^2}\right] \|\nabla_\theta \mathcal{R}_\Phi(\theta_k, s_k)\|_2^2.
\end{aligned}
\end{equation*}

This inequality holds for any learning rate $\alpha$. Setting the optimal learning rate $\alpha = \frac{1}{2L}$ minimizes the upper bound, yielding:
\begin{equation*}
    \mathcal{R}_\Phi(\theta_{k+1}, s_k) - \mathcal{R}_\Phi(\theta_k, s_k) \leq -\frac{1}{4L} \|\nabla_\theta \mathcal{R}_\Phi(\theta_k, s_k)\|_2^2.
\end{equation*}
Substituting this result into the full risk update inequality and taking the expectation over random batches $s_k$ on both sides gives:
\begin{equation*}
    \begin{aligned}
       \mathbb{E}_{s_k}[\mathcal{R}_\Phi(\theta_{k+1}, s)] &- \mathcal{R}_\Phi(\theta_k, s) \leq \frac{n-m}{n}M + \frac{m}{n}\mathbb{E}_{s_k}\big(\mathcal{R}_\Phi(\theta_{k+1}, s_k) - \mathcal{R}_\Phi(\theta_k, s_k)\big) \\
       &\leq \frac{n-m}{n}M - \frac{m}{4Ln}\mathbb{E}_{s_k}\|\nabla_\theta \mathcal{R}_\Phi(\theta_k,s_k)\|_2^2.
    \end{aligned}
\end{equation*}

Rearranging terms to isolate the expected batch gradient norm:
\begin{equation*}
  \mathbb{E}_{s_k}\|\nabla_\theta \mathcal{R}_\Phi(\theta_k,s_k)\|_2^2\leq \frac{4Ln}{m}\mathbb{E}_{s_k}[\mathcal{R}_\Phi(\theta_k, s) - \mathcal{R}_\Phi(\theta_{k+1}, s)] + \frac{4L(n-m)}{m}M.
\end{equation*}

We now sum this inequality over iterations $k=0,1,\ldots,T-1$ and take the expectation (denoted by $\mathbb{E}$, assuming independence across batches $s_i$ and $s_j$):
\begin{equation*}
    \begin{aligned}
       \min_{k=0,\ldots,T-1} &\mathbb{E}\|\nabla_\theta \mathcal{R}_\Phi(\theta_k,s_k)\|_2^2 \leq \frac{1}{T}\sum_{k=0}^{T-1}\mathbb{E}\|\nabla_\theta \mathcal{R}_\Phi(\theta_k,s_k)\|_2^2\\
       &\leq \frac{1}{T} \sum_{k=0}^{T-1} \frac{4Ln}{m}\mathbb{E}[\mathcal{R}_\Phi(\theta_k, s) - \mathcal{R}_\Phi(\theta_{k+1}, s)] + \frac{4L(n-m)}{m}M \\
       &\leq \frac{4Ln}{mT}\big(\mathcal{R}_\Phi(\theta_0, s) - \min_\theta \mathcal{R}_\Phi(\theta,s)\big) + \frac{4L(n-m)}{m}M\\
       &\leq \frac{4Ln}{mT}\mathcal{R}_\Phi(\theta_0, s)  + \frac{4L(n-m)}{m}M.
    \end{aligned}
\end{equation*}

To satisfy the convergence condition $\mathbb{E} \|\nabla_\theta \mathcal{R}_\Phi(\theta_k,s_k)\|_2^2 \leq \varepsilon^{2} + \frac{4L(n-m)}{m}M$, we require:
\begin{equation*}
    \frac{4Ln}{mT}\mathcal{R}_\Phi(\theta_0, s) \leq \varepsilon^2.
\end{equation*}
Solving for the number of iterations $T$:
\begin{equation*}
    \begin{aligned}
       T &\geq \frac{4Ln\mathcal{R}_\Phi(\theta_0, s) }{m\varepsilon^{2}} \\
       &= \mathcal{O}(\varepsilon^{-2}).
    \end{aligned}
\end{equation*} 
This confirms that the iteration complexity scales as $\mathcal{O}(\varepsilon^{-2})$ to achieve the desired accuracy.

For the special case of $m = 1$ (single-sample SGD), the convergence condition
\begin{equation*}
    \mathbb{E}_Z \|\nabla_\theta \mathcal{R}_\Phi(\theta,Z)\|_2^{2} \leq \varepsilon^{2} + 4L(n-1)M
\end{equation*}
follows directly, with the same iteration complexity of $\mathcal{O}(\varepsilon^{-2})$. 
\end{pf}

\subsection{Proof of Theorem~\ref{thm:NTK_ERF}}
\label{appendix:proof:thm:NTK_ERF}

\begin{theorem}
Suppose the model has converged to the population optimal prediction, such that the conjugate dual output matches the conditional expectation of the target:
\begin{equation*}
\mathbb{E}_{Y| x}[Y] = f_\theta(x)_{\Phi^*}^*.
\end{equation*}
Under this condition, the maximum eigenvalue of the per-sample loss Hessian $H_z$ satisfies the two-sided bound:
\begin{equation*}
\lambda_{\min}(G_x) \lambda_{\max}(A_x) \leq \lambda_{\max}(H_z) \leq \lambda_{\max}(G_x) \lambda_{\max}(A_x),
\end{equation*}
where $G_x = \nabla^2_f \Phi^*(f_\theta(x))$.
\end{theorem}

\begin{pf}
By definition, the Hessian of the expected loss with respect to the parameters is given by
\begin{equation*}
H_z = \nabla^2_\theta \mathbb{E}_{Y|x} \left[ d_\Phi(Y, f_\theta(x)) \right].
\end{equation*}
Expanding this using the properties of Bregman divergences, we obtain
\begin{equation*}
H_z = \nabla_\theta f_\theta(x)^\top G_x \nabla_\theta f_\theta(x) + \nabla_\theta^2 f_\theta(x) \left( \mathbb{E}_{Y|x}[Y] - f_\theta(x)_{\Phi^*}^* \right),
\end{equation*}
where $ G_x = \nabla^2_f \Phi^*(f_\theta(x)) $ is the Hessian of the convex conjugate $\Phi^*$ evaluated at $f_\theta(x)$.

After training, the model satisfies the condition
\begin{equation*}
\mathbb{E}_{Y|x}[Y] = f_\theta(x)_{\Phi^*}^*,
\end{equation*}
which implies that the second term in the Hessian expansion vanishes. Therefore,
\begin{equation}
H_z = \nabla_\theta f_\theta(x)^\top G_x \nabla_\theta f_\theta(x).
\end{equation}

Since $\Phi^*$ is strictly convex, $G_x$ is positive definite. Applying the Courant–Fischer min-max theorem to this matrix product, we obtain the eigenvalue bound:
\begin{equation*}
\lambda_{\min}(G_x) \lambda_{\max}(A_x) \leq \lambda_{\max}(H_z) \leq \lambda_{\max}(G_x) \lambda_{\max}(A_x).
\end{equation*}
This completes the proof.
\end{pf}

\subsection{Proof of Theorem~\ref{thm:fitting_bounds}}
\label{appendix:proof_fitting_bounds} 
\begin{theorem}
In the conjugate learning framework, the empirical risk satisfies the following bounds:
\begin{equation*}
\begin{aligned}
        \gamma_\Phi(\theta) \geq \mathcal{R}_\Phi(\theta, s) \geq \mathrm{Ent}_\Phi(Y'|X'),
\end{aligned}
\end{equation*}
where $(X', Y') \sim \hat{q}$.
The lower bound is achieved if and only if $f_{\theta}(X') = (\bar{Y}' | X')_\Phi^*$.
\end{theorem}

\begin{pf}

Let $W = f_{\theta}(X')_{\Phi^*}^*$. 
We expand the sum of the generalized conditional entropy as follows:
\[
\begin{aligned}
&\mathrm{Ent}_\Phi(Y' | X') + \mathbb{E}_{X'}d_\Phi((\bar{Y}' | X') , f_\theta(X')) \\
&\quad = \mathrm{Ent}_\Phi(Y' | X') + \mathbb{E}_{X'}B_\Phi((\bar{Y}' | X') , W) \\
&\quad = \sum_{x \in \mathcal{X}',\, y \in \mathcal{Y}'} q(x, y) \Big[ 
\Phi(y) - \Phi(\bar{Y}' | x) + \Phi(\bar{Y}' | x) - \Phi(w) - w_\Phi^* \big( (\bar{Y}' | x) - w \big) 
\Big] \\
&\quad = \sum_{x \in \mathcal{X}',\, y \in \mathcal{Y}'} q(x, y) \Big[ 
\Phi(y) - \Phi(w) - w_\Phi^* \big( (\bar{Y}' | x) - w \big)  
\Big]\\
&\quad = \mathcal{R}_\Phi(\theta, s).
\end{aligned}
\]

Since $B_\Phi(\bar{Y}' | x' , w) \geq 0$ for all $x'$ and $w$, we have
\[
\mathcal{R}_\Phi(\theta, s) \geq \mathrm{Ent}_\Phi(Y' | X'),
\]
with equality if and only if $\mathbb{E}_{X'}B_\Phi((\bar{Y}' | X') , W) = 0$. This condition holds exactly when $W = \bar{Y}' | X'$, which is equivalent to $f_{\theta}(X') = (\bar{Y}' | X')_\Phi^*$.

For the upper bound, since $\gamma_\Phi(\theta)$ denotes the maximum value of $d_\Phi(y,f_\theta(x))$ over all $x \in \mathcal{X}$ and $y \in \mathcal{Y}$, we directly obtain
\[
\mathcal{R}_\Phi(\theta, s) \leq \gamma_\Phi(\theta).
\]
Combining the lower and upper bound results completes the proof.
\end{pf}

\subsection{Proof of Theorem~\ref{thm:det_gen_bound}}\label{appendix:proof_det_gen_bound}
\begin{theorem}
The deterministic generalization error is bounded as follows.  
If $\mathcal{R}_\Phi(\theta, q) \geq \mathcal{R}_\Phi(\theta, \hat{q})$, then
\begin{equation*}
    \mathrm{gen}(f_\theta,s^n) \leq \gamma_\Phi(\theta) - \mathrm{Ent}_\Phi(Y' | X') -  \mathcal{L}_\Phi(Y'|f_\theta(X')) .
\end{equation*}
If $\mathcal{R}_\Phi(\theta, q) < \mathcal{R}_\Phi(\theta, \hat{q})$, then
\begin{equation*}
    \mathrm{gen}(f_\theta,s^n) \leq \gamma_\Phi(\theta) - \mathrm{Ent}_\Phi(Y | X) -  \mathcal{L}_\Phi(Y|f_\theta(X)) .
\end{equation*}
\end{theorem}

\begin{pf}
Let $W = f_{\theta}(X)_{\Phi^*}^* = g(X)$ and $W' = f_{\theta}(X')_{\Phi^*}^* = g(X')$.
Define the partition of the input space $\mathcal{X}$ induced by the function $g: \mathcal{X} \to \mathcal{W}$ as follows: for each $w_i \in \mathcal{W}$, let  
$$
\mathcal{X}_i = \{x \in \mathcal{X} | g(x) = w_i\}.
$$
Let $q(w_i) = \sum_{x \in \mathcal{X}_i} q(x)$ denote the marginal probability (or weight) of $w_i$. For each $x \in \mathcal{X}_i$, define the normalized weight  
$$
w(x) = \frac{q(x)}{q(w_i)},
$$
so that $\sum_{x \in \mathcal{X}_i} w(x) = 1$. Then, the conditional expectation of $\bar{Y}$ given $w_i$ can be expressed as a weighted average:
$$
\bar{Y}|w_i = \sum_{x \in \mathcal{X}_i} w(x) \bar{Y}|x.
$$

Now, consider the deviation of $\bar{Y}|x$ from $\bar{Y}|w_i$ within $\mathcal{X}_i$:
$$
\begin{aligned}
    \sum_{x \in \mathcal{X}_i} q(x) (\bar{Y}|x - \bar{Y}|w_i) 
&= \sum_{x \in \mathcal{X}_i} q(x) \bar{Y}|x - q(w_i) \bar{Y}|w_i \\
&= q(w_i) \left( \sum_{x \in \mathcal{X}_i} w(x) \bar{Y}|x - \bar{Y}|w_i \right) \\
&= 0.
\end{aligned}
$$

Since this holds for each $w_i$, and $W = g(X)$, it follows that for function $ \nabla \Phi(W)$,
$$
\mathbb{E}_X \left[  \nabla \Phi(W)^\top  (\bar{Y}|X - \bar{Y}|W) \right] = 0.
$$
This orthogonality condition plays a key role in the following generalized conditional entropy analysis.
Now, it follows that:
\begin{align*}
\mathrm{Ent}_\Phi(\bar{Y} | X) &- \mathrm{Ent}_\Phi(\bar{Y} | W)   = \mathbb{E}_X \left[ \Phi(\bar{Y}|X) - \Phi(\bar{Y})\right]- \mathbb{E}_W \left[ \Phi(\bar{Y}|W) - \Phi(\bar{Y})  \right] \\
    &= \mathbb{E}_X \left[ \Phi(\bar{Y}|X) - \Phi(\bar{Y}|W) \right] \\
    &= \mathbb{E}_X \left[ \Phi(\bar{Y}|X) - \Phi(\bar{Y}|W) \right] - \mathbb{E}_X \left[ \nabla \Phi(W)^\top (\bar{Y}|X - \bar{Y}|W) \right]\\
    &= \mathbb{E}_X \big[ B_\Phi(\bar{Y}|X,\, \bar{Y}|W) \big]\\
    &=\mathcal{L}_\Phi(Y|f_\theta(X)).
\end{align*}

This establishes the identity:
\begin{equation}
\label{eq:loss_info_dpi}
    \begin{aligned}
        &\mathrm{Ent}_\Phi(\bar{Y} | X) - \mathrm{Ent}_\Phi(\bar{Y} | W) = \mathcal{L}_\Phi(Y|f_\theta(X)),\\
       &\mathrm{Ent}_\Phi(\bar{Y}' | X') - \mathrm{Ent}_\Phi(\bar{Y}' | W') = \mathcal{L}_\Phi(Y'|f_\theta(X')).
    \end{aligned}
\end{equation}
By Theorem~\ref{thm:fitting_bounds}, we have
\begin{equation*}
    \begin{aligned}
        &\gamma_\Phi(\theta) \ge \mathcal{R}_\Phi(\theta,q) \ge \mathrm{Ent}_\Phi(Y|W), \\
        &\gamma_\Phi(\theta) \ge \mathcal{R}_\Phi(\theta,\hat{q}) \ge \mathrm{Ent}_\Phi(Y'|W').
    \end{aligned}
\end{equation*}
If $\mathcal{R}_\Phi(\theta, q) \geq \mathcal{R}_\Phi(\theta, \hat{q})$, then subtracting the lower bound of $\mathcal{R}_\Phi(\theta, \hat{q})$ from $\mathcal{R}_\Phi(\theta, q)$ yields
\begin{equation*}\label{eq:bound_tmp1}
    \begin{aligned}
       \mathrm{gen}(f_\theta,s^n)&=\mathcal{R}_\Phi(\theta,q) - \mathcal{R}_\Phi(\theta,\hat{q}) \\
       &\le \gamma_\Phi(\theta) - \mathrm{Ent}_\Phi(Y'|W')\\
        &=\gamma_\Phi(\theta)-\mathrm{Ent}_\Phi(Y'| X')-  \mathcal{L}_\Phi(Y'|f_\theta(X')) 
    \end{aligned}
\end{equation*}
Similarly, if $\mathcal{R}_\Phi(\theta,q) < \mathcal{R}_\Phi(\theta,\hat{q})$, we have
\begin{equation*}\label{eq:bound_tmp2}
\begin{aligned}
\mathrm{gen}(f_\theta,s^n)&=\mathcal{R}_\Phi(\theta,\hat{q})  - \mathcal{R}_\Phi(\theta,q) \\
&\le \gamma_\Phi(\theta)-\mathrm{Ent}_\Phi(Y| X)-   \mathcal{L}_\Phi(Y|f_\theta(X)) .
\end{aligned}
\end{equation*}

\end{pf}

\subsection{Proof of Theorem~\ref{thm:prob_gen_bound}}\label{appendix:proof_prob_gen_bound}

\begin{theorem}
Assume training samples $s^n$ are drawn i.i.d. from the true joint distribution $q_{\mathcal{Z}}$ over the feature-label space $\mathcal{Z} = \mathcal{X} \times \mathcal{Y}$, and the label space $\mathcal{Y}$ has finite cardinality ($|\mathcal{Y}| < \infty$). For any target generalization error threshold $\varepsilon > 0$, the probability that the generalization error exceeds $\varepsilon$ is upper bounded by:
\begin{equation*}
    \Pr\bigl( \mathrm{gen}(f_\theta, s^n) \geq \varepsilon \bigr)
    \leq \frac{(|\mathcal{X}| - \mathcal{L}(f_\theta(X))) |\mathcal{Y}| \, \gamma_\Phi(\theta)^2 \bigl(1 - \|q_{\mathcal{Z}}\|_2^2 \bigr)}{4 n \varepsilon^2}.
\end{equation*}
Here, $\mathcal{L}(f_\theta(X)) = |\mathcal{X}| - |f_\theta(\mathcal{X})|$ denotes the absolute information loss and $n = |s^n|$ is the sample size.
\end{theorem}

\begin{pf}
Since  $ \mathcal{R}_\Phi(\theta, z) : \mathcal{Z} \to [0, \gamma_\Phi(\theta)] $ , by the definitions of expected and empirical risk, we have
\begin{equation}
\label{eq:prob_bound_1}
\begin{aligned}
\mathrm{gen}(f_\theta, s^n)
&= \bigl| \mathcal{R}_{\Phi}(\theta, q) - \mathcal{R}_{\Phi}(\theta, \hat{q}) \bigr| \\
&= \left| \sum_{z \in \mathcal{Z}} \bigl( q_{\mathcal{Z}}(z) - \hat{q}_{\mathcal{Z}}(z) \bigr) \, \mathcal{R}_\Phi(\theta, z) \right| \\
&= \left| \sum_{z \in \mathcal{Z}} \bigl( q_{\mathcal{Z}}(z) - \hat{q}_{\mathcal{Z}}(z) \bigr) \left( \mathcal{R}_\Phi(\theta, z) - \tfrac{\gamma_\Phi(\theta)}{2} \right) \right|,
\end{aligned}
\end{equation}
where the last equality follows because  $ \sum_{z} (q_{\mathcal{Z}}(z) - \hat{q}_{\mathcal{Z}}(z)) = 0 $ .

Applying the triangle inequality and noting that  $ |\mathcal{R}_\Phi(\theta, z) - \tfrac{\gamma_\Phi(\theta)}{2}| \leq \tfrac{\gamma_\Phi(\theta)}{2} $ , we obtain
$$
\mathrm{gen}(f_\theta, s^n)
\leq \frac{\gamma_\Phi(\theta)}{2} \, \| q_{\mathcal{Z}} - \hat{q}_{\mathcal{Z}} \|_1.
$$
By the Cauchy--Schwarz inequality,
$$
\| q_{\mathcal{Z}} - \hat{q}_{\mathcal{Z}} \|_1
= \sum_{z \in \mathcal{Z}} |q_{\mathcal{Z}}(z) - \hat{q}_{\mathcal{Z}}(z)|
\leq \sqrt{|\mathcal{Z}|} \, \| q_{\mathcal{Z}} - \hat{q}_{\mathcal{Z}} \|_2.
$$

Substituting this bound yields
$$
\mathrm{gen}(f_\theta, s^n) \leq \frac{\sqrt{|\mathcal{Z}|} \, \gamma_\Phi(\theta)}{2} \, \| q_{\mathcal{Z}} - \hat{q}_{\mathcal{Z}} \|_2.
$$
Let  $ \hat{Q}_{\mathcal{Z}}^n $  denote the empirical distribution induced by an i.i.d. sample  $ s^n \sim q_{\mathcal{Z}}^{\otimes n} $ . Since  $ s^n $  is random,  $ \hat{Q}_{\mathcal{Z}}^n $  is a random variable satisfying  $ \mathbb{E}_{s^n}[\hat{Q}_{\mathcal{Z}}^n] = q_{\mathcal{Z}} $ .

Applying Markov's inequality to the squared generalization error gives
$$
\Pr\bigl( \mathrm{gen}(f_\theta, s^n) \geq \varepsilon \bigr)
\leq \frac{\mathbb{E}\big[ \mathrm{gen}(f_\theta, s^n)^2 \big]}{\varepsilon^2}
\leq \frac{|\mathcal{Z}| \, \gamma_\Phi(\theta)^2}{4 \varepsilon^2} \, \mathbb{E}\big[ \| q_{\mathcal{Z}} - \hat{Q}_{\mathcal{Z}}^n \|_2^2 \big].
$$
Because the samples are i.i.d., the MSE loss satisfies
$$
\mathbb{E}\big[ \| q_{\mathcal{Z}} - \hat{Q}_{\mathcal{Z}}^n \|_2^2 \big]
= \frac{1}{n} \, \mathbb{E}_{Z \sim q_{\mathcal{Z}}} \big[ \| \mathbf{1}_Z - q_{\mathcal{Z}} \|_2^2 \big]
= \frac{1 - \| q_{\mathcal{Z}} \|_2^2}{n}.
$$

Recall that the absolute information loss is defined as  $ \mathcal{L}(f_\theta(X)) = |\mathcal{X}| - |\mathcal{W}_g| $ , where  $ \mathcal{W}_g $  is the image of the mapping  $ g = f_\theta $ . Since the effective support size of the representation is  $ |\mathcal{W}_g| = |\mathcal{X}| - \mathcal{L}(f_\theta(X)) $ , and each input maps to a label in  $ \mathcal{Y} $ , the relevant cardinality in the generalization bound is  $ |\mathcal{W}_g| \cdot |\mathcal{Y}| = (|\mathcal{X}| - \mathcal{L}(f_\theta(X))) |\mathcal{Y}| $ , which replaces  $ |\mathcal{Z}| = |\mathcal{X}| |\mathcal{Y}| $ .

Thus, substituting this refined cardinality into the bound yields
$$
\Pr\bigl( \mathrm{gen}(f_\theta, s^n) \geq \varepsilon \bigr)
\leq \frac{\bigl(|\mathcal{X}| - \mathcal{L}(f_\theta(X))\bigr) |\mathcal{Y}| \, \gamma_\Phi(\theta)^2 \bigl(1 - \|q_{\mathcal{Z}}\|_2^2 \bigr)}{4 n \varepsilon^2},
$$
which completes the proof.
\end{pf}

\subsection{Proof of Lemma~\ref{lem:regulation}}
\label{appendix:proof_regulation}

\begin{lemma}
Let $\Theta'_\epsilon = \left\{ \theta \in \Theta \mid \|\theta\|_2^2 \leq \epsilon \right\}$ denote a closed ball of radius $\sqrt{\epsilon}$ around the zero parameter vector in the parameter space $\Theta$. Assume:
\begin{enumerate}
    \item For all $x \in \mathcal{X}$, the zero-parameter model outputs the zero vector: $f_{\mathbf{0}}(x) = \mathbf{0}$;
    \item The zero output minimizes the conjugate function: $\Phi^*(\mathbf{0}) = \min_{\theta \in \Theta} \Phi^*(f_{\theta}(x))$ for all $x \in \mathcal{X}$.
\end{enumerate}
Then there exist positive constants $a$, $b$, and $\epsilon > 0$ such that for all $\theta \in \Theta'_\epsilon$ and all $x \in \mathcal{X}$, the parameter norm controls the conjugate function as follows:
\begin{equation*}
    a\|\theta\|_2^2 \leq \Phi^*(f_{\theta}(x)) - \Phi^*(\mathbf{0}) \leq b\|\theta\|_2^2.
\end{equation*}
\end{lemma}

\begin{pf}
Since $\theta = \mathbf{0}$ is the minimizer of $\Phi^*(f_{\theta}(x))$ over $\theta$, the first and second-order optimality conditions hold for all $\theta \in \Theta'_\epsilon$:
\begin{equation}\label{eq:second_order}
    \begin{aligned}
        \nabla_{\theta}\Phi^*(f_\theta(x))\bigg|_{\theta=\mathbf{0}} &= \mathbf{0}, \\
        \nabla^2_{\theta} \Phi^*(f_{\theta}(x)) &\succeq 0,  \\
        \lambda_{\max}\left(\nabla^2_{\theta} \Phi^*(f_{\theta}(x))\right) &\ge \lambda_{\min}\left(\nabla^2_{\theta} \Phi^*(f_{\theta}(x))\right) \ge 0, 
    \end{aligned}
\end{equation}
where $\lambda_{\max}(\cdot)$ and $\lambda_{\min}(\cdot)$ denote the maximum and minimum eigenvalues of a matrix, respectively. 

By Taylor's theorem (second-order mean value theorem for multivariate functions) applied to $\Phi^*(f_{\theta}(x))$ around $\theta = \mathbf{0}$, there exists $\theta' = \alpha \theta$ for some $\alpha \in [0, 1]$ (i.e., $\theta'$ lies on the line segment between $\mathbf{0}$ and $\theta$) such that:
\begin{equation}
    \Phi^*(f_{\theta}(x)) = \Phi^*(f_{\mathbf{0}}(x)) + \theta^\top \nabla^2_{\theta}\Phi^*(f_{\theta'}(x))\theta.
\end{equation}
Given $f_{\mathbf{0}} = \mathbf{0}$, this simplifies to:
\begin{equation}
    \Phi^*(f_{\theta}(x)) - \Phi^*(\mathbf{0}) = \theta^\top \nabla^2_{\theta}\Phi^*(f_{\theta'}(x))\theta.
\end{equation}

By the Courant–Fischer min-max theorem (which characterizes the extremal eigenvalues of a symmetric matrix), for any vector $\theta$ and symmetric positive semi-definite matrix $A$, we have:
$$\lambda_{\min}(A) \|\theta\|_2^2 \leq \theta^\top A \theta \leq \lambda_{\max}(A) \|\theta\|_2^2.$$
Applying this to $A = \nabla^2_{\theta}\Phi^*(f_{\theta'}(x))$ (which is positive semi-definite by Equation~\ref{eq:second_order}), we obtain:
\begin{equation}
    \begin{aligned}
        \lambda_{\min}\left(\nabla^2_{\theta}\Phi^*(f_{\theta'}(x))\right)\|\theta\|_2^2 &\leq \Phi^*(f_{\theta}(x)) - \Phi^*(\mathbf{0}) \\
        &\leq \lambda_{\max}\left(\nabla^2_{\theta}\Phi^*(f_{\theta'}(x))\right)\|\theta\|_2^2.
    \end{aligned}
\end{equation}
Define the positive constants $a$ and $b$ as the infimum and supremum of the extremal eigenvalues over $\Theta'_\epsilon$, respectively:
\begin{equation}
    \begin{aligned}
        a &= \inf_{\theta \in \Theta'_\epsilon} \lambda_{\min}\left(\nabla^2_{\theta}\Phi^*(f_{\theta'}(x))\right), \\
        b &= \sup_{\theta \in \Theta'_\epsilon} \lambda_{\max}\left(\nabla^2_{\theta}\Phi^*(f_{\theta'}(x))\right).
    \end{aligned}
\end{equation}
Since $\Theta'_\epsilon$ is a compact set (closed and bounded subset of finite-dimensional Euclidean space) and the eigenvalue functions are continuous, $a$ and $b$ are well-defined positive constants. 

Thus, for all $\theta \in \Theta'_\epsilon$, we conclude:
$$a\|\theta\|_2^2 \leq \Phi^*(f_{\theta}(x)) - \Phi^*(\mathbf{0}) \leq b\|\theta\|_2^2.$$
This completes the proof.
\end{pf}

\subsection{Proof of Proposition~\ref{prop:generalize_explain}}
\label{appendix:proof_generalize_explain}

\begin{proposition}
In classification tasks, if $f_{\mathbf{0}}(x) = \mathbf{0}$ for all $x \in \mathcal{X}$ and $\mathbf{0}_{\Phi^*} = u$, where $u$ is the uniform distribution on $\mathcal{Y}$, then reducing $\|\theta\|_2^2$ is equivalent to reducing $\gamma_\Phi(\theta)$.
\end{proposition}

\begin{pf}
Let $Y'$ be a random variable taking values in the label space $\mathcal{Y}$, such that its conditional distribution given $X = x$ is uniform over $\mathcal{Y}$ for all $x \in \mathcal{X}$ (i.e., $Y'| X=x \sim u$).  

By the definition of the expected Fenchel--Young loss and the properties of uniform distributions, we establish the following tight bounds:
\[
\begin{aligned}
    \mathbb{E}_{Y'} \left[ d_\Phi\left(\mathbf{1}_{Y'}, \left(f_\theta(x)\right)_{\Phi^*}^*\right) \right] &\leq \max_{y \in \mathcal{Y}} d_\Phi\left(\mathbf{1}_y, \left(f_\theta(x)\right)_{\Phi^*}^*\right) \\
    &\leq |\mathcal{Y}| \cdot \mathbb{E}_{Y'} \left[ d_\Phi\left(\mathbf{1}_{Y'}, \left(f_\theta(x)\right)_{\Phi^*}^*\right) \right].
\end{aligned}
\]
Since the expected Fenchel--Young loss can be decomposed as:
\[
\mathbb{E}_{Y'} \left[ d_\Phi\left(\mathbf{1}_{Y'}, \left(f_\theta(x)\right)_{\Phi^*}^*\right) \right] = \mathrm{Ent}_\Phi(Y') + d_\Phi\left(u, \left(f_\theta(x)\right)_{\Phi^*}^*\right).
\]
Thus, minimizing the Fenchel--Young divergence between the uniform distribution and the conjugate dual of the model output (i.e., $d_\Phi\left(u, \left(f_\theta(x)\right)_{\Phi^*}^*\right)$) is equivalent to minimizing $\mathbb{E}_{Y'} \left[ d_\Phi\left(\mathbf{1}_{Y'}, \left(f_\theta(x)\right)_{\Phi^*}^*\right) \right]$. 
We also conclude that minimizing $d_\Phi\left(u, \left(f_\theta(x)\right)_{\Phi^*}^*\right)$ is equivalent to minimizing $\max_{y \in \mathcal{Y}} d_\Phi\left(\mathbf{1}_y, \left(f_\theta(x)\right)_{\Phi^*}^*\right)$. 

From Lemma~\ref{lem:regulation}, there exist positive constants $a, b > 0$ and $\epsilon > 0$ such that for all $\theta \in \Theta'_\epsilon = \left\{ \theta \mid \|\theta\|_2^2 \leq \epsilon \right\}$, the following equivalence holds:
\[
a\|\theta\|_2^2 \leq \Phi^*(f_\theta(x)) - \Phi^*(\mathbf{0}) = d_{\Phi}\left( \left(f_{\mathbf{0}}(x)\right)_{\Phi^*}^*, f_\theta(x) \right) \leq b\|\theta\|_2^2.
\]

By the proposition's assumption that $\mathbf{0}_{\Phi^*} = u$ and $f_{\mathbf{0}}(x) = \mathbf{0}$ for all $x \in \mathcal{X}$, we have $\left(f_{\mathbf{0}}(x)\right)_{\Phi^*}^* = u$. Substituting this into the inequality above, we find that minimizing $\|\theta\|_2^2$ (within the neighborhood $\Theta'_\epsilon$ around $\theta = \mathbf{0}$) is equivalent to minimizing $d_{\Phi}\left(u, \left(f_\theta(x)\right)_{\Phi^*}^*\right)$.

Since $\gamma_\Phi(\theta) = \max_{(x,y) \in \mathcal{Z}} d_\Phi\left(\mathbf{1}_y, \left(f_\theta(x)\right)_{\Phi^*}^*\right)$ (by the definition of $\gamma_\Phi(\theta)$), minimizing $d_\Phi\left(u, \left(f_\theta(x)\right)_{\Phi^*}^*\right)$ directly reduces the maximum loss value $\gamma_\Phi(\theta)$. Combining these results, we conclude that reducing $\|\theta\|_2^2$ (in a neighborhood of $\theta = \mathbf{0}$) is equivalent to reducing $\gamma_\Phi(\theta)$.

This completes the proof.
\end{pf}

\printcredits

\bibliographystyle{cas-model2-names}

\bibliography{cas-refs}

\end{document}